\RequirePackage[hyphens]{url}
\documentclass[final,authoryear,twocolumn]{elsarticle}

\usepackage[draft]{hyperref}

\usepackage{color}

\usepackage{amsmath,amssymb} 
\usepackage{soul}
\usepackage{mathtools}
\usepackage{algorithm}
\usepackage{algorithmic}
\usepackage{enumitem}
\usepackage{url}
\usepackage{subfigure}
\usepackage{comment}
\usepackage{caption}
\usepackage{amsthm}
\usepackage{multirow}
\usepackage{hyperref}

\DeclarePairedDelimiter\floor{\lfloor}{\rfloor}

\newtheorem{lemma}{Lemma}

\newcommand{\cC}{\mathcal{C}}

\newcommand{\cI}{\mathcal{I}}

\newcommand{\bP}{\mathbf{P}}
\newcommand{\bp}{\mathbf{p}}
\newcommand{\bQ}{\mathbf{Q}}
\newcommand{\bq}{\mathbf{q}}
\newcommand{\bR}{\mathbf{R}}

\newcommand{\bt}{\mathbf{t}}
\newcommand{\cJ}{\mathcal{J}}

\newcommand{\bV}{\mathbf{V}}
\newcommand{\bv}{\mathbf{v}}

\newcommand{\mbbI}{\mathbb{I}}

\DeclarePairedDelimiter{\norm}{\lVert}{\rVert}

\DeclareMathOperator*{\argmax}{arg\,max}

\makeatletter
\def\namedlabel#1#2{\begingroup
    #2%
    \def\@currentlabel{#2}%
    \phantomsection\label{#1}\endgroup
}

\journal{ISPRS Journal of Photogrammetry and Remote Sensing}

\begin{document}

\begin{frontmatter}

\title{Practical optimal registration of terrestrial LiDAR scan pairs}

\author[addressAdelaide]{Zhipeng Cai\corref{mycorrespondingauthor}}
\cortext[mycorrespondingauthor]{Corresponding author}
\ead{zhipeng.cai@adelaide.edu.au}

\author[addressAdelaide]{Tat-Jun Chin}
\ead{tat-jun.chin@adelaide.edu.au}

\author[addressAdelaide]{Alvaro Parra Bustos}
\ead{alvaro.parrabustos@adelaide.edu.au}

\author[addressZurich]{Konrad Schindler}
\ead{schindler@ethz.ch}

\address[addressAdelaide]{School of Computer Science, The University of Adelaide, Australia.}
\address[addressZurich]{Institute of Geodesy and Photogrammetry, ETH Zurich, Switzerland.}

\begin{abstract}

  Point cloud registration is a fundamental problem in 3D scanning. In
  this paper, we address the frequent special case of registering
  terrestrial LiDAR scans (or, more generally, levelled point clouds).
  Many current solutions still rely on the Iterative Closest Point
  (ICP) method or other heuristic procedures, which require good
  initializations to succeed and/or provide no guarantees of
  success. On the other hand, exact or optimal registration algorithms
  can compute the best possible solution without requiring
  initializations; however, they are currently too slow to be
  practical in realistic applications.

  Existing optimal approaches ignore the fact that in
  routine use the relative rotations between scans are constrained to
  the azimuth, via the built-in level compensation in LiDAR
  scanners. We propose a novel, optimal and computationally
  efficient registration method for this 4DOF scenario. Our approach
  operates on candidate 3D keypoint correspondences, and contains two main
  steps: (1) a deterministic selection scheme that significantly
  reduces the candidate correspondence set in a way that is guaranteed to
  preserve the optimal solution; and (2) a fast
  branch-and-bound (BnB) algorithm with a novel polynomial-time
  subroutine for 1D rotation search, that quickly finds the optimal alignment for the reduced set.  We demonstrate the
  practicality of our method on realistic point clouds from multiple
  LiDAR surveys.

\end{abstract}


\end{frontmatter}


\section{Introduction}\label{sec:intro}

LiDAR scanners are a standard instrument in contemporary surveying
practice. An individual scan produces a 3D point cloud, consisting of
densely sampled, polar line-of-sight measurements of the instrument's
surroundings, up to some maximum range. Consequently, a recurrent
basic task in LiDAR surveying is to register individual scans into one
big point cloud that covers the entire region of interest.  The
fundamental operation is the relative alignment of a pair of
scans. Once this can be done reliably, it can be applied sequentially
until all scans are registered; normally followed by a simultaneous
refinement of all registration parameters.

Our work focuses on the pairwise registration: given two point clouds,
compute the rigid transformation that brings them into
alignment. Arguably the most widely used method for this now classical
problem is the ICP (Iterative Closest Point)
algorithm~\citep{besl1992method,rusinkiewicz2001efficient,pomerleau2013comparing},
which alternates between finding nearest-neighbour point
matches and updating the transformation parameters. Since the
procedure converges only to a local optimum, it requires a reasonably
good initial registration to produce correct results.

Existing industrial solutions, which are shipped either as on-board
software of the scanner itself, or as part of the manufacturer's
offline processing software (e.g., Zoller + Fr\"ohlich\footnote{\url{https://www.zf-laser.com/Z-F-LaserControl-R.laserscanner_software_1.0.html?&L=1}},
Riegl\footnote{\url{http://www.riegl.com/products/software-packages/risolve/}}, Leica\footnote{\url{https://leica-geosystems.com/products/laser-scanners/software/leica-cyclone}}), rely on external aids or
additional sensors. For instance, a GNSS/IMU sensor package, or a
visual-inertial odometry system with a panoramic camera~\citep{houshiar2015study} setup, to
enable scan registration in GNSS-denied environments, in particular
indoors and under ground. Another alternative is to determine the
rotation with a compass, then perform only a translation search, which
often succeeds from a rough initialisation, such as setting the
translation to $0$.
Another, older but still popular approach is to install artificial
targets in the environment~\citep{akca2003full, franaszek2009fast} that
act as easily detectable and matchable ``beacons". However, this comes
at the cost of installing and maintaining the targets.

More sophisticated point cloud registration techniques have been
proposed that are not as strongly dependent on good initializations,
e.g.,~\citep{chen1999ransac,drost2010model,albarelli2010game,theiler2014keypoint,theiler2015globally}. These
techniques, in particular the optimization routines they employ, are
heuristic. They often succeed, but cannot guarantee to find an optimal
alignment (even according to their own definition of optimality).
Moreover, such methods typically are fairly sensitive to the tuning of
somewhat unintuitive, input-specific parameters, such as the
approximate proportion of overlapping points in
4PCS~\citep{aiger20084}, or the annealing rate of the penalty component
in the lifting method of~\citep{zhou2016fast}.
In our experience, when applied to new, unseen registration tasks
these methods often do not reach the performance reported on popular
benchmark datasets\footnote{For example, the Stanford 3D Scanning Repository~\citep{turk1994zippered}.}.

In contrast to the locally convergent algorithms and heuristics above, 
optimal algorithms have been developed for point cloud
registration~\citep{breuel2001practical,yang2016go,campbell2016gogma,parra2016fast}.
Their common theme is to set up a clear-cut, transparent objective
function and then apply a suitable exact optimization scheme -- often
branch-and-bound type methods -- to find the solution that maximises the objective function\footnote{As opposed to approximate, sub-optimal, or locally optimal solutions with lesser objective values than the maximum achievable.}.
It is thus ensured that the best registration parameters (according to
the adopted objective function) will always be found. Importantly,
convergence to the optimal value independent of the starting
point implies that these methods do not require
initialization.
However, a serious limitation of optimal methods is that they
are computationally much costlier (due to NP-hardness of most of the robust objective functions used in point cloud registration~\citep{chin2018robust}), but also by the experiments in previous literatures, e.g., more than 4 h for only
$\approx$ 350 points in~\citep{parra2016fast}, and even longer
in~\citep{yang2016go}. This makes them impractical for LiDAR surveying.

\subsection{Our contributions}\label{sec:contributions}

In this work, we aim to make optimal registration practical for
terrestrial LiDAR scans.

\begin{figure}[t]\centering
\subfigure{\includegraphics[width=1\columnwidth]{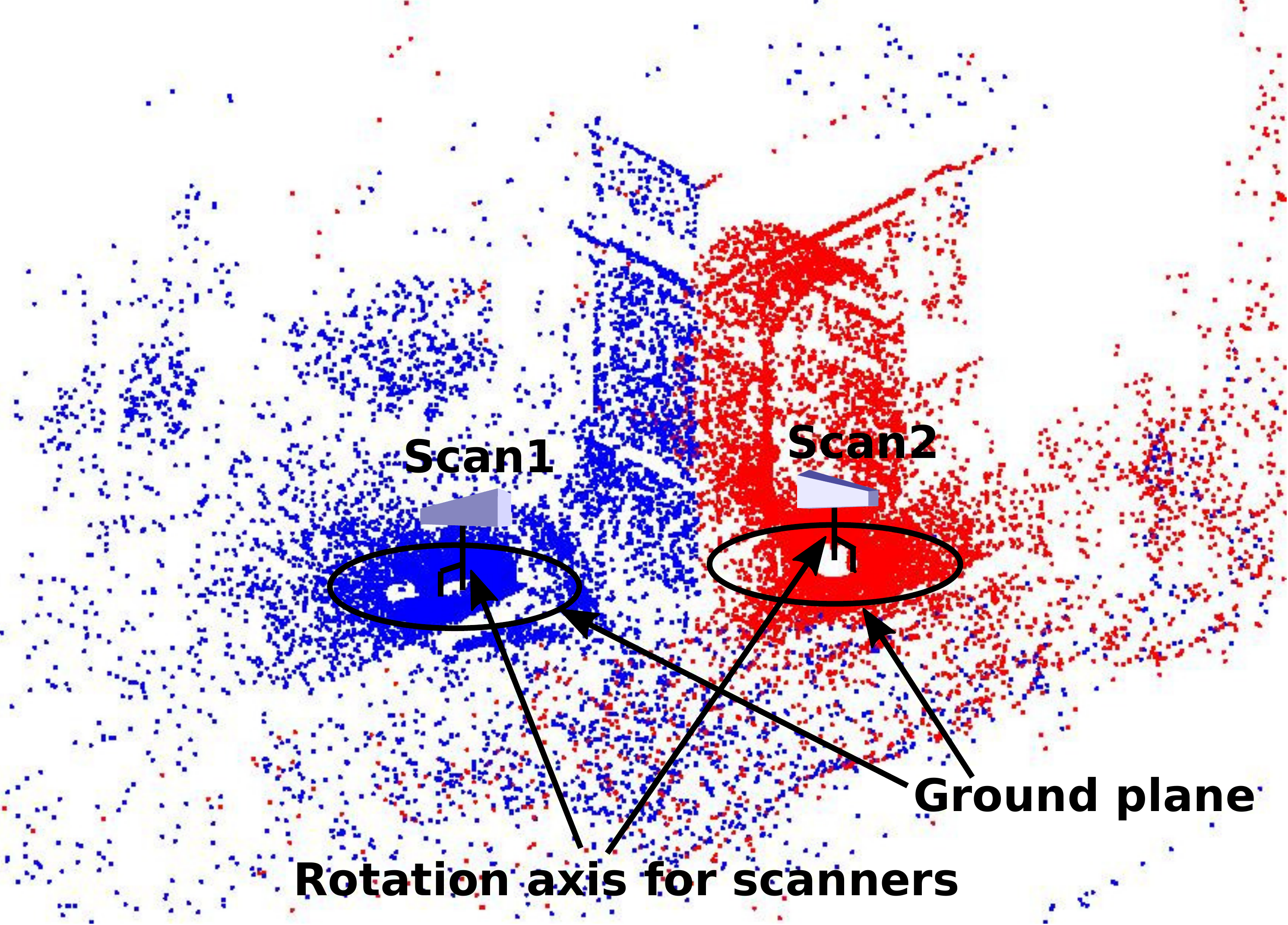}}
\caption{4DOF registration of two LiDAR scans (blue and red). The level compensator forces scanners to rotate around the vertical axis, resulting in the azimuthal relative rotation.}\label{fig:levelCompDemo}
\end{figure}

\begin{figure*}[t]\centering
\subfigure{\includegraphics[width=2\columnwidth]{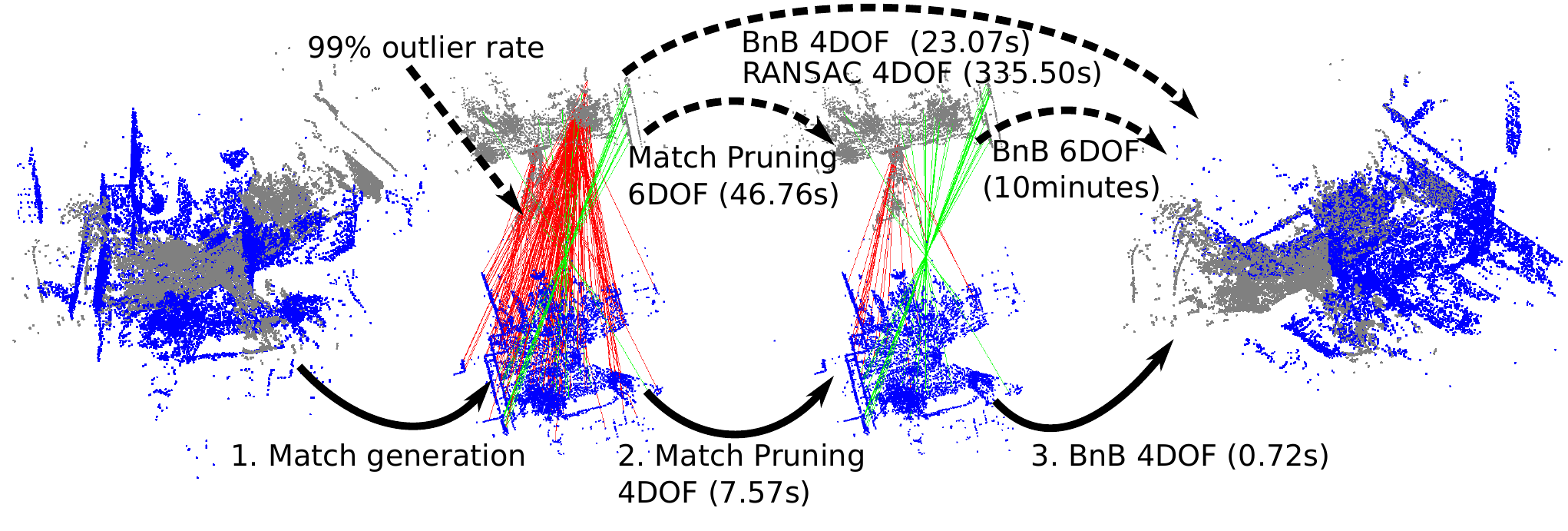}\label{subfig:smallBnB}}
\caption{Method illustration. Given 2 point clouds (grey and
  blue), as shown by the solid arrows, we first generate a set of matches/candidate correspondences (only a subset is shown for visual clarity), which
  have both inliers (green) and outliers (red). The
  matches are quickly pruned \emph{without removing any
    inliers}. Then, we run 4DOF BnB on the remaining matches
  to find the optimal solution. The dashed arrows show some alternative
  optimization choices. Note that our two-step process is much faster
  than directly optimising on the initial matches, and also more
  practically useful than its 6DOF
  counterpart~\citep{bustos2017guaranteed}.}\label{fig:mainDemo}
\end{figure*}

Towards that goal we make the following observations:
\begin{itemize}
\item Modern LiDAR devices are equipped with a highly accurate\footnote{The tilt accuracy is $0.002^\circ$ (see page 2 of \url{http://www.gb-geodezie.cz/wp-content/uploads/2016/01/datenblatt_imager_5006i.pdf}) for the Zoller\&Fr\"ohlich Imager 5006i used to capture datasets in our experiment. Similar accuracy can be found in scanners from other companies, e.g., 7.275e-6 radians for the Leica P40 tilt compensator (see page 8 of \url{https://www.abtech.cc/wp-content/uploads/2017/04/Tilt_compensation_for_Laser_Scanners_WP.pdf}).} level compensator,
  which reduces the relative rotation between scans to the azimuth;
  see Figure~\ref{fig:levelCompDemo}. In most applications the search
  space therefore has only 4 degrees of freedom (DOF) rather than
  6. This difference is significant, because the runtime of optimal methods grows very quickly with the problem dimension, see Figure~\ref{fig:mainDemo}.
  
\item A small set of correspondences, i.e., \emph{correct} point matches referring to close enough locations in the scene, is
  sufficient to reliably estimate the relative sensor pose, see
  Figure~\ref{fig:mainDemo}. The problem of match-based
  registration methods is normally not that there are too few
  correspondences to solve the problem; but rather that they are
  drowned in a large number of incorrect point matches, because of the
  high failure rate of existing 3D matching methods. The task would be
  a lot easier if we had a way to discard many false correspondences without
  losing true ones.
\end{itemize}

\begin{figure*}[t]\centering
\subfigure{\includegraphics[width=2\columnwidth]{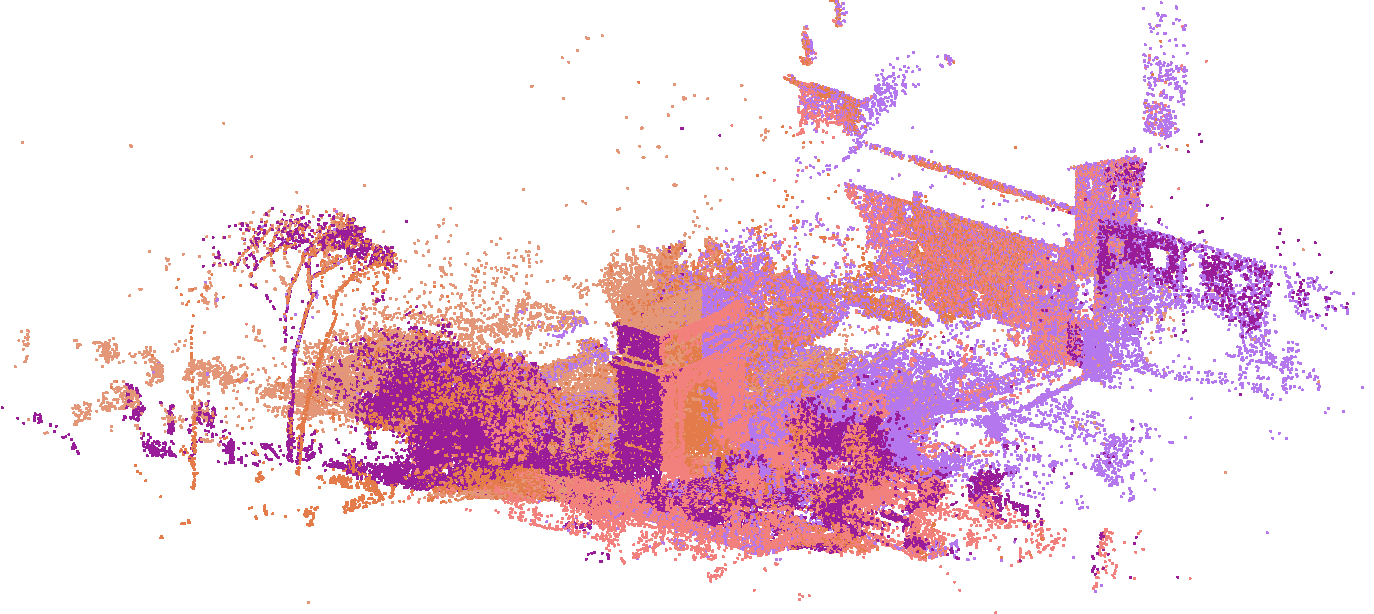}}
\caption{Registration result of our method for \textit{Arch} dataset. The registration of all five scans (with 15k-20k input point matches per pair, not shown) takes 187.53s, without requiring manual initializations. Note that our main contribution in this paper is a fast optimal algorithm for registering LiDAR scan pairs. To register multiple scans (as achieved on the \textit{Arch} dataset in this figure), we sequentially register the individual scans; see Section~\ref{sec:multiScanReg} for details.}\label{fig:regArchOutAll}
\end{figure*}

On the basis of these observations we develop a novel method for
optimal, match-based LiDAR registration, that has
the following main steps:
\begin{itemize}
\item Instead of operating directly on the input match set, a
  fast deterministic preprocessing step is executed to aggressively
  prune the input set, in a way which guarantees that only false
  correspondences are discarded. In this way, it is ensured that the
  optimal alignment of the reduced set is the same as for
  the initial, complete set of matches.
\item A fast 4DOF BnB algorithm is then run on the remaining
  matches to compute the optimal alignment parameters. Our BnB algorithm contains a deterministic
  polynomial-time subroutine for 1DOF rotation search, which
  accelerates the optimization.
\end{itemize}
Figure~\ref{fig:mainDemo} illustrates our approach. As suggested in
the figure and by our comprehensive testing (see
Section~\ref{sec:Exp}), our approach significantly speeds up optimal registration and makes it practical for realistic LiDAR surveying
tasks. For example, on the
\textit{Arch}\footnote{\url{http://www.prs.igp.ethz.ch/research/completed\_projects/automatic\_registration\_of\_point\_clouds.html}}
dataset, it is able to accurately register all 5 scans in around 3
min \emph{without} any
initializations to the registration; see
Figure~\ref{fig:regArchOutAll}.

Please visit our project homepage\footnote{\url{https://github.com/ZhipengCai/Demo---Practical-optimal-registration-of-terrestrial-LiDAR-scan-pairs}} for the video demo and source code.

\section{Related work}

Point-based registration techniques can be broadly categorized into
two groups: methods using ``raw" point clouds, and methods using 3D
point matches. Since our contribution belongs to the second
group, we will focus our survey on that group. Nonetheless, in our
experiments, we will compare against both raw point cloud methods and
match-based methods.

Match-based methods first extract a set of candidate correspondences from the input point clouds (using feature matching techniques such
as~\citep{scovanner20073,glomb2009detection,zhong2009intrinsic,rusu2008aligning,rusu2009fast}),
then optimize the registration parameters using the extracted candidates
only. Since the accuracy of 3D keypoint detection and matching is much
lower than their 2D
counterparts~\citep{harris1988combined,lowe1999object}, a major concern
of match-based methods is to discount the corrupting effects
of false correspondences or outliers.

A widely used strategy for robust estimation is Random Sample
Consensus (RANSAC)~\citep{fischler1981random}. However, the runtime of
RANSAC increases exponentially with the outlier rate, which is
generally very high in 3D registration,e.g., the input match set in
Figure~\ref{fig:mainDemo} contains more than $99\%$ outliers. More
efficient approaches have been proposed for dealing with outliers in
3D match sets, such as the game-theoretic method
of~\citep{albarelli2010game} and the lifting optimization method
of~\citep{zhou2016fast}. However, these are either heuristic (e.g.,
using randomisation~\citep{albarelli2010game}) and find solutions that
are, at most, \emph{correct with high probability}, but may also be
grossly wrong; or they are only locally optimal~\citep{zhou2016fast},
and will fail when initialized outside of the correct solution's
convergence basin.

Optimal algorithms for match-based 3D registration
also
exist~\citep{bazin2012globally,bustos2017guaranteed,yu2018maximum}. However,~\citep{bazin2012globally}
is restricted to pure rotational motion, while the 6DOF algorithms
of~\citep{bustos2017guaranteed,yu2018maximum} are computationally
expensive for unfavourable
configurations. E.g.,~\citep{bustos2017guaranteed} takes more than
  10 min to register the pair of point clouds in Figure~\ref{fig:mainDemo}.

Recently, clever filtering schemes have been
developed~\citep{svarm2014accurate,parra2015guaranteed,chin2017maximum,bustos2017guaranteed}
which have the ability to efficiently prune a set of putative
correspondences and only retain a much smaller subset, in a manner
that does not affect the optimal solution (more details in
Section~\ref{sec:GORE}). Our method can be understood as an extension
of the 2D rigid registration method of~\citep[Section
  4.2.1]{chin2017maximum} to the 4DOF case.
  
Beyond points, other geometric primitives (lines, planes, etc.) have also been exploited for LiDAR registration~\citep{brenner2007automatic, rabbani2007integrated}. By using more informative order structures in the data, such methods can potentially provide more accurate registration. Developing optimal registration algorithms based on higher-order primitives would be interesting future work.

\section{Problem formulation}

Given two input point clouds $\bP$ and $\bQ$, we first extract a set
of 3D keypoint matches $\cC = \{ (\bp_i,\bq_i)\}_{i=1}^M$
between $\bP$ and $\bQ$. This can be achieved using fairly standard
means --- Section~\ref{sec:expSynth} will describe our version of
the procedure. Given $\cC$, our task is to estimate the 4DOF rigid
transformation
\begin{align}
f(\bp \mid \theta, \bt) = \bR(\theta)\bp + \bt,
\end{align}
parameterized by a rotation angle $\theta \in [0,2\pi]$ and
translation vector $\bt \in \mathbb{R}^3$, that aligns as many of the
pairs in $\cC$ as possible. Note that
\begin{align}
\bR(\theta) = \left[ \begin{matrix}
& \cos{\theta} & \sin{\theta} & 0 \\
& -\sin{\theta} & \cos{\theta} & 0 \\
& 0 & 0 & 1
\end{matrix} \right]
\end{align}
defines a rotation about the 3rd axis, which we assume to be
aligned with gravity, expressing the fact that LiDAR scanners in
routine use are levelled.%
\footnote{The method is general and will work for any setting that
  allows only a 1D rotation around a known axis.}

Since $\cC$ contains outliers (false correspondences), $f$ must be
estimated robustly. To this end, we seek the parameters $\theta, \bt$
that maximize the objective
\begin{align}\label{eq:energy}
E(\theta, \bt \mid \cC, \epsilon) = \sum_{i=1}^M \mbbI\left(\norm{\bR(\theta)\bp_i+\bt-\bq_i}\leq\epsilon\right),
\end{align}
where $\epsilon$ is the \emph{inlier threshold}, and $\mbbI$ is an
indicator function that returns 1 if the input predicate is satisfied
and 0 otherwise. Intuitively,~\eqref{eq:energy} calculates the number
of pairs in $\cC$ that are aligned up to distance $\epsilon$ by $f(\bp
\mid \theta,\bt)$. Allowing alignment only up to $\epsilon$ is vital to
exclude the influence of the outliers.
Note that choosing the right $\epsilon$ is usually not an obstacle,
since LiDAR manufacturers specify the precision of the device, which
can inform the choice of an appropriate threshold. Moreover, given the
fast performance of the proposed technique, one could conceivably run
multiple rounds of registration with different $\epsilon$ and choose
one based on the alignment residuals.

Our overarching aim is thus to solve the optimization problem
\begin{align}\label{eq:obj}
E^* = \underset{\theta,\bt}{\text{max}}\ E(\theta,\bt \mid \cC,\epsilon)
\end{align}
exactly or optimally; in other words we are searching for the angle
and translation vector $\theta^*,\bt^*$ that yield the highest
objective value $E^* = E(\theta^*,\bt^* \mid \cC, \epsilon)$.
We note that in the context of registration, optimality is not
merely an academic exercise. Incorrect local minima are a real
problem, as amply documented in the literature on ICP and
illustrated by regular failures of the automatic registration routines
in production software.

\subsection{Main algorithm}

As shown in Algorithm~\ref{alg:main}, our approach to
solve~\eqref{eq:obj} optimally has two main steps: a
deterministic pruning step to reduce $\cC$ to a much smaller subset
$\cC^\prime$ while removing only matches that cannot be
correct, hence preserving the optimal solution $\theta^\ast, \bt^\ast$
in $\cC^\prime$ (Section~\ref{sec:GORE}); and a fast, custom-tailored
BnB algorithm to search for $\theta^*,\bt^*$ in $\cC^\prime$
(Section~\ref{sec:BnB}). For better flow of the presentation, we will
describe the BnB algorithm first, before the pruning.

\begin{algorithm}[t]\centering
\caption{Main algorithm.}\label{alg:main}
\begin{algorithmic}[1]
\REQUIRE Point clouds $\bP$ and $\bQ$ with 1D relative rotation, inlier threshold $\epsilon$.
\STATE Extract match set $\cC$ from $\bP$ and $\bQ$ (Section~\ref{sec:expSynth}).
\STATE Prune $\cC$ into a smaller subset $\cC^\prime$ (Section~\ref{sec:GORE}).\label{step:gore}
\STATE Solve~\eqref{eq:obj} on $\cC^\prime$ to obtain registration parameters $\theta^*,\bt^*$ (Section~\ref{sec:BnB}).
\RETURN Optimal solution $\theta^*, \bt^*$.
\end{algorithmic}
\end{algorithm}

\subsection{Registering multiple scans}\label{sec:multiScanReg}

In some applications, registering multiple scans is required. For this purpose, we can first perform pair-wise registration to estimate the relative poses between consecutive scans. These pair-wise poses can then be used to initialize simultaneous multi-scan registration methods like the maximum-likelihood alignment of~\citep{lu1997globally} and conduct further optimization.  

To refrain from further polishing that would obfuscate the contribution made by our robust pair-wise registration, all multi-scan registration results in this paper are generated by sequential pair-wise registration \emph{only},
i.e., starting from the initial scan, incrementally register the next scan to the
current one using Algorithm{~\ref{alg:main}. As shown in Figure~\ref{fig:regArchOutAll} and later in Section~\ref{sec:expReal}, our results are already promising even without further refinement.

\section{Fast BnB algorithm for 4DOF registration}\label{sec:BnB}

To derive our fast BnB algorithm, we first rewrite~\eqref{eq:obj} as
\begin{align}\label{eq:objExplicitT}
E^\ast = \underset{\bt}{\text{max}}\ U(\bt \mid \cC,\epsilon),
\end{align}
where
\begin{align}\label{eq:energyExplicitT}
U(\bt \mid \cC, \epsilon) = \underset{\theta}{\text{max}}\ E(\theta, \bt \mid \cC, \epsilon).
\end{align}
It is not hard to see the equivalence of~\eqref{eq:obj}
and~\eqref{eq:objExplicitT}. The purpose of~\eqref{eq:objExplicitT} is
twofold:
\begin{itemize}
\item As we will see in Section~\eqref{sec:1DRot}, estimating $\theta$
  \emph{given} $\bt$ can be accomplished deterministically in
  polynomial time, such that the optimization of $\theta$ can be
  viewed as ``evaluating" the function $U(\bt \mid \cC, \epsilon)$.
\item By ``abstracting away" the variable $\theta$ as
  in~\eqref{eq:objExplicitT}, the proposed BnB algorithm
  (Section~\ref{sec:BnB1}) can be more conveniently formulated as
  searching over only the 3-dimensional translation space
  $\mathbb{R}^3$.
\end{itemize}

\subsection{Deterministic rotation estimation}\label{sec:1DRot}

\begin{figure}[t]\centering
\subfigure{\includegraphics[width=1\columnwidth]{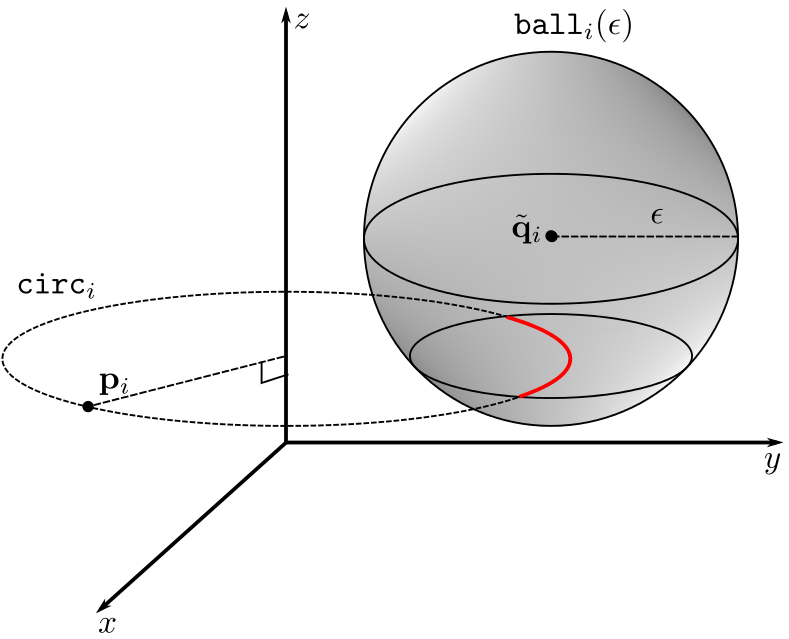}}
\caption{The intersection of $\mathtt{circ}_i$ and $\mathtt{ball}_i(\epsilon)$. The
  intersection part is rendered in red.}\label{fig:1DRotDemo}
\end{figure}

For completeness, the definition of $U(\bt \mid \cC, \epsilon)$ is as
follows
\begin{align}\label{eq:obj1DRot}
U(\bt \mid \cC, \epsilon) = \underset{\theta}{\text{max}}
\sum_{i=1}^{M}
\mbbI\left(\norm{\bR(\theta)\bp_i-\tilde{\bq}_i}\leq\epsilon\right),
\end{align}
where $\tilde{\bq}_i = \bq_i-\bt$. Intuitively, evaluating this function
amounts to finding the rotation $\bR(\theta)$ that aligns as many
pairs $\{(\bp_i,\tilde{\bq}_i)\}^{M}_{i=1}$ as possible.

Note that for each $\bp_i$, rotating it with $\bR(\theta)$ for all
$\theta \in [0,2\pi]$ forms the circular trajectory
\begin{align}
\mathtt{circ}_i = \left\{ \bR(\theta)\bp_i \mid \theta\in[0,2\pi] \right\}.
\end{align}
Naturally, $\mathtt{circ}_i$ collapses to a point if $\bp_i$ lies on the z-axis. Define 
\begin{align}
\mathtt{ball}_i(\epsilon) = \left\{ \bq \in \mathbb{R}^3 \mid \| \bq - \tilde{\bq}_i \| \le \epsilon \right\}
\end{align}
as the $\epsilon$-ball centered at $\tilde{\bq}_i$. It is clear that the pair
$(\bp_i,\tilde{\bq}_i)$ can only be aligned by $\bR(\theta)$ if
$\mathtt{circ}_i$ and $\mathtt{ball}_i(\epsilon)$ intersect; see
Figure~\ref{fig:1DRotDemo}. Moreover, if $\mathtt{circ}_i$ and
$\mathtt{ball}_i(\epsilon)$ do not intersect, then we can be sure that
the $i$-th pair plays no role in~\eqref{eq:obj1DRot}.

For each $i$, denote
\begin{align}
\mathtt{int}_{i} = [\alpha_{i},\beta_{i}] \subseteq [0,2\pi]
\end{align}
as the angular interval such that, for all $\theta \in
\mathtt{int}_i$, $\bR(\theta)\bp_i$ is aligned with $\tilde{\bq}_i$
within distance $\epsilon$. The limits $\alpha_{i}$ and $\beta_i$ can
be calculated in closed form via circle-to-circle intersections; see
Appendix~\ref{app:compInt}. Note that $\mathtt{int}_i$ is empty if
$\mathtt{circ}_i$ and $\mathtt{ball}_i(\epsilon)$ do not
intersect.
For brevity, in the following we take all $\mathtt{int}_i$ to be
single intervals. For the actual implementation, it is
straight-forward to break $\mathtt{int}_i$ into two intervals if it
extends beyond the range $[0,2\pi]$.
The function $U(\bt \mid \cC, \epsilon)$ can then be rewritten as
\begin{align}\label{eq:objStab}
U(\bt \mid \cC, \epsilon) = \underset{\theta}{\text{max}} \sum_{i=1}^M\mbbI\left( \theta \in[\alpha_{i},\beta_{i}]\right),
\end{align}
which is an instance of the \textit{max-stabbing}
problem~\citep[Chapter 10]{de2000computational}; see
Figure~\ref{fig:stabbing}. Efficient algorithms for max-stabbing are
known, in particular, the version in Algorithm~\ref{alg:1DRot} in the
Appendix runs deterministically in $\mathcal{O}(M\log M)$ time. This
supports a practical optimal algorithm.

\begin{figure}[t]\centering
\subfigure{\includegraphics[width=1\columnwidth]{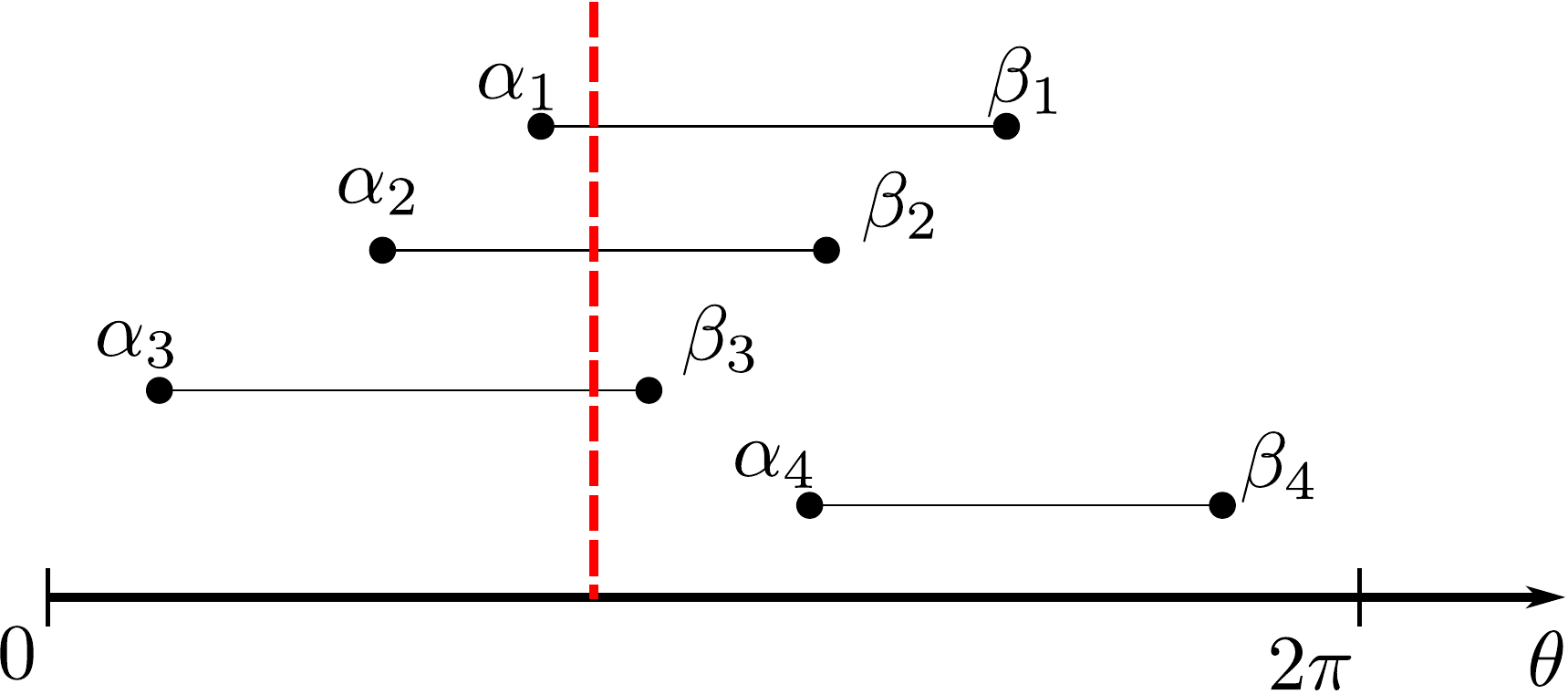}}
\caption{The max-stabbing problem aims to find a vertical line
  (defined by angle $\theta$ in our case) that ``stabs" the maximum
  number of intervals, e.g., the dashed red line. Note that $0$ and $2\pi$ refer to the same 1D rotation. To ensure their stabbing values are equal, if an interval has only one end on one of these two angles, an extra interval that starts and ends both at the other angle is added to the input.}\label{fig:stabbing}
\end{figure}

\subsection{BnB for translation search}\label{sec:BnB1}

In the context of solving~\eqref{eq:objExplicitT}, the BnB method
initializes a cube $\mathbb{S}_0$ in $\mathbb{R}^3$ that contains the
optimal solution $\bt^*$, then recursively partitions $\mathbb{S}_0$
into 8 sub-cubes; see Algorithm~\ref{alg:BnB1}. For each sub-cube
$\mathbb{S} \subset \mathbb{S}_0$, let $\bt_\mathbb{S}$ be the center
point of $\mathbb{S}$. If $\bt_\mathbb{S}$ gives a higher objective
value than the current best estimate $\hat{\bt}$, the latter is
updated to become the former, $\hat{\bt}\leftarrow\bt_\mathbb{S}$;
else, either
\begin{itemize}
\item a decision is made to discard $\mathbb{S}$ (see below); or
\item $\mathbb{S}$ is partitioned into 8 sub-cubes and the process
  above is repeated.
\end{itemize}
In the limit, $\hat{\bt}$ approaches the optimal solution $\bt^\ast$.

\begin{algorithm}[t]\centering
\caption{BnB for 4DOF match-based registration~\eqref{eq:objExplicitT}.}
\label{alg:BnB1}
\begin{algorithmic}[1]
\REQUIRE Initial matches $\cC$ and the inlier threshold
$\epsilon$.
\STATE Set the priority queue $w$ to $\emptyset$,
$\mathbb{S}_0 \leftarrow $ the initial translation cube, $\hat{\bt}\leftarrow \bt_{\mathbb{S}_0}$.
\STATE Compute $\bar{U}(\mathbb{S}_0\mid \cC, \epsilon)$ by Algorithm~\ref{alg:1DRot} and insert
($\mathbb{S}_0, \bar{U}(\mathbb{S}_0 \mid \cC, \epsilon)$) into $w$.
\WHILE{$w$ is not empty}
\STATE Pop out the cube $\mathbb{S}$ with the highest $\bar{U}(\mathbb{S} \mid \cC, \epsilon)$ from
$w$.\label{line:popQueue}
\STATE Compute $U(\bt_{\mathbb{S}} \mid \cC,\epsilon)$ by Algorithm~\ref{alg:1DRot}; If $U(\bt_{\mathbb{S}} \mid
\cC, \epsilon) = \bar{U}(\mathbb{S} \mid \cC, \epsilon)$, break.
\STATE If $U(\bt_{\mathbb{S}} \mid \cC, \epsilon)>U(\hat{\bt} \mid\cC, \epsilon)$, $\hat{\bt}\leftarrow \bt_{\mathbb{S}}$ and prune $w$ according to $U(\hat{\bt} \mid \cC, \epsilon)$.
\STATE Divide $\mathbb{S}$ into 8 sub-cubes $\{\mathbb{S}_o\}_{o=1}^8$ and compute $\bar{U}(\mathbb{S}_o \mid \cC, \epsilon)$ by Algorithm~\ref{alg:1DRot} for all $\mathbb{S}_o$.\label{line:branch1}
\STATE For each $\mathbb{S}_o$, if $\bar{U}(\mathbb{S}_o \mid \cC,\epsilon)>U(\hat{\bt} \mid \cC, \epsilon)$, insert ($\mathbb{S}_o,\bar{U}(\mathbb{S}_o \mid \cC, \epsilon)$) into $w$.\label{line:insertSubcube}
\ENDWHILE
\RETURN $\hat{\bt}$.
\end{algorithmic}
\end{algorithm}

Given a sub-cube $\mathbb{S}$ and the incumbent solution $\hat{\bt}$,
$\mathbb{S}$ is discarded if
\begin{align}\label{eq:bound}
\bar{U}(\mathbb{S} \mid \cC, \epsilon) \le U( \hat{\bt} \mid \cC,
\epsilon),
\end{align}
where $\bar{U}(\mathbb{S} \mid \cC, \epsilon)$ calculates an
\emph{upper bound} of $U(\bt \mid \cC, \epsilon)$ over domain
$\mathbb{S}$, i.e.,
\begin{align}
\bar{U}(\mathbb{S} \mid \cC, \epsilon) \ge \max_{\bt \in \mathbb{S}}~U(\bt \mid \cC, \epsilon).
\end{align}
The rationale is that if~\eqref{eq:bound} holds, then a solution that
is better than $\hat{\bt}$ cannot exist in $\mathbb{S}$. In our work,
the upper bound is obtained as
\begin{align}\label{eq:upBnd}
\bar{U}(\mathbb{S} \mid \cC, \epsilon) = U(\bt_\mathbb{S} \mid \cC, \epsilon+d_\mathbb{S}),
\end{align}
where $d_{\mathbb{S}}$ is half of the diagonal length of
$\mathbb{S}$. Note that computing the bound amounts to evaluating the
function $U$, which can be done efficiently via max-stabbing.

The following lemma establishes the validity of~\eqref{eq:upBnd} for BnB.
\begin{lemma}\label{lem:upBnd}
For any translation cube $\mathbb{S} \subset \mathbb{R}^3$,
\begin{align}\label{eq:lem_upBnd}
\bar{U}(\mathbb{S}\mid \cC, \epsilon) = U(\bt_\mathbb{S} \mid \cC,
\epsilon+d_\mathbb{S}) \geq
\underset{\bt\in\mathbb{S}}{\text{max}}~U(\bt \mid \cC, \epsilon),
\end{align}
and as $\mathbb{S}$ tends to a point $\bt$, then
\begin{align}
\bar{U}(\mathbb{S}\mid \cC, \epsilon) \to U(\bt \mid \cC, \epsilon).
\end{align}
See Appendix~\ref{app:lemProof} for the proofs.
\end{lemma}

To conduct the search more strategically, in Algorithm~\ref{alg:BnB1}
the unexplored sub-cubes are arranged in a priority queue $w$ based on
their upper bound value.
Note that while Algorithm~\ref{alg:BnB1} appears to be solving only
for the translation $\bt$ via problem~\eqref{eq:objExplicitT},
\emph{implicitly} it is simultaneously optimizing the angle
$\theta$: given the output $\bt^*$ from Algorithm~\ref{alg:BnB1}, the
optimal $\theta^*$ per the original problem~\eqref{eq:obj} can be
obtained by evaluating $U( \bt^* \mid \cC, \epsilon)$ and keeping the
maximizer.

\section{Fast preprocessing algorithm}\label{sec:GORE}

Instead of invoking BnB (Algorithm~\ref{alg:BnB1}) on the
match set $\cC$ directly, our approach first executes a
preprocessing step (see Step~\ref{step:gore} in
Algorithm~\ref{alg:main}) to reduce $\cC$ to a much smaller subset
$\cC^\prime$, then runs BnB on $\cC^\prime$. Remarkably. this pruning
can be carried out in such a way that the optimal solution is
preserved in $\cC^\prime$, i.e.,
\begin{align}\label{eq:preserve1}
\theta^*,\bt^* = \argmax_{\theta,\bt}~E(\theta,\bt \mid \cC, \epsilon) = \argmax_{\theta,\bt}~E(\theta,\bt \mid \cC^\prime, \epsilon).
\end{align}
Hence, BnB runs a lot faster, but still finds the optimum
w.r.t.\ the original, full match set.

Let $\mathcal{I}^*$ be the subset of $\cC$ that are aligned by
$\theta^*,\bt^*$, formally
\begin{align}
\| \bR(\theta^*)\bp_i + \bt^* - \bq_i \| \le \epsilon \;\; \forall (\bp_i,\bq_i) \in \mathcal{I}^*.
\end{align}
If the following condition holds
\begin{align}
\mathcal{I}^* \subseteq \cC^\prime \subseteq \cC,
\end{align}
then it follows that~\eqref{eq:preserve1} will also hold. Thus, the
trick for preprocessing is to remove only matches that are
in $\cC \setminus \cI^*$, i.e., ``certain outliers".

\subsection{Identifying the certain outliers}

To accomplish the above, define the problem \hypertarget{P[k]}{$\mathcal{P}[k]$}\footnote{Note that the ``1+" in~\eqref{eq:fixedk} is necessary because we want the optimal value $E^*_k$ of $\mathcal{P}[k]$ to be exactly equal to $E^*$ if the k-th match is an inlier, which is the basis of Lemma~\ref{lem:ktest}.}:
\begin{equation}
\begin{aligned}
& \underset{\theta,\bt}{\text{max}}
& & 1 + \sum_{i \in \cJ_k} \mbbI\left(\norm{\bR(\theta)\bp_i+\bt-\bq_i}\leq\epsilon\right)\\
& \text{s.t.}
& & \| \bR(\theta)\bp_k + \bt - \bq_k \| \le \epsilon,
\end{aligned}\label{eq:fixedk}
\end{equation}

where $k\in \{1,\dots,M \}$, and $\cJ_k = \{1,\dots,M\}\setminus\{ k
\}$. In words, \hyperlink{P[k]}{$\mathcal{P}[k]$} is the same problem as the original
registration problem~\eqref{eq:obj}, except that the $k$-th
match must be aligned. We furthermore define
$E^*_k$ as the optimal value of \hyperlink{P[k]}{$\mathcal{P}[k]$};
$\bar{E}_k$ as an upper bound on the value of
\hyperlink{P[k]}{$\mathcal{P}[k]$}, such that $\bar{E}_k \ge E^*_k$;
and $\underline{E}$ as a lower bound on the value of \eqref{eq:obj},
  so $\underline{E} \le E^*$.
Note that, similar to $E^*$, $E^*_k$ can only be obtained by optimal
search using BnB, but we want to avoid such a costly computation in the pruning
stage.
Instead, if we have access to $\bar{E}_k$ and $\underline{E}$
(details in Section~\ref{sec:kbound}), the following test can be made.

\begin{lemma}\label{lem:ktest}
If $\bar{E}_k < \underline{E}$, then $(\bp_k,\bq_k)$ is a certain outlier.
\end{lemma}
\begin{proof}
If $(\bp_k,\bq_k)$ is in $\cI^*$, then we must have that $E^*_k =
E^*$, i.e., \hyperlink{P[k]}{$\mathcal{P}[k]$} and~\eqref{eq:obj} must have the same
solution. However, if we are given that $\bar{E}_k < \underline{E}$,
then
\begin{align}
E^*_k \le \bar{E}_k < \underline{E} \le E^*
\end{align}
which contradicts the previous statement. Thus, $(\bp_k,\bq_k)$ cannot be in $\cI^*$.
\end{proof}

The above lemma underpins a pruning algorithm that removes certain
outliers from $\cC$ in order to reduce it to a smaller subset
$\cC^\prime$, which still includes all inlier putative
correspondences; see Algorithm~\ref{alg:GORE}. The algorithm simply
iterates over $k = 1,\dots,M$ and attempts the test in
Lemma~\ref{lem:ktest} to remove $(\bp_k,\bq_k)$. At each $k$, the
upper and lower bound values $\bar{E}_k$ and $\underline{E}$ are
computed and/or refined (details in the following section). As we will
demonstrate in Section~\ref{sec:Exp}, Algorithm~\ref{alg:GORE} is able
to shrink $\cC$ to less than $20\%$ of its original size for practical
cases.

\begin{algorithm}[ht]\centering
\caption{Fast match pruning (FMP) for 4DOF registration.}
\label{alg:GORE}
\begin{algorithmic}[1]
\REQUIRE Initial matches $\cC$, the inlier threshold $\epsilon$.
\STATE $\underline{E}\leftarrow 0$, $\cC' \leftarrow \cC$.
\FOR{$k = 1,...,M$}
\STATE Compute $\bar{E}_k$ (Section~\ref{sec:kbound}).
\IF {$\bar{E}_k<\underline{E}$}
\STATE $\cC' \leftarrow \cC'\backslash (\bp_k,\bq_k)$.
\ELSE
\STATE Re-evaluate $\underline{E}$ using the corresponding solution of $\bar{E}_k$ (Section~\ref{sec:kbound}).
\ENDIF
\ENDFOR
\STATE Remove from $\cC'$ the remaining $(\bp_k,\bq_k)$ whose $\bar{E}_k<\underline{E}$ .
\RETURN $\cC'$
\end{algorithmic}
\end{algorithm}

\subsection{Efficient bound computation}\label{sec:kbound}

For the data in problem~$\mathcal{P}[k]$, let them be centered
w.r.t.~$\bp_k$ and $\bq_k$, i.e.,
\begin{align}
\bp^\prime_i = \bp_i - \bp_k,~~~~\bq^\prime_i = \bq_i - \bq_k,~~~~\forall i.
\end{align}
Then, define the following pure rotational problem $\mathcal{Q}[k]$:
\begin{equation}
\begin{aligned}
& \underset{\theta}{\text{max}}
& & 1 + \sum_{i \in \cJ_k} \mbbI \left(\norm{\bR(\theta)\bp'_i-\bq'_i} \leq 2\epsilon\right).
\end{aligned}\label{eq:fixedkcent}
\end{equation} 
We now show that $\bar{E}_k$ and $\underline{E}$ in
Algorithm~\ref{alg:GORE} can be computed by
solving~$\mathcal{Q}[k]$, which can again be done efficiently
using max-stabbing (Algorithm~\ref{alg:1DRot}).

First, we show by Lemma~\ref{lem:PkQk} that the value
of~$\mathcal{Q}[k]$ can be directly used as $\bar{E}_k$, i.e., the
number of inliers in~$\mathcal{Q}[k]$ is an upper bound of the one
in~$\mathcal{P}[k]$.
\begin{lemma}~\label{lem:PkQk}
If $(\bp_i,\bq_i)$ is aligned by the optimal solution $\theta_k^\ast$
and $\bt^\ast_k$ of~$\mathcal{P}[k]$, $(\bp'_i,\bq'_i)$ must also be
aligned by $\theta_k^\ast$ in~$\mathcal{Q}[k]$.
\end{lemma}
\begin{proof}
To align $(\bp_k,\bq_k)$ in~$\mathcal{P}[k]$, $\bt^\ast_k$ must be
within the $\epsilon$-ball centered at
$\bq_k-\bR(\theta^\ast_k)\bp_k$, i.e., $\bt^\ast_k$ can be
re-expressed by $\bq_k-\bR(\theta^\ast_k)\bp_k+\bt_k^\ast{'}$, where
$\norm{\bt_k^\ast{'}}\leq\epsilon$. Using this re-expression, when
$(\bp_i,\bq_i)$ is aligned by $\theta^\ast_k$ and $\bt^\ast_k$, we
have
\begin{align}
&\norm{\bR(\theta^\ast_k)\bp_i+(\bq_k-\bR(\theta^*_k)\bp_k+\bt_k^*{'})-\bq_i} \\ 
= &\norm{\bR(\theta^\ast_k)\bp'_i+\bt_k^*{'}-\bq'_i} \leq \epsilon\label{eq:lem3_1}\\
\Rightarrow & \norm{\bR(\theta^\ast_k)\bp'_i-\bq'_i} - \norm{\bt_k^*{'}} \le \epsilon\label{eq:lem3_2}\\
\Rightarrow & \norm{\bR(\theta^*_k)\bp'_i-\bq'_i} - \epsilon \leq \epsilon \Leftrightarrow \norm{\bR(\theta^*_k)\bp'_i-\bq'_i} \leq 2\epsilon,\label{eq:lem3_3}
\end{align}
\eqref{eq:lem3_2} and~\eqref{eq:lem3_3} are due respectively to the
triangle inequality\footnote{\url{https://en.wikipedia.org/wiki/Triangle\_inequality}} and to
$\norm{\bt_k^*{'}}\leq\epsilon$. According to \eqref{eq:lem3_3},
$(\bp'_i,\bq'_i)$ is also aligned by $\theta_k^*$
in~$\mathcal{Q}[k]$.
\end{proof} 
On the other hand, $\underline{E}$, the lower bound of $E^\ast$, can
be calculated using the optimal solution $\tilde{\theta}_k$ of
$\mathcal{Q}[k]$. Specifically, we set $\tilde{\bt}_k =
\bq_k-\bR(\tilde{\theta}_k)\bp_k$ and compute $\underline{E} =
U(\tilde{\bt}_k \mid \cC', \epsilon) = E(\tilde{\theta}_k,
\tilde{\bt}_k \mid \cC', \epsilon)$, directly following
Equation~\eqref{eq:energy}.

In this way, evaluating $\bar{E}_k$ and $\underline{E}$ takes $\mathcal{O}(M\log
M)$, respectively $\mathcal{O}(M)$ time. As Algorithm~\ref{alg:GORE} repeats
both evaluations $M$ times, its time complexity is $\mathcal{O}(M^2 \log M)$.

\section{Experiments}\label{sec:Exp}

The experiments contain two parts, which show respectively the results
on \emph{controlled} and \emph{real-world} data. All experiments were
implemented in C++ and executed on a laptop with 16 GB RAM and Intel
Core 2.60 GHz i7 CPU.

\begin{figure*}[!htb]\centering
\subfigure[Runtime---\textit{Bunny}.]{\includegraphics[width=0.7\columnwidth]{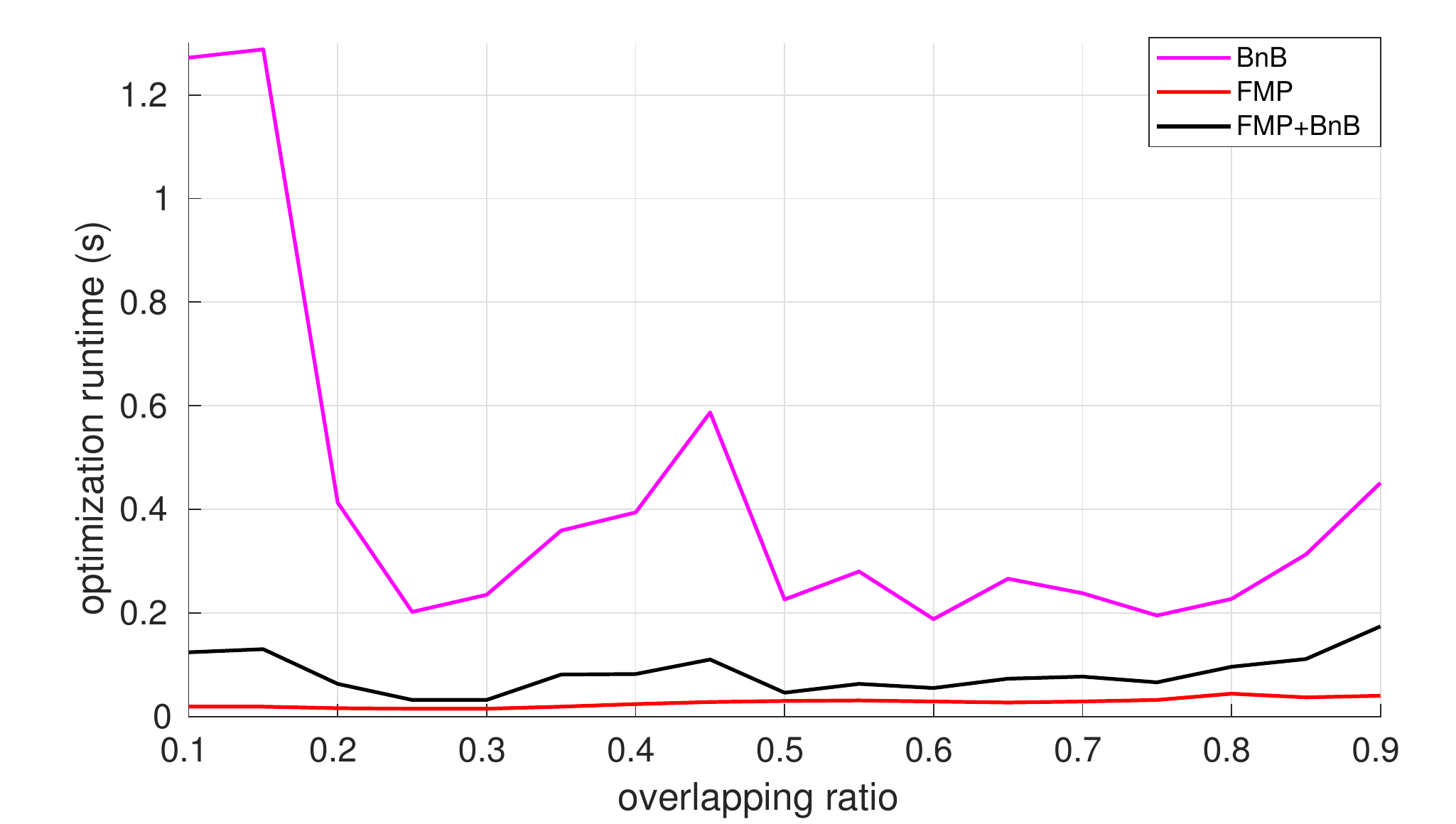}\label{subfig:1DRot_Bunny}}
\subfigure[Runtime---\textit{Armadillo}.]{\includegraphics[width=0.7\columnwidth]{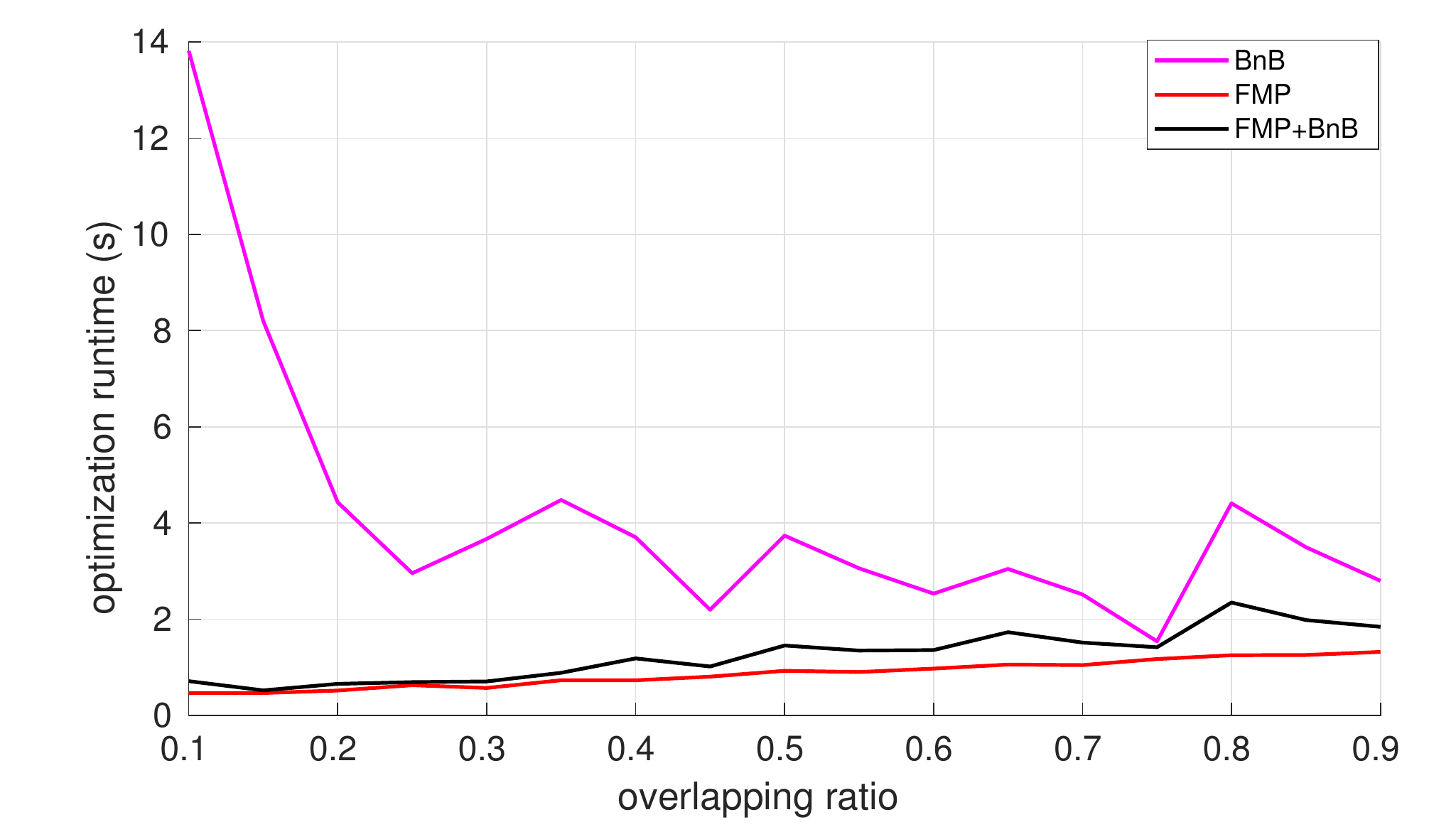}\label{subfig:1DRot_Armadillo}}
\subfigure[Number of matches before and after FMP---\textit{Bunny}.]{\includegraphics[width=0.7\columnwidth]{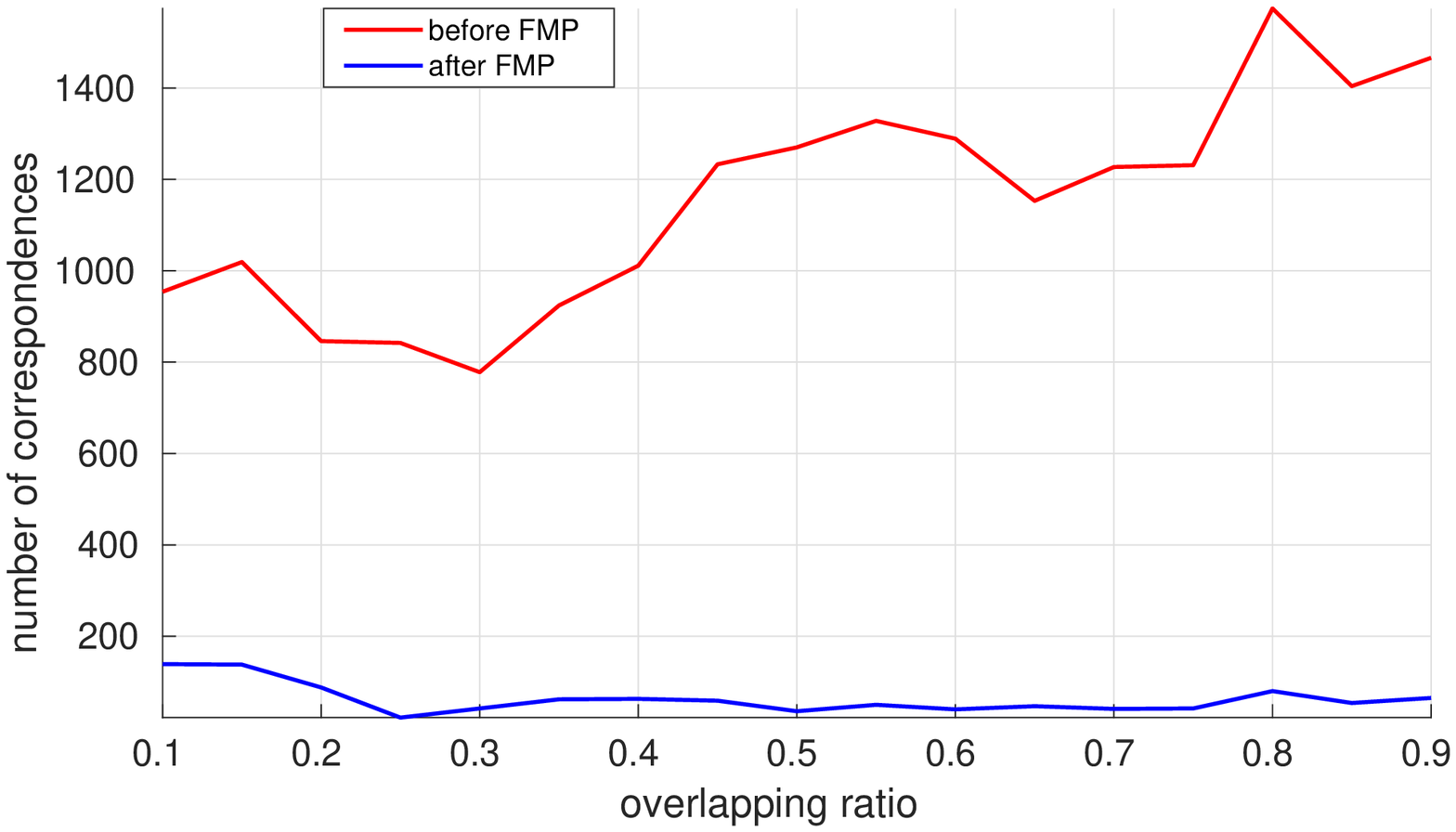}\label{subfig:GORESize_Bunny}}
\subfigure[Number of matches before and after FMP---\textit{Armadillo}.]{\includegraphics[width=0.7\columnwidth]{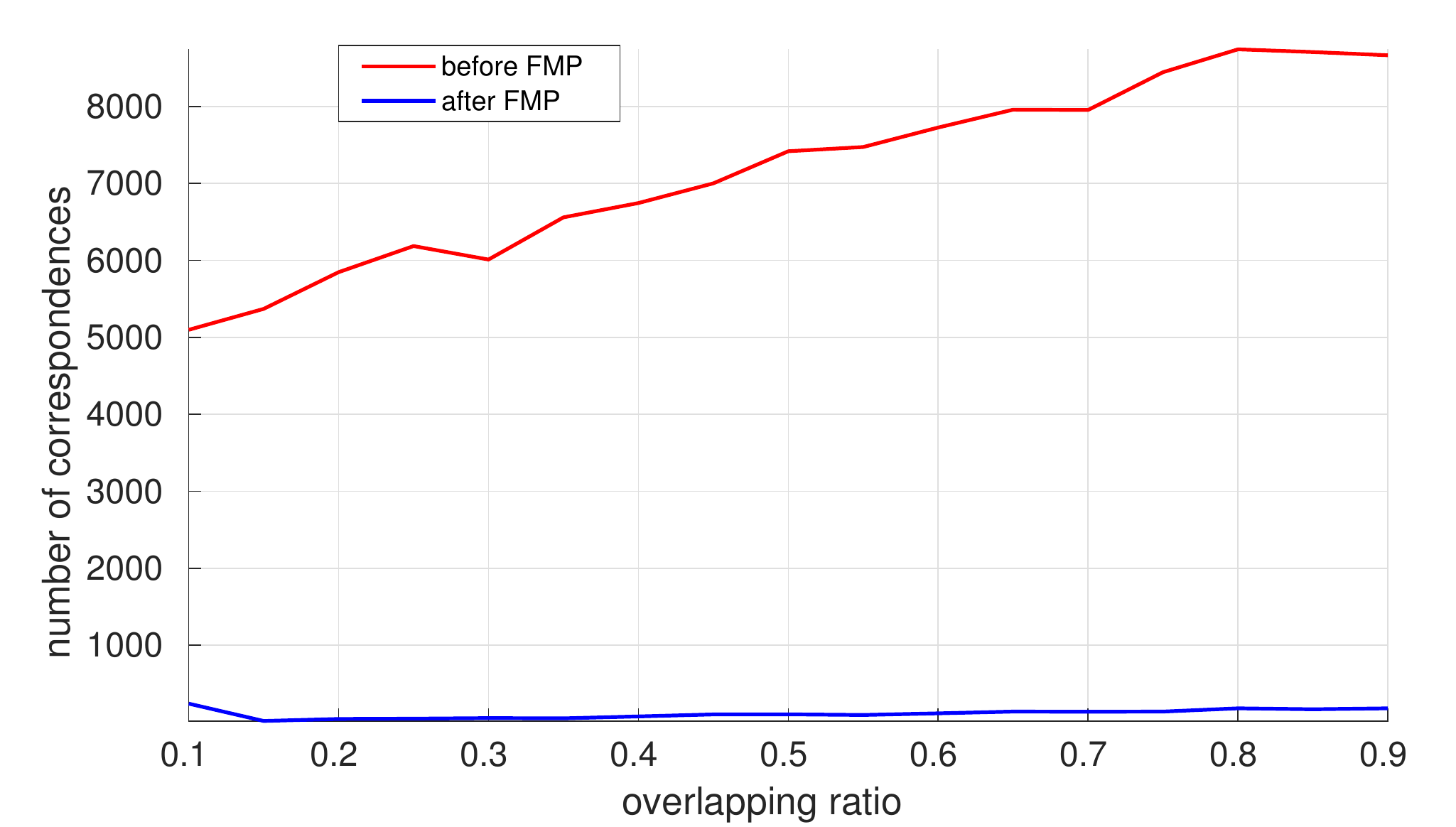}\label{subfig:GORESize_Armadillo}}
\subfigure[Statistics of $|\mathcal{C'}|$ in 100 runs of FMP with shuffled data processing order---\textit{Bunny}.]{\includegraphics[width=0.7\columnwidth]{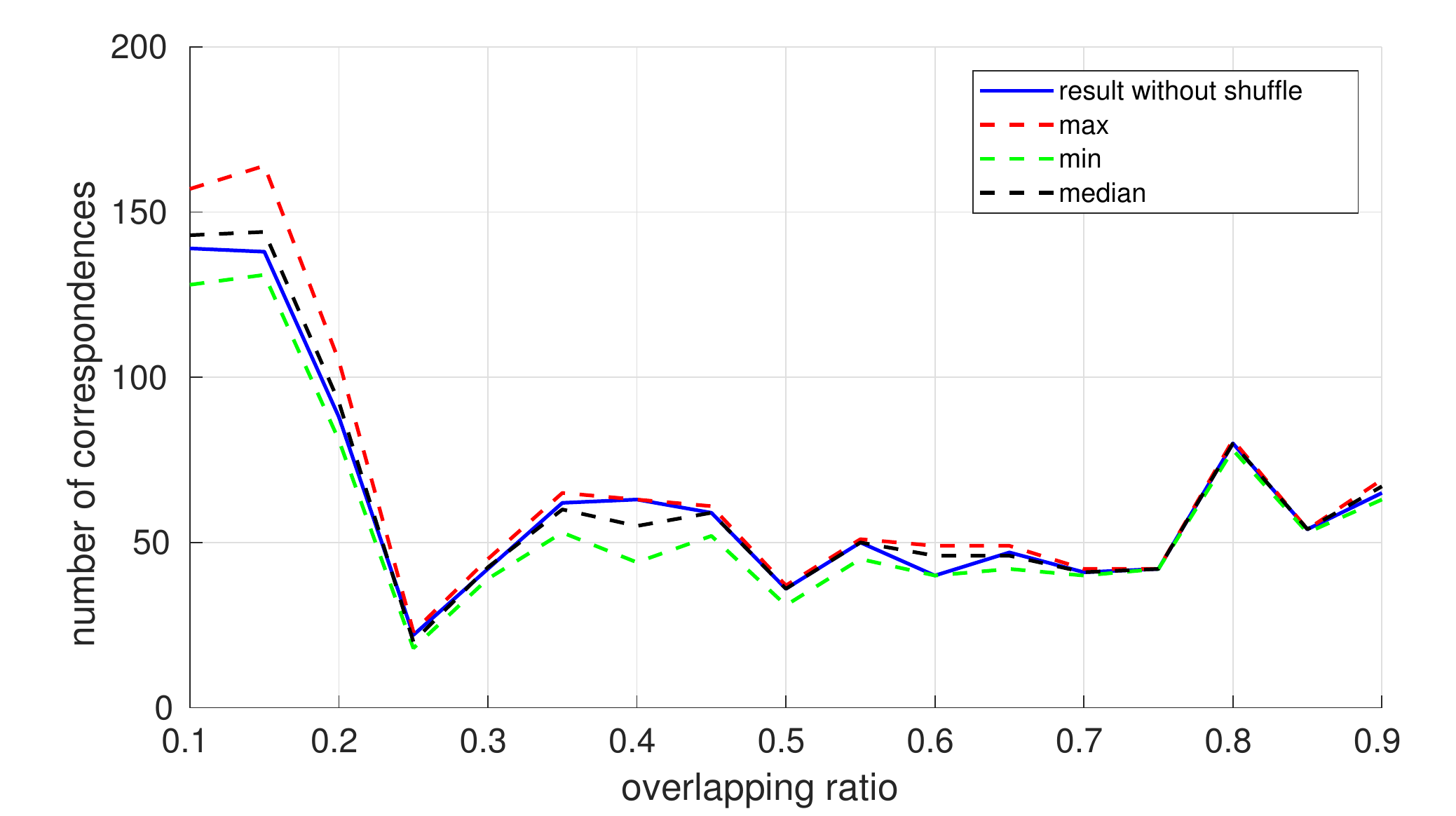}\label{subfig:GOREDep_Bunny}}
\subfigure[Statistics of $|\mathcal{C'}|$ in 100 runs of FMP with shuffled data processing order---\textit{Armadillo}.]{\includegraphics[width=0.7\columnwidth]{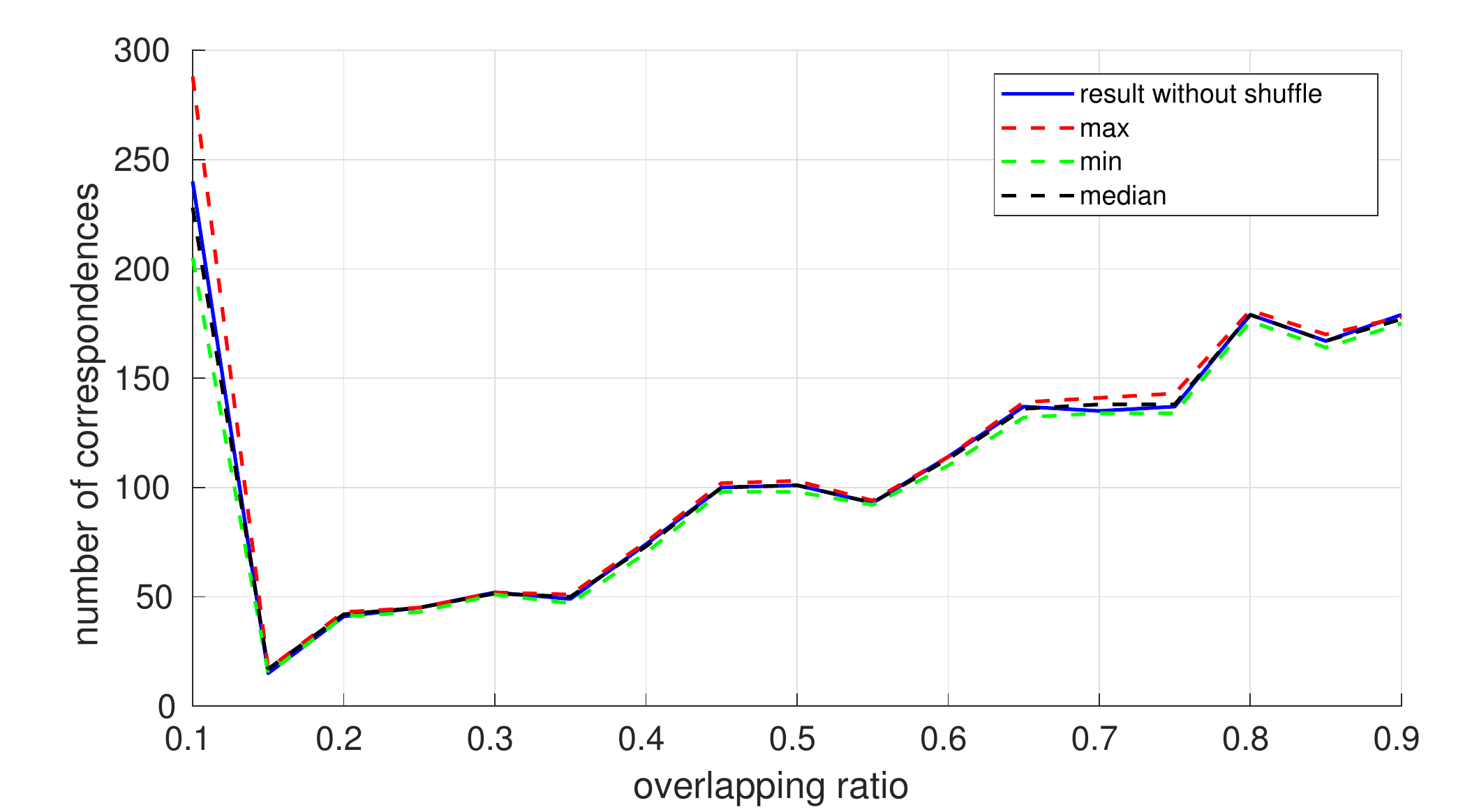}\label{subfig:GOREDep_Armadillo}}
\caption{The effect of FMP on data with varied
  overlapping ratios $\tau$. The runtime of only FMP was reported to show the portion of runtime in FMP + BnB spending on FMP.}\label{fig:1DRot}
\end{figure*}

\begin{table*}[t]\centering
\scriptsize{
\begin{tabular}{|c|c|c|c|c|c|c|c|c|c|c|}
 \hline
 Data & \multicolumn{5}{|c|}{\textit{Bunny}} & \multicolumn{5}{|c|}{\textit{Armadillo}} \\
 \hline
 $\tau$ & $|\bP|$ & $|\bQ|$ & $|\bP_{key}|$ & $|\bQ_{key}|$ & $|\cC|$ &  $|\bP|$ & $|\bQ|$ & $|\bP_{key}|$ & $|\bQ_{key}|$ & $|\cC|$\\
 \hline
0.10 & 15499 & 15499 & 233 & 140 & 954 & 77927 & 77927 & 1410 & 820 & 5097 \\
0.15 & 15918 & 15918 & 248 & 139 & 1019 & 80033 & 80033 & 1478 & 863 & 5369 \\
0.20 & 16360 & 16360 & 257 & 147 & 846 & 82256 & 82256 & 1554 & 922 & 5846 \\
0.25 & 16827 & 16827 & 266 & 153 & 842 & 84606 & 84606 & 1613 & 967 & 6185 \\
0.30 & 17322 & 17322 & 276 & 155 & 778 & 87095 & 87095 & 1676 & 981 & 6010 \\
0.35 & 17847 & 17847 & 290 & 160 & 924 & 89734 & 89734 & 1745 & 1045 & 6558 \\
0.40 & 18405 & 18405 & 298 & 173 & 1011 & 92538 & 92538 & 1809 & 1083 & 6744 \\
0.45 & 18999 & 18999 & 307 & 170 & 1233 & 95524 & 95524 & 1862 & 1136 & 7000 \\
0.50 & 19632 & 19632 & 321 & 175 & 1270 & 98708 & 98708 & 1923 & 1140 & 7417 \\
0.55 & 20309 & 20309 & 332 & 182 & 1328 & 102111 & 102111 & 1969 & 1198 & 7472 \\
0.60 & 21034 & 21034 & 340 & 190 & 1289 & 105758 & 105758 & 2033 & 1216 & 7725 \\
0.65 & 21813 & 21813 & 348 & 193 & 1153 & 109675 & 109675 & 2053 & 1284 & 7957 \\
0.70 & 22652 & 22652 & 380 & 200 & 1227 & 113894 & 113894 & 2114 & 1284 & 7953 \\
0.75 & 23558 & 23558 & 390 & 205 & 1231 & 118449 & 118449 & 2160 & 1360 & 8444 \\
0.80 & 24540 & 24540 & 400 & 221 & 1574 & 123385 & 123385 & 2253 & 1407 & 8741 \\
0.85 & 25607 & 25607 & 413 & 228 & 1404 & 128749 & 128749 & 2297 & 1441 & 8706 \\
0.90 & 26771 & 26771 & 426 & 227 & 1466 & 134602 & 134602 & 2352 & 1475 & 8664 \\
 \hline
\end{tabular}
}
\caption{Size of the controlled data. $|\bP|$ and $|\bQ|$: size of the
  input point clouds. $|\bP_{key}|$ and $|\bQ_{key}|$: number of
  keypoints.}\label{tab:dataSynth}
\end{table*}

\begin{figure*}[!htb]\centering
\subfigure[Rotation error---\textit{Bunny}.]{\includegraphics[width=0.7\columnwidth]{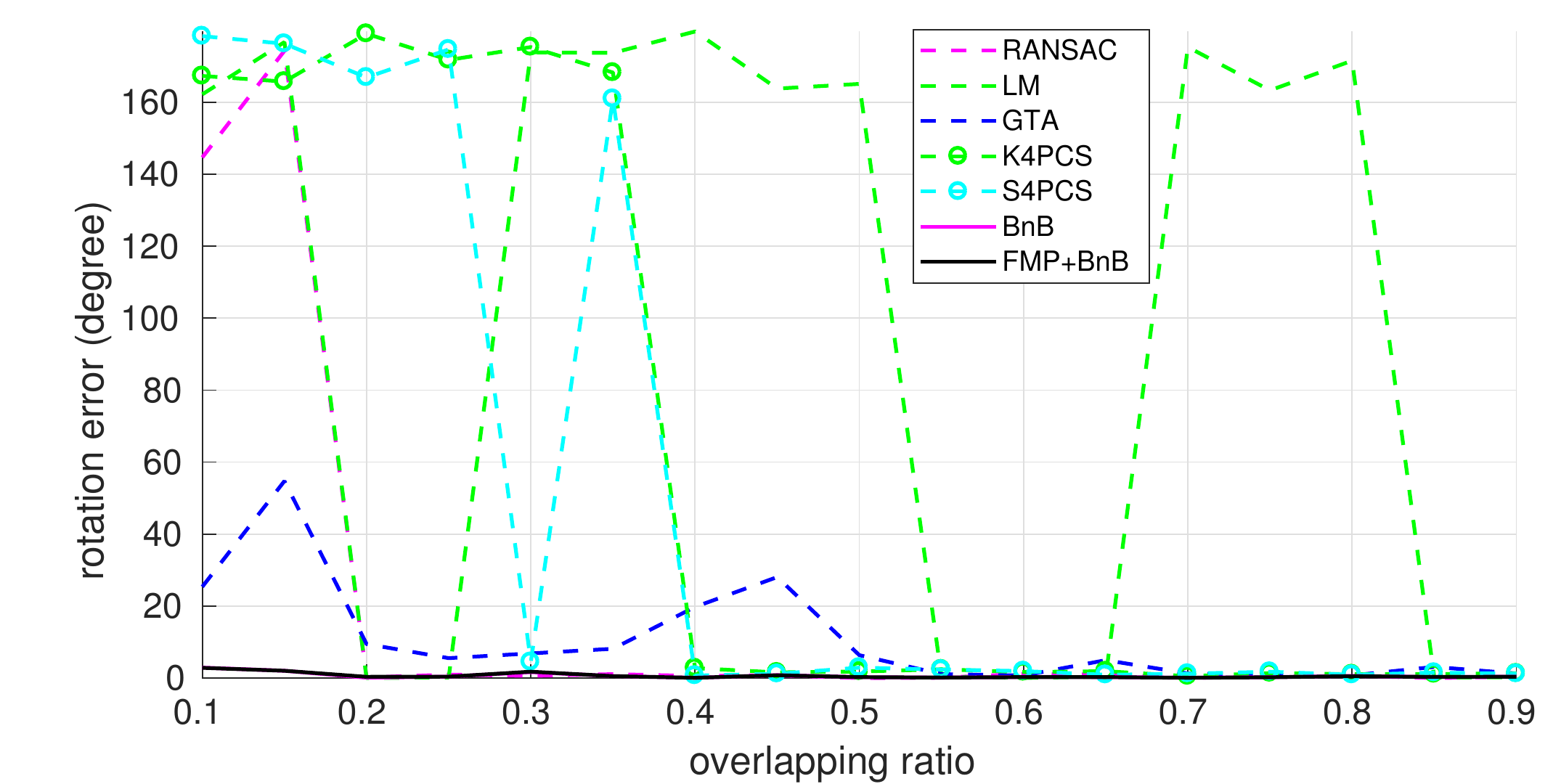}\label{subfig:Reg_Bunny_eAng}}
\subfigure[Rotation error---\textit{Armadillo}.]{\includegraphics[width=0.7\columnwidth]{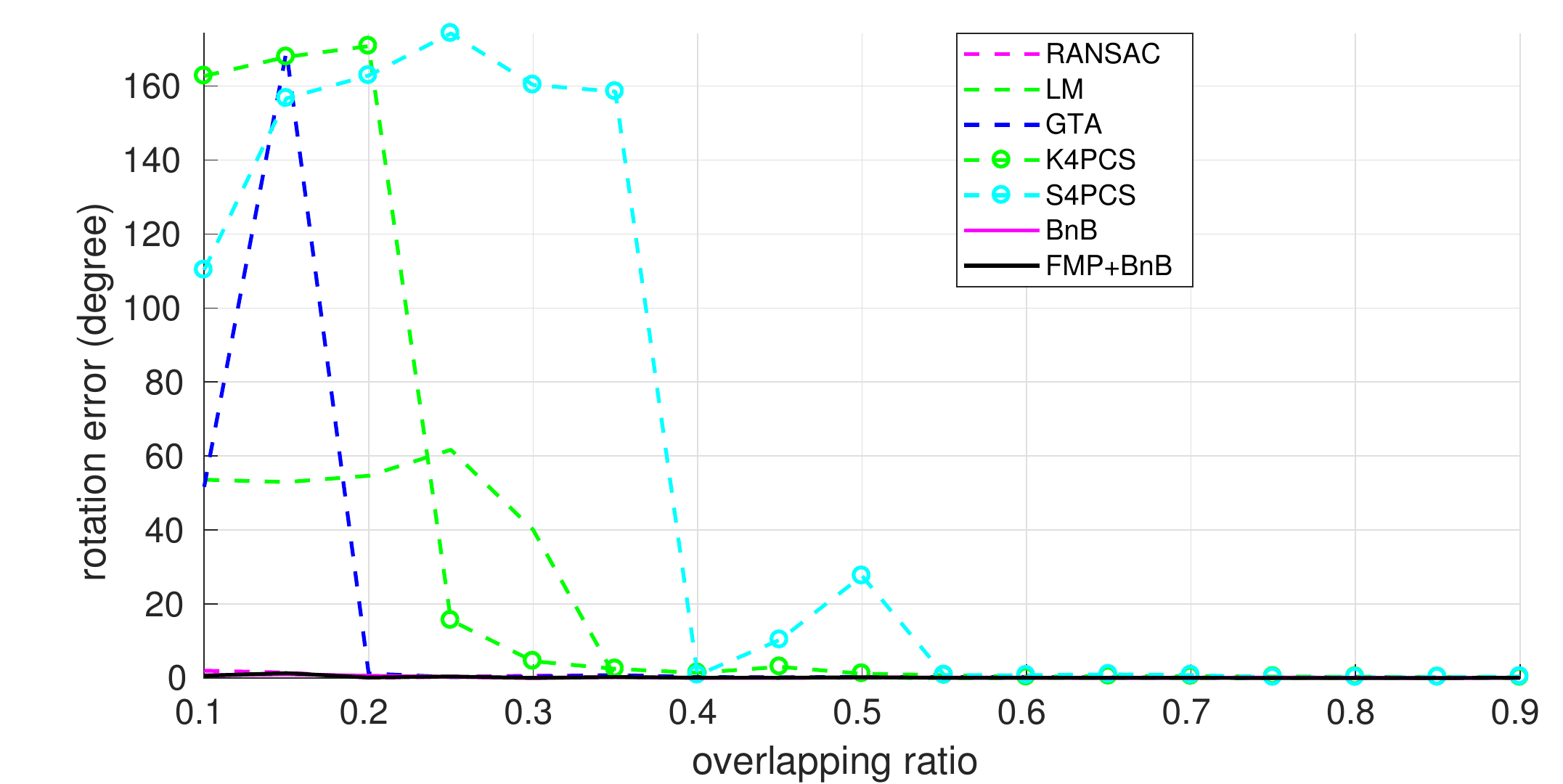}\label{subfig:Reg_Armadillo_eAng}}
\subfigure[Translation error---\textit{Bunny}.]{\includegraphics[width=0.7\columnwidth]{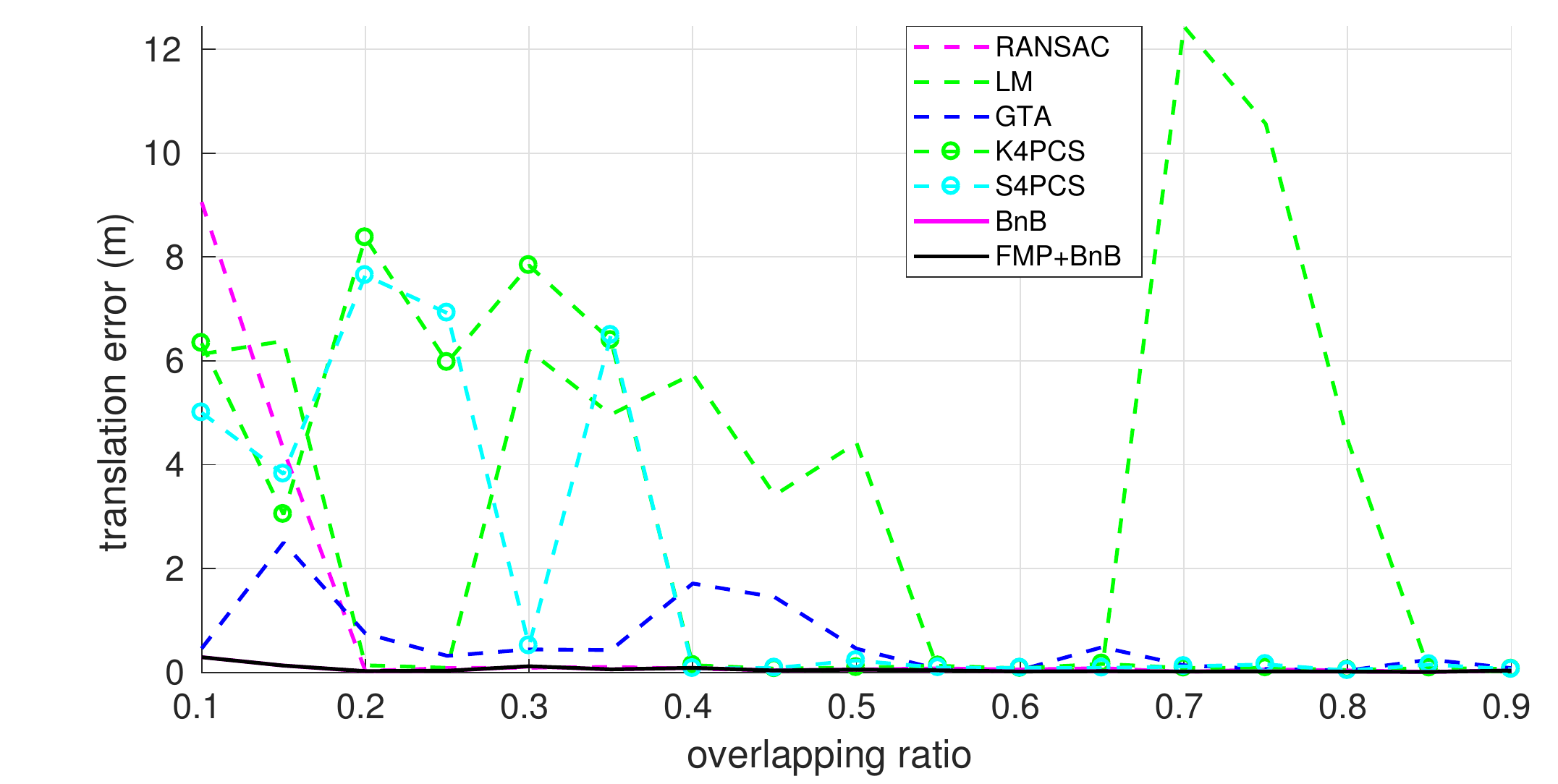}\label{subfig:Reg_Bunny_eTran}}
\subfigure[Translation error---\textit{Armadillo}.]{\includegraphics[width=0.7\columnwidth]{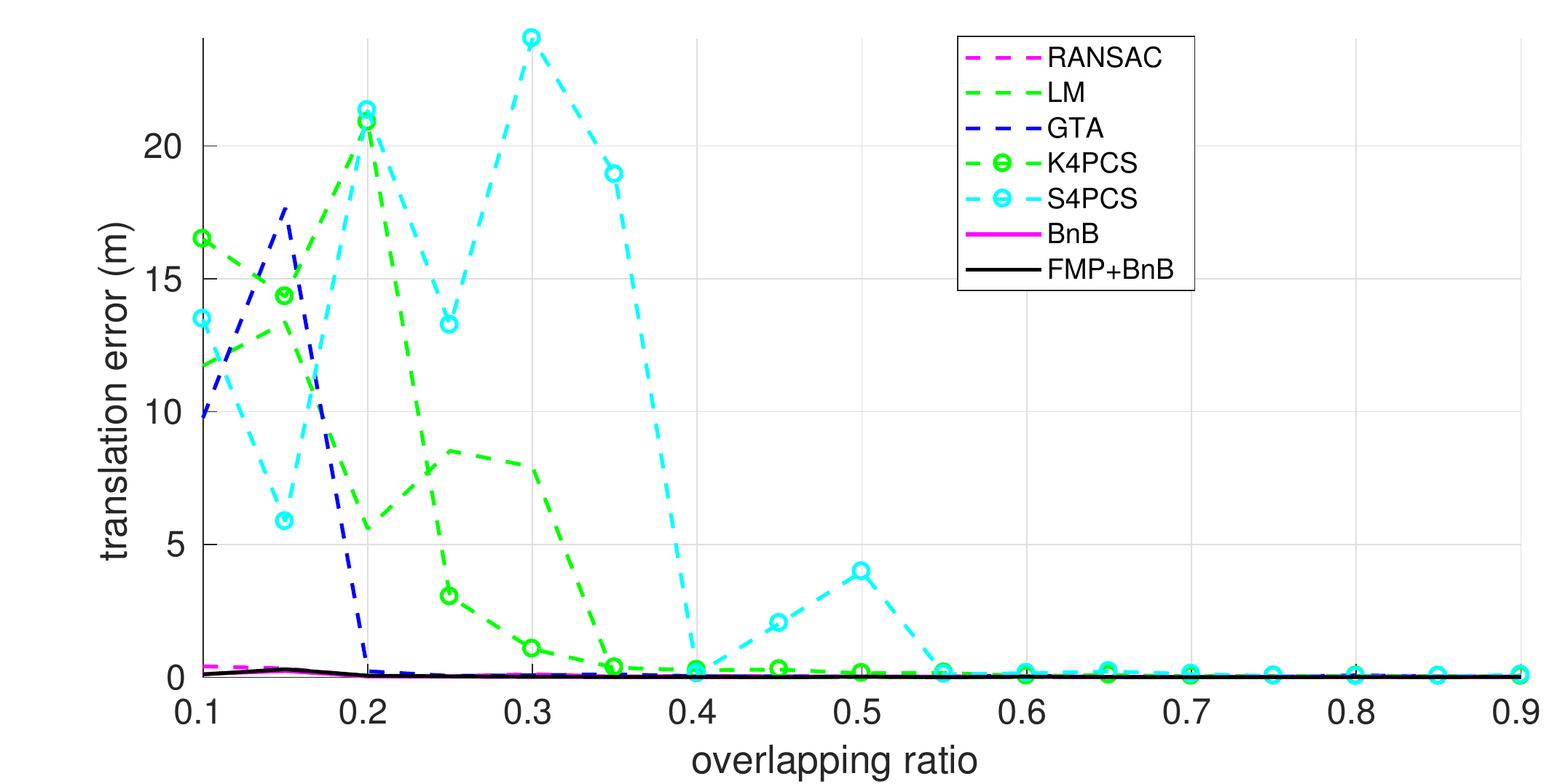}\label{subfig:Reg_Armadillo_eTran}}
\caption{The registration accuracy of all methods on data with varied overlapping ratios $\tau$.}\label{fig:Reg_Acc_Synth}
\end{figure*}

\subsection{Controlled data}\label{sec:expSynth}

To refer to our fast match pruning step
(Algorithm~\ref{alg:GORE}), we abbreviate it as FMP in the rest of
this paper. We first show the effect of FMP to the speed of our
method, by comparing
\begin{itemize}
\item FMP + BnB: Algorithm~\ref{alg:BnB1} with FMP for preprocessing;
\item BnB: Algorithm~\ref{alg:BnB1} without preprocessing;
\item FMP: Algorithm~\ref{alg:GORE};   
\end{itemize}
on data with varied overlap ratios $\tau$. Varying $\tau$ showed the
performance for different outlier rates. And to preserve the effect of
feature matching, we chose to manipulate $\tau$ instead of the outlier
rate directly. $\tau$ was controlled by sampling two subsets from a
complete point cloud (\textit{Bunny} and
\textit{Armadillo}\footnote{\url{http://graphics.stanford.edu/data/3Dscanrep/}}
in this experiment). Each subset contains $\floor{\frac{2}{4-2\tau}}$ of all
points, and the second subset is displaced relative to the first one
with a randomly generated 4DOF transformation. (3D translation and
rotation around the vertical). Moreover, the point clouds are rescaled
to have an average point-to-point distance of 0.05 m, and contaminated
with uniform random noise of magnitude $[0\hdots 0.05]$ m.

Given two point clouds, the initial match set $\cC$ was
generated (here and in all subsequent experiments) by
\begin{enumerate}
\item Voxel Grid~\citep{lee2001data} down-sampling and
  ISS~\citep{zhong2009intrinsic} keypoint extraction;
\item FPFH feature~\citep{rusu2009fast} computation and matching on
  keypoints. $\bp_i$ and $\bq_i$ were selected into $\cC$ if their
  FPFH features are one of the $\lambda$ nearest neighbours \emph{to each
    other}. Empirically, $\lambda$ needs to be a bit larger than 1 to
  generate enough inliers. We set $\lambda=10$ in all experiments.

\end{enumerate}
The inlier threshold $\epsilon$ was set to 0.05 m according to the
noise level.

Figure~\ref{fig:1DRot} reports the runtime of the three algorithms and
the number of matches before and after FMP, with $\tau$ varied
from 0.1 to 0.9. See Table~\ref{tab:dataSynth} for the input size. As
shown in Figure~\ref{subfig:1DRot_Bunny}
and~\ref{subfig:1DRot_Armadillo}, due to the extremely high outlier
rate, BnB was much slower when $\tau$ is small, whereas FMP + BnB
remained efficient for all $\tau$. This significant acceleration came
from the drastically reduced input size (more than 90\% of the
outliers were pruned) after executing FMP
(Figure~\ref{subfig:GORESize_Bunny}
and~\ref{subfig:GORESize_Armadillo}), and the extremely low overhead
of FMP. 

Note that the data storing order of $\mathcal{C}$ was used as the data processing order (the order of $k$ in the for-loop of Algorithm{~\ref{alg:GORE}}) in FMP for all registration tasks in this paper. Though theoretically the data processing order of FMP does affect the size $|\mathcal{C}'|$ of the pruned match set, in practice, this effect is minor. To show the stability of FMP under different data processing orders, we executed FMP 100 times given the same input match set, but with randomly shuffled data processing order in each FMP execution. Figures{~\ref{subfig:GOREDep_Bunny}} and{~\ref{subfig:GOREDep_Armadillo}} report the median, minimum and maximum value of $|\mathcal{C'}|$ in 100 runs, and the value of $|\mathcal{C'}|$ using the original data storing order as the data processing order. As can be seen, the value $|\mathcal{C'}|$ is stable across all instances.

\begin{figure*}[t]\centering
\subfigure[Runtime for optimization---\textit{Bunny}.]{\includegraphics[width=0.7\columnwidth]{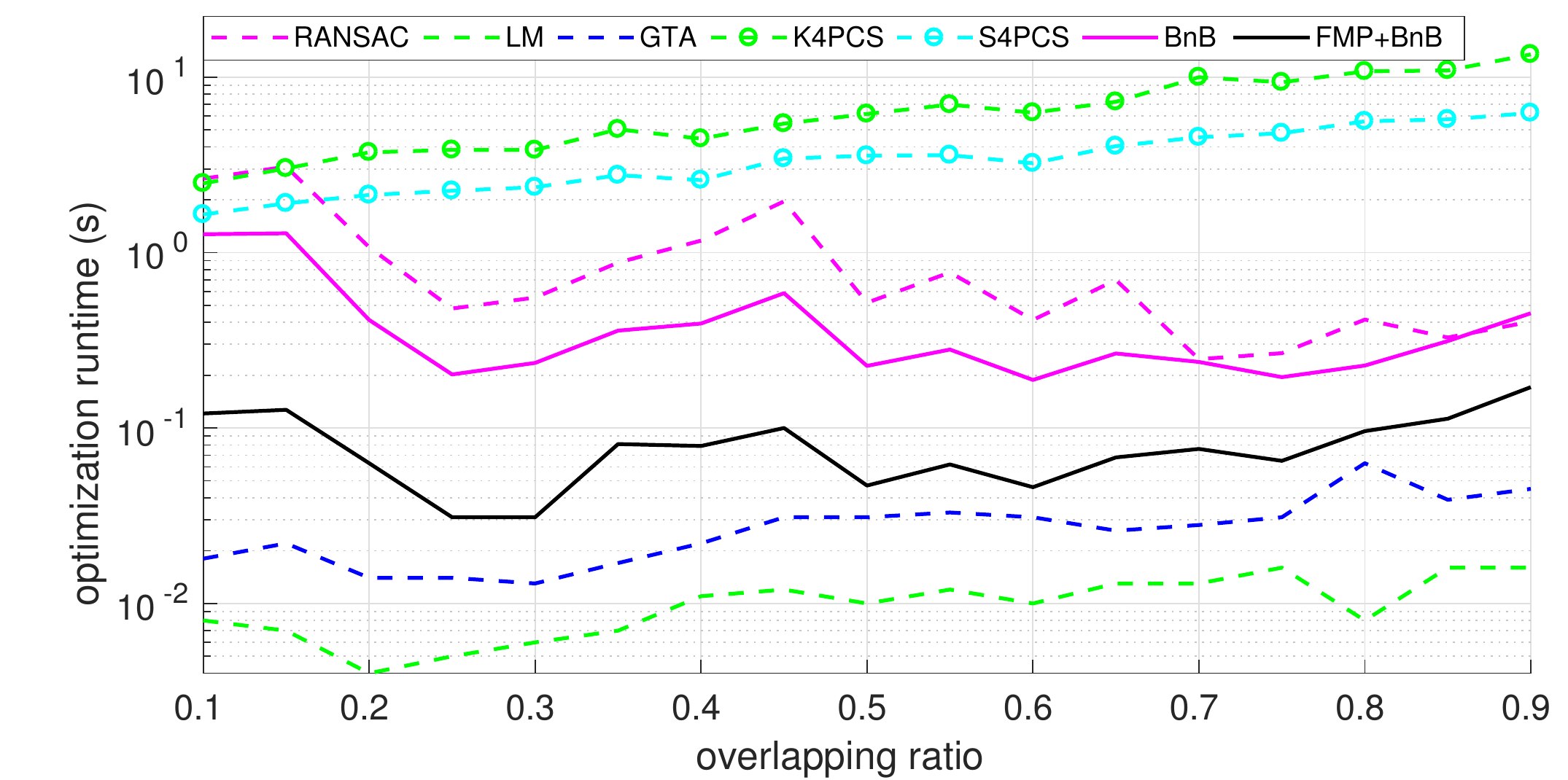}\label{subfig:Reg_Bunny_regTime}}
\subfigure[Runtime for optimization---\textit{Armadillo}.]{\includegraphics[width=0.7\columnwidth]{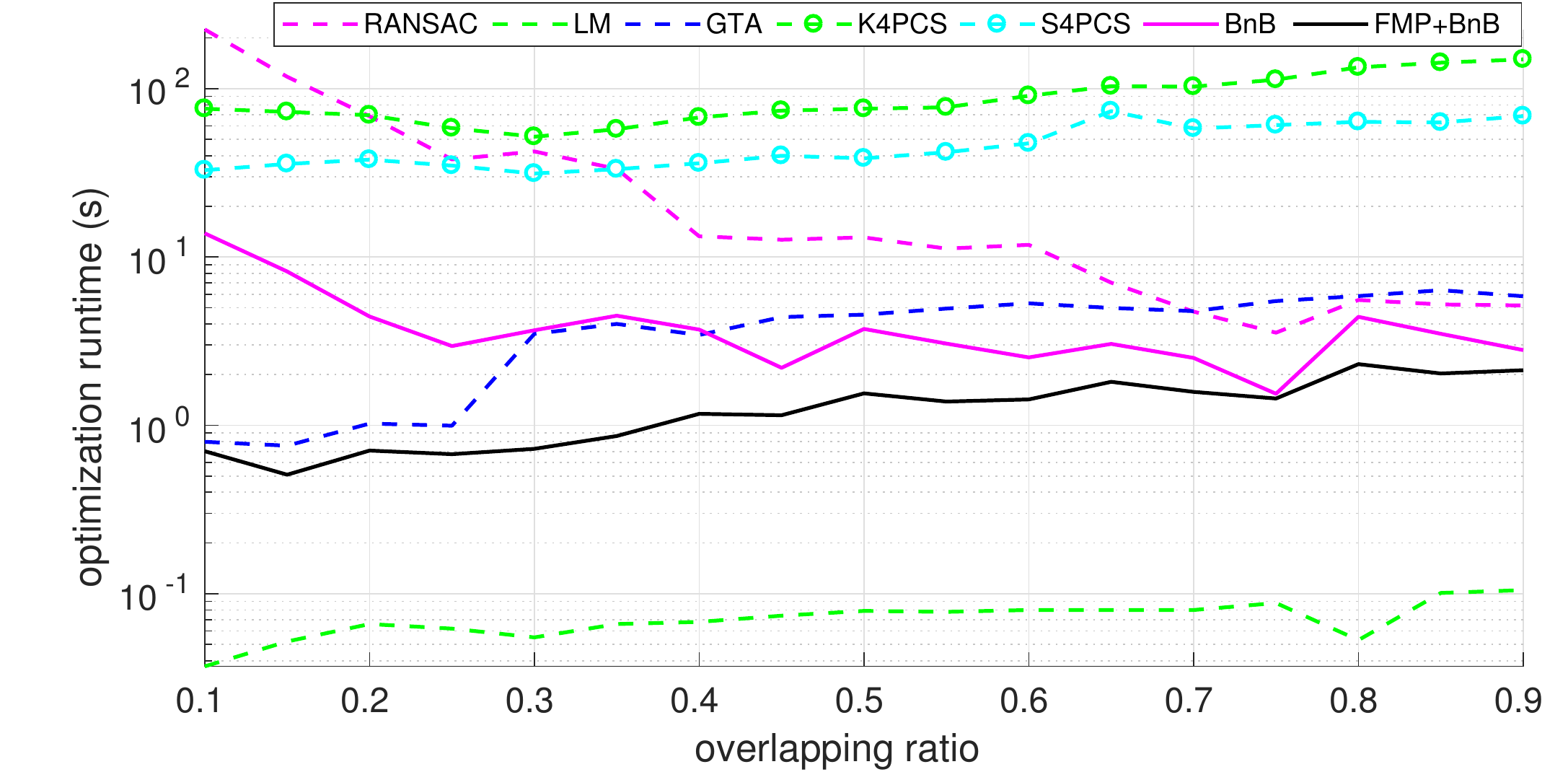}\label{subfig:Reg_Armadillo_regTime}}
\subfigure[Total runtime---\textit{Bunny}.]{\includegraphics[width=0.7\columnwidth]{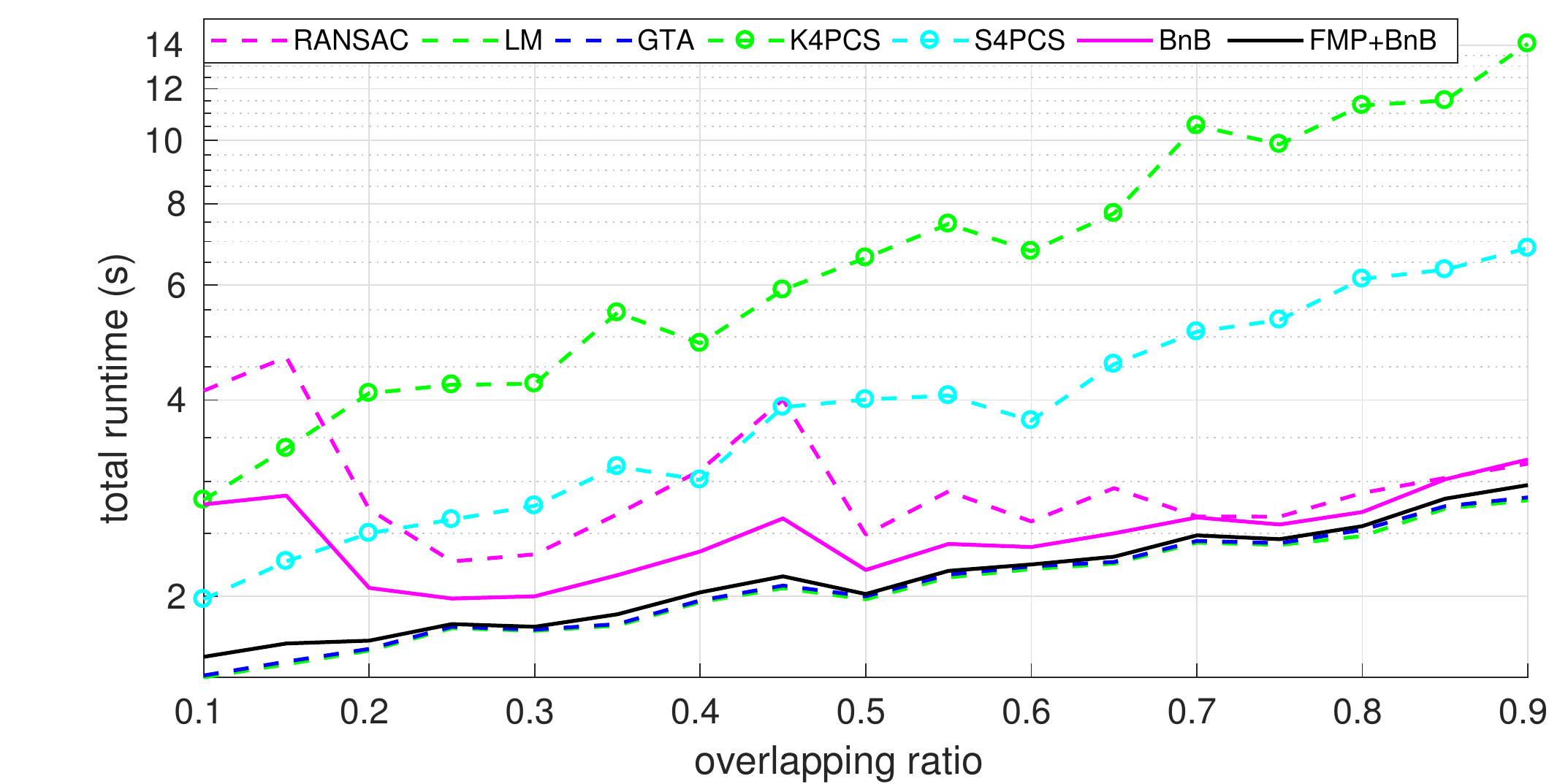}\label{subfig:Reg_Bunny_regTimeWithPrep}}
\subfigure[Total runtime---\textit{Armadillo}.]{\includegraphics[width=0.7\columnwidth]{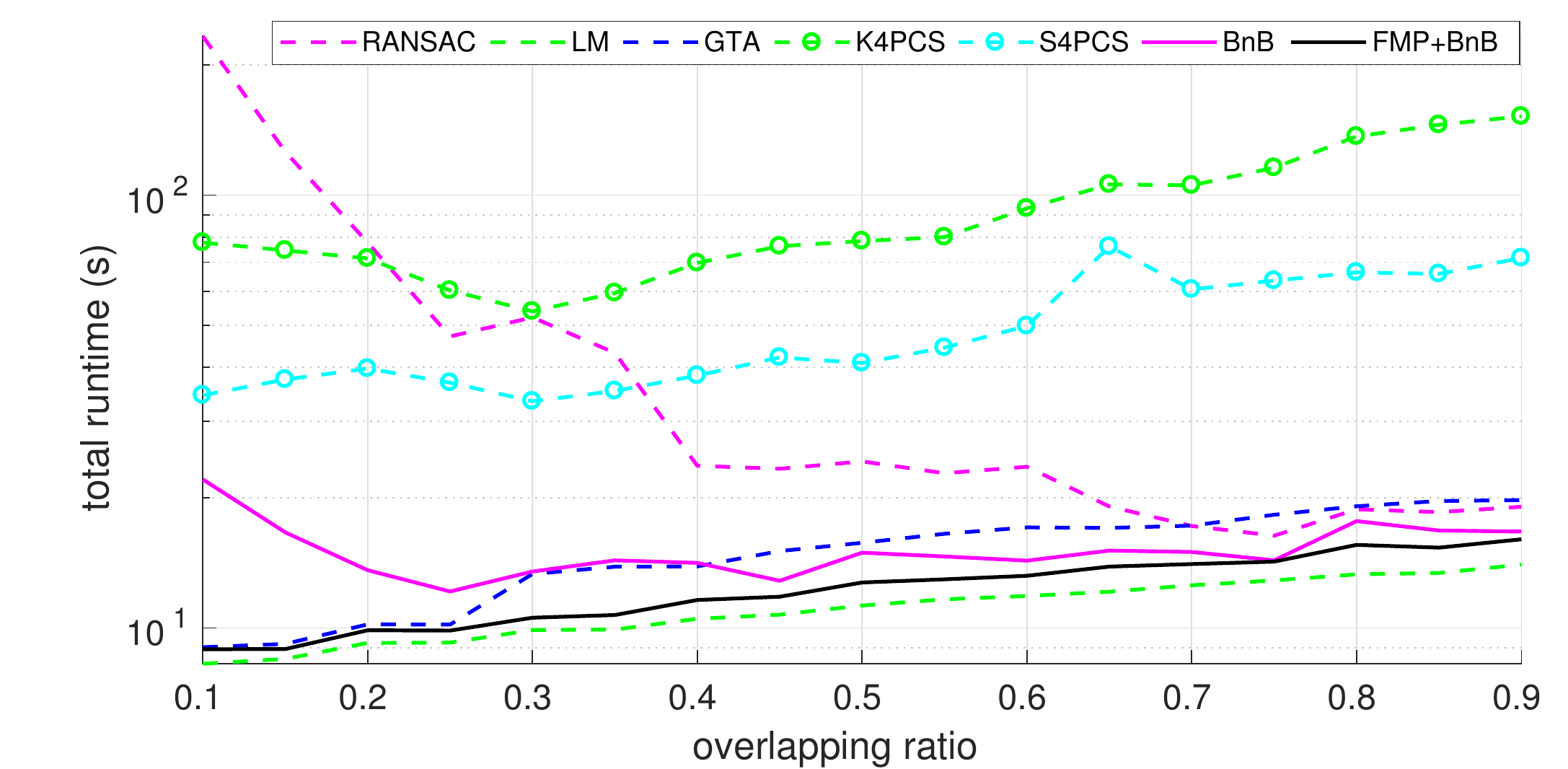}\label{subfig:Reg_Armadillo_regTimeWithPrep}}
\caption{The log scaled runtime of all methods with (up) and without (bottom) input genreration (only generating keypoints for K4PCS and S4PCS) on controlled data.}\label{fig:Reg_Time_Synth}
\end{figure*}

\begin{figure*}[!htb]\centering
\subfigure[\textit{Bunny}.]{\includegraphics[width=0.6\columnwidth]{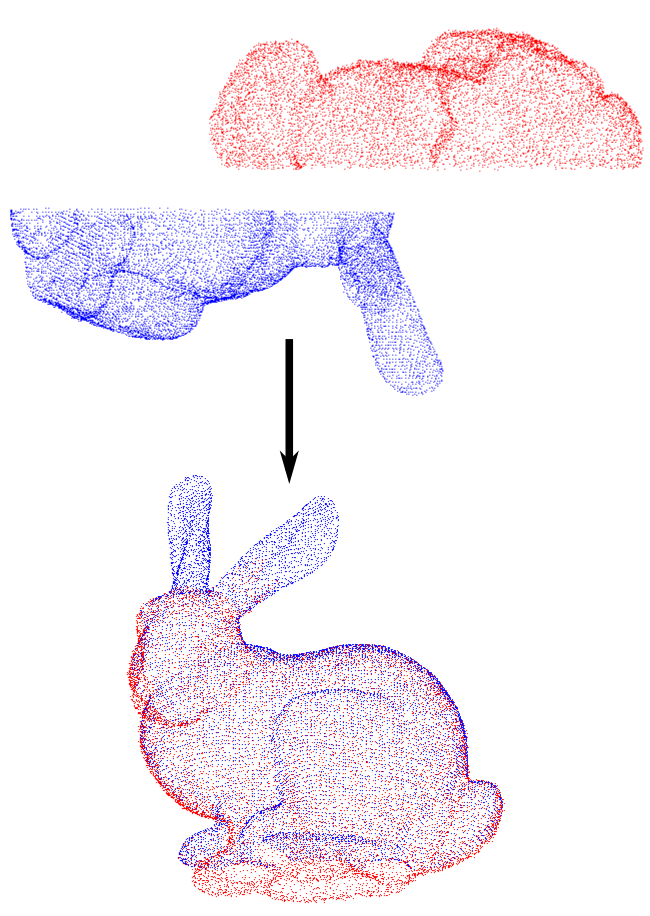}\label{subfig:Reg_Bunny_Vis}}
\subfigure[\textit{Armadillo}.]{\includegraphics[width=0.6\columnwidth]{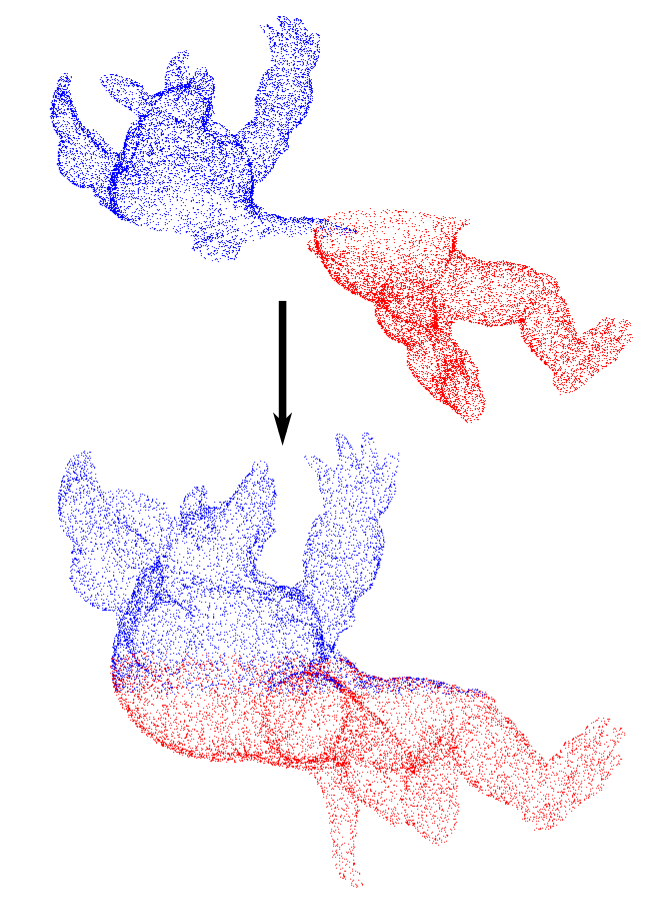}\label{subfig:Reg_Armadillo_Vis}}
\subfigure[\textit{Dragon}.]{\includegraphics[width=0.6\columnwidth]{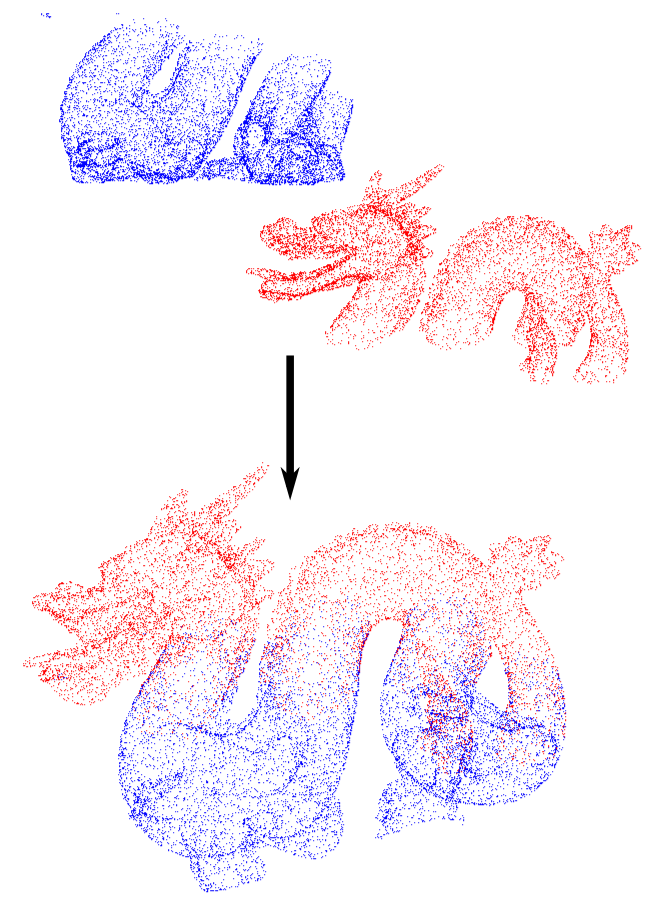}\label{subfig:Reg_Dragon_Vis}}
\caption{Registered controlled data ($\tau = 0.1$). 10k points are shown for each point cloud.}\label{fig:Reg_Synth_Vis}
\end{figure*}

To show the advantage of our method, we also compared it against the
following state-of-the-art approaches, on the same set of data.

\begin{itemize}
\item 4DOF RANSAC~\citep{fischler1981random}, using 2-point samples for
  4DOF pose estimation,and with the probability of finding a
  valid sample set to 0.99\footnote{C++ code in
    \url{https://bitbucket.org/Zhipeng_Cai/isprsjdemo/src/default/}.};
\item 4DOF version of the lifting method (LM)~\citep{zhou2016fast}, a
  match-based robust optimization approach\footnote{C++ implementation based on the code from
    \url{http://vladlen.info/publications/fast-global-registration/}.}. The
  annealing rate was tuned to 1.1 for the best performance;
\item The Game-Theory approach (GTA)~\citep{albarelli2010game}, a fast
  outlier removal method\footnote{C++ implementation based on the
    code from \url{http://www.isi.imi.i.u-tokyo.ac.jp/~rodola/sw.html}
    and
    \url{http://vision.in.tum.de/_media/spezial/bib/cvpr12-code.zip}.};
\item Super 4PCS (S4PCS)~\citep{mellado2014super}, a fast
  4PCS~\citep{aiger20084} variant\footnote{C++ code from
    \url{http://geometry.cs.ucl.ac.uk/projects/2014/super4PCS/}}.
\item Keypoint-based 4PCS (K4PCS)~\citep{theiler2014keypoint}, which
  applies 4PCS to keypoints\footnote{C++ code from PCL:
    \url{http://pointclouds.org/}}.
\end{itemize}
The first three approaches are match-based and the last two
operates on raw point sets (ISS keypoints were used here). The 4DOF
version of RANSAC and LM were used for a fair comparison, since they
had similar accuracy but were much faster than their 6DOF
counterparts.
We note that, when working with levelled point clouds, the translation
alone must already align the $z$-coordinates of two points up to the
inlier threshold $\epsilon$. This can be checked before
sampling/applying the rotation. Where applicable, we use this trick to
save computations (for all methods).
For GTA as well as 4PCS variants, the original versions for 6DOF had
to be used, since it is not obvious how to constrain the underlying
algorithms to 4DOF.

Figure~\ref{fig:Reg_Acc_Synth} shows the accuracy of all methods,
which was measured as the difference between the estimated
transformation and the known ground truth. As expected, BnB and
FMP + BnB returned high quality solutions for all setups. In contrast,
due to the lack of optimality guarantees, other methods performed
badly when $\tau$ was small. Meanwhile,
Figure~\ref{fig:Reg_Time_Synth} shows both the runtime including and
not including the input generation (keypoint extraction and/or feature
matching). FMP + BnB was faster than most of its competitors. Note that
most of the total runtime was spent on input generation.  Other than
sometimes claimed, exact optimization is not necessarily slow and does
in fact not create a computational bottleneck for registration.
Figure~\ref{fig:Reg_Synth_Vis} shows some visual examples of point
clouds aligned with our method.

\begin{figure*}[!htb]\centering
\subfigure[Rotation error---\textit{Arch}.]{\includegraphics[width=0.8\columnwidth]{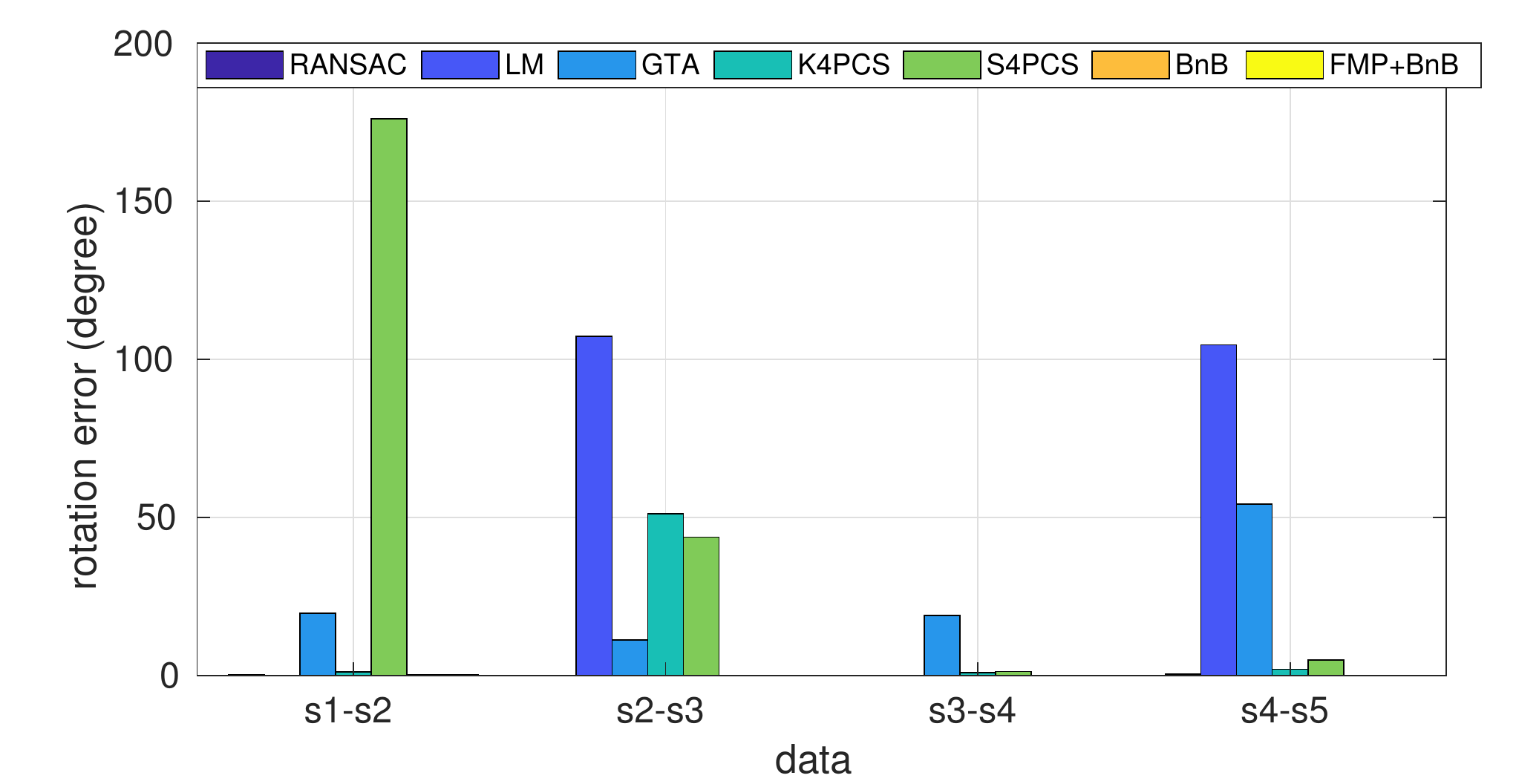}\label{subfig:Reg_Arch_eAng}}
\subfigure[Rotation error---\textit{Facility}.]{\includegraphics[width=0.8\columnwidth]{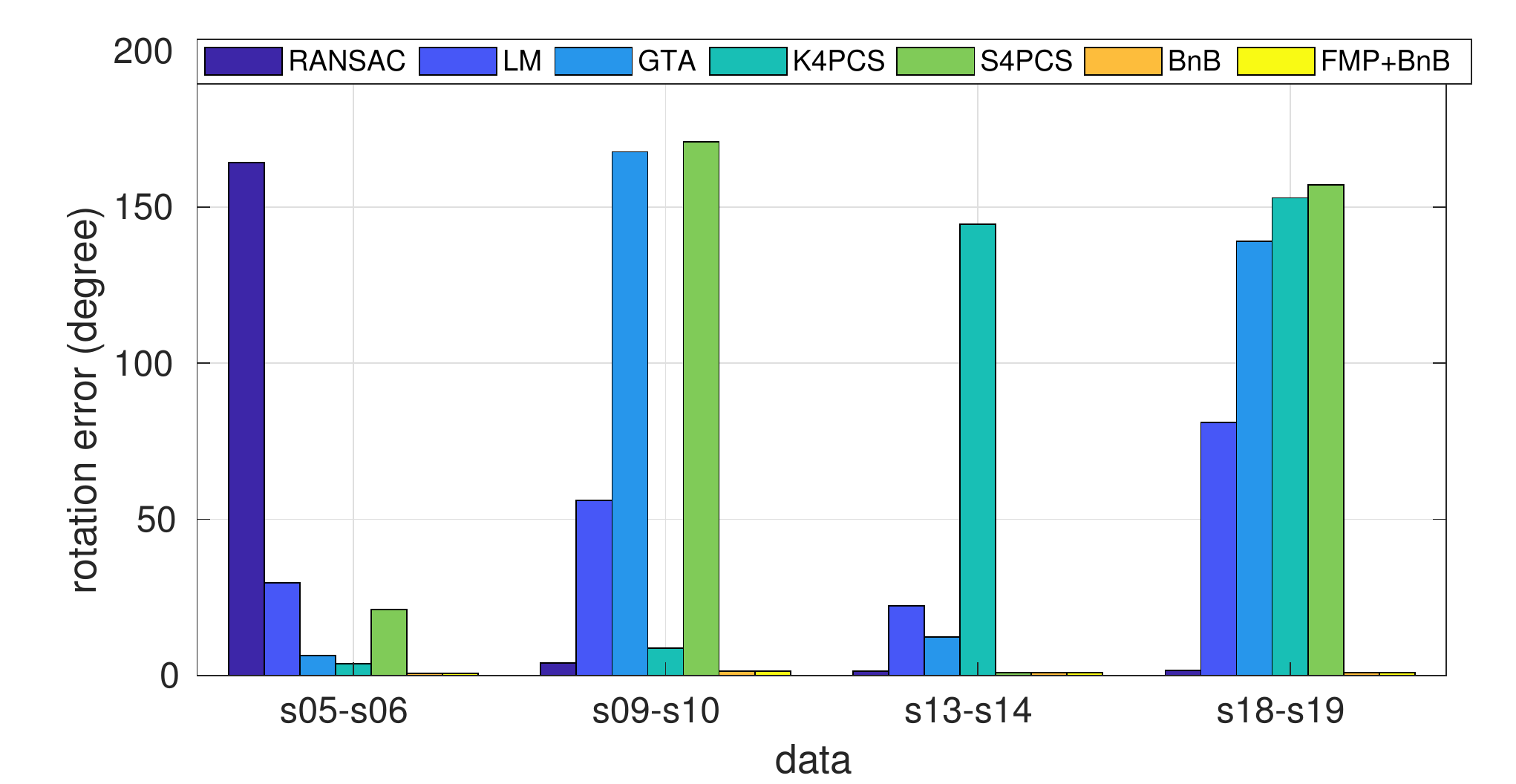}\label{subfig:Reg_Facility_eAng}}
\subfigure[Translation error---\textit{Arch}.]{\includegraphics[width=0.8\columnwidth]{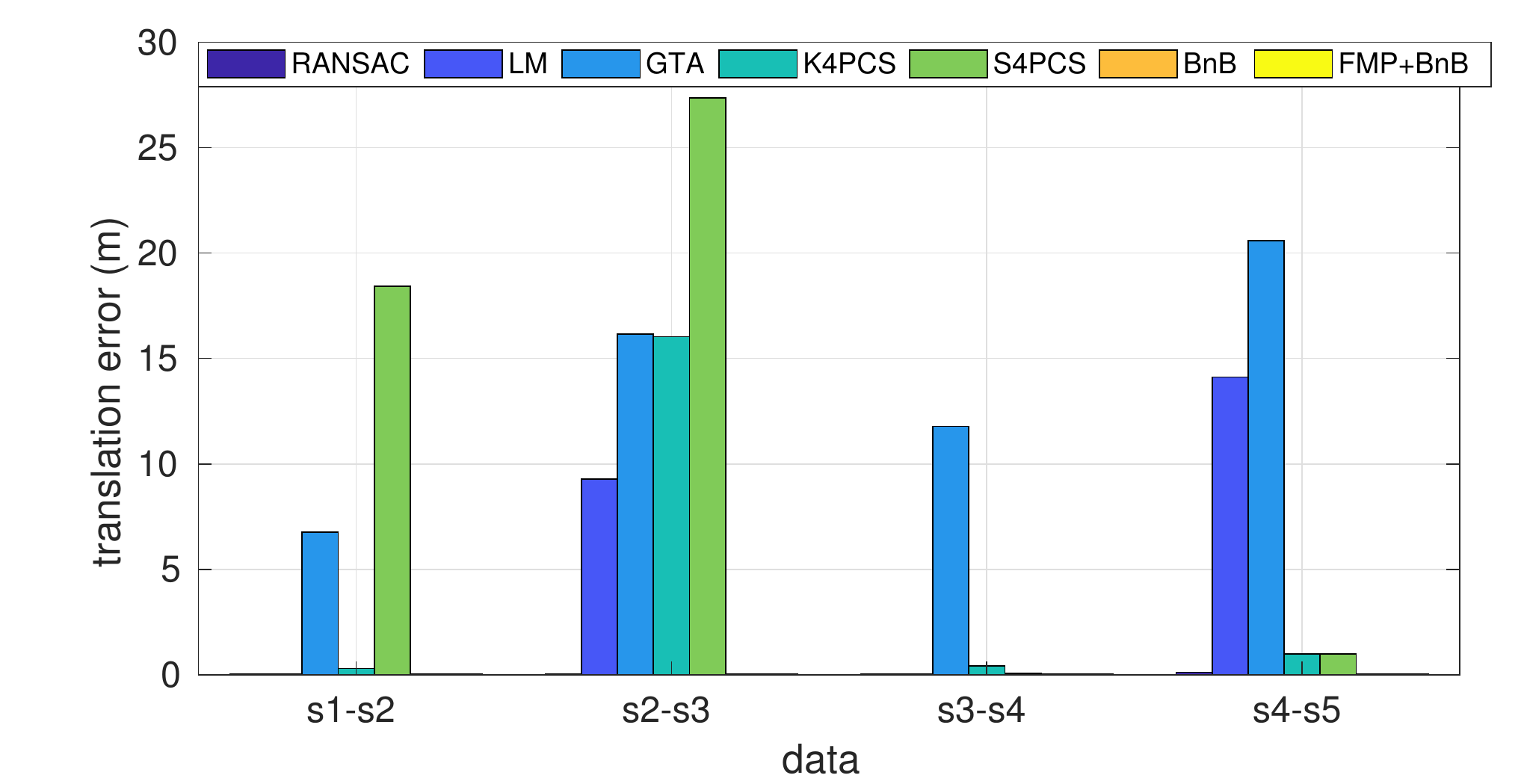}\label{subfig:Reg_Arch_eTran}}
\subfigure[Translation error---\textit{Facility}.]{\includegraphics[width=0.8\columnwidth]{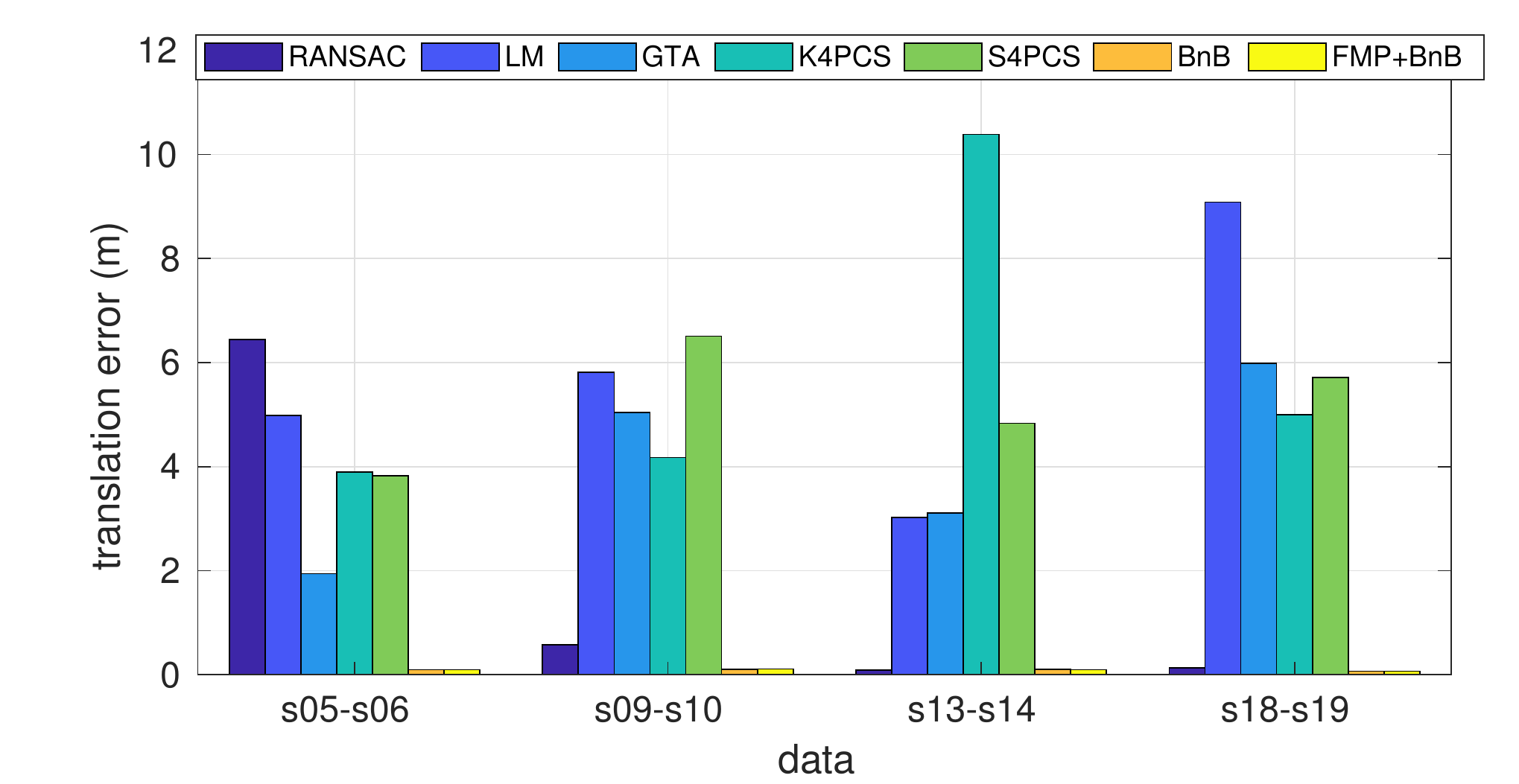}\label{subfig:Reg_Facility_eTran}}
\caption{The accuracy of all registration methods on real-world data.}\label{fig:Reg_Acc_Real}
\end{figure*}

\begin{figure*}[!htb]\centering
\subfigure[Runtime for optimization---\textit{Arch}.]{\includegraphics[width=0.8\columnwidth]{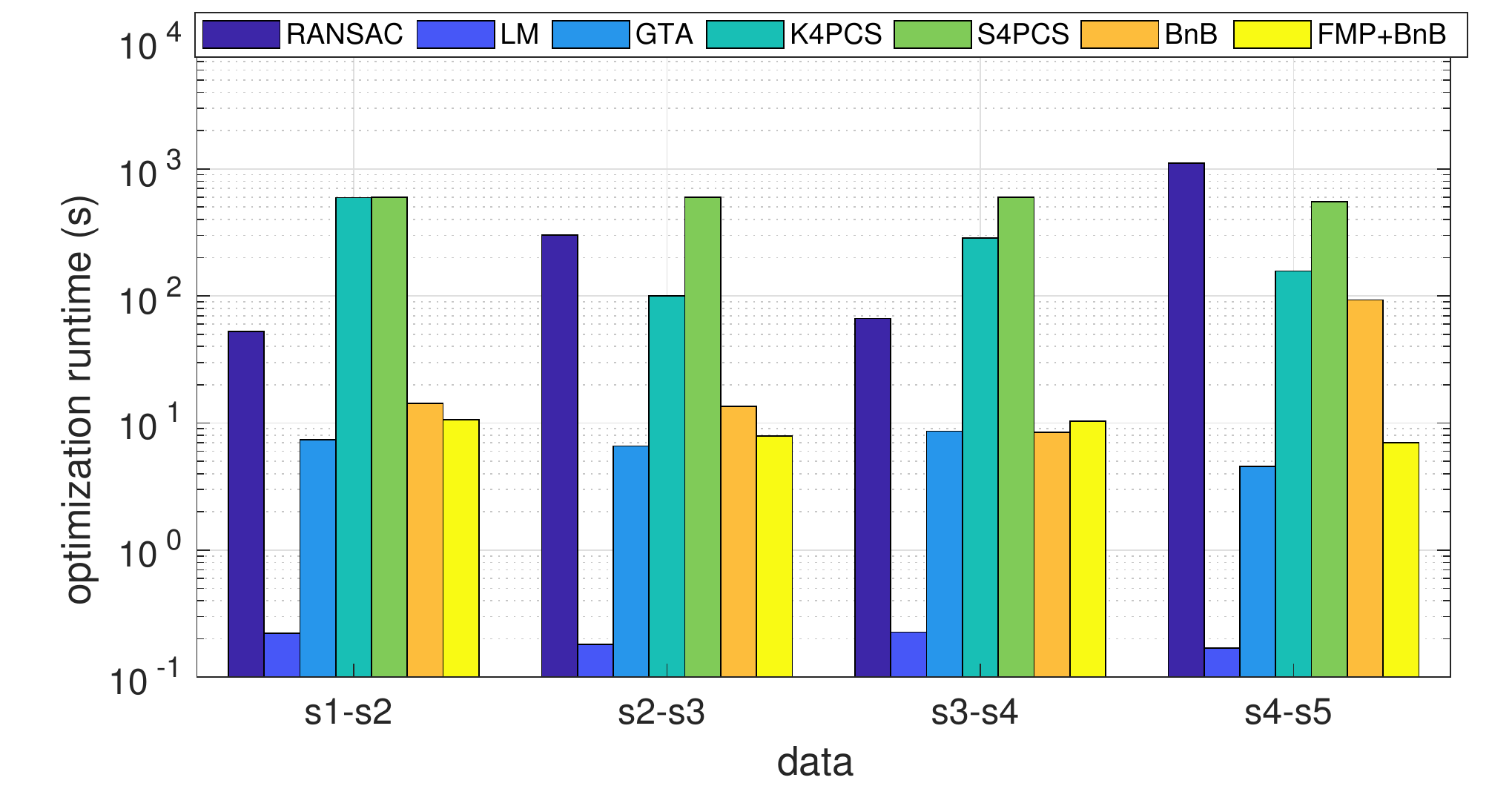}\label{subfig:Reg_Arch_regTime}}
\subfigure[Runtime for optimization---\textit{Facility}.]{\includegraphics[width=0.8\columnwidth]{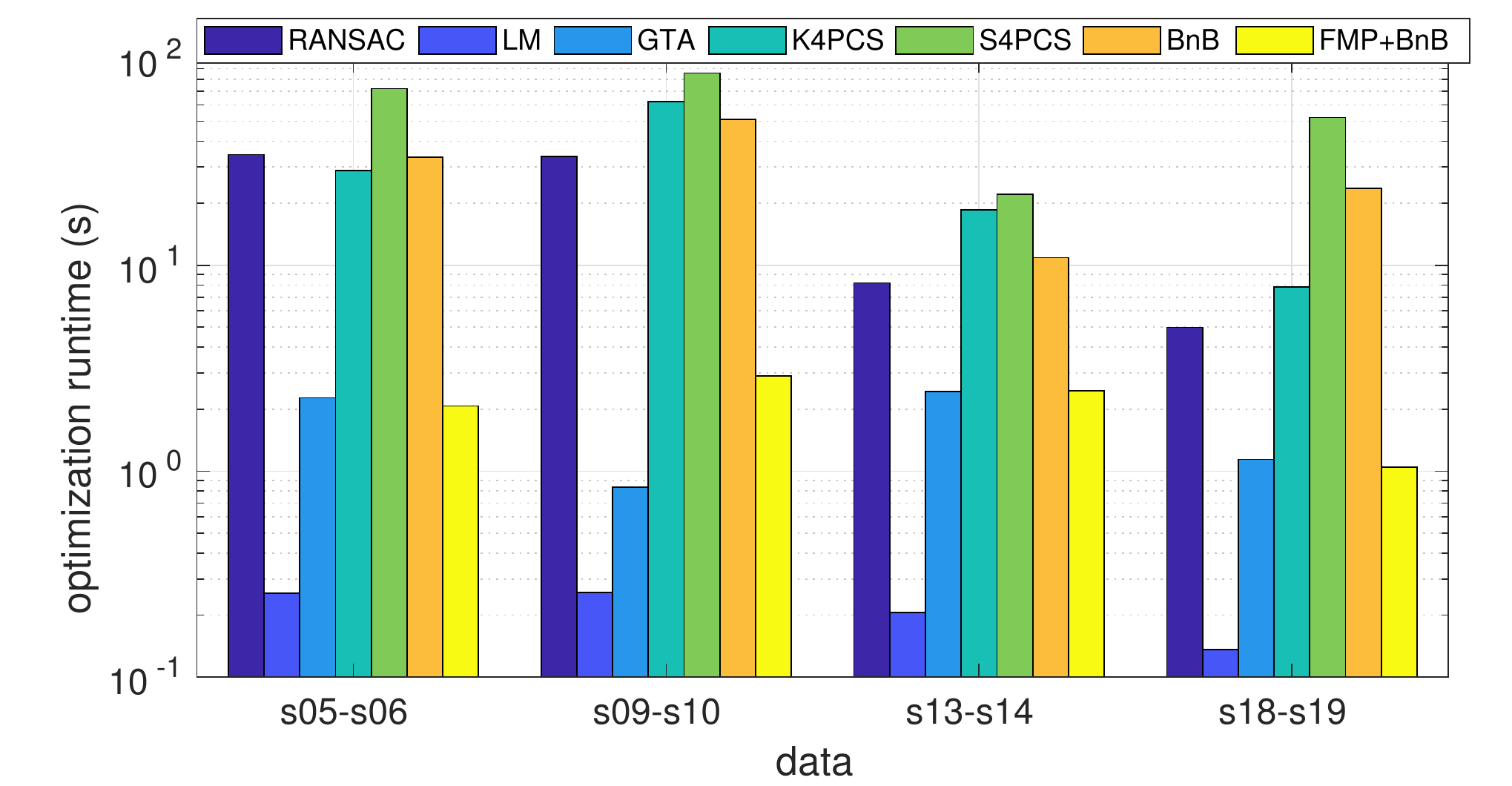}\label{subfig:Reg_Facility_regTime}}
\subfigure[Total runtime---\textit{Arch}.]{\includegraphics[width=0.8\columnwidth]{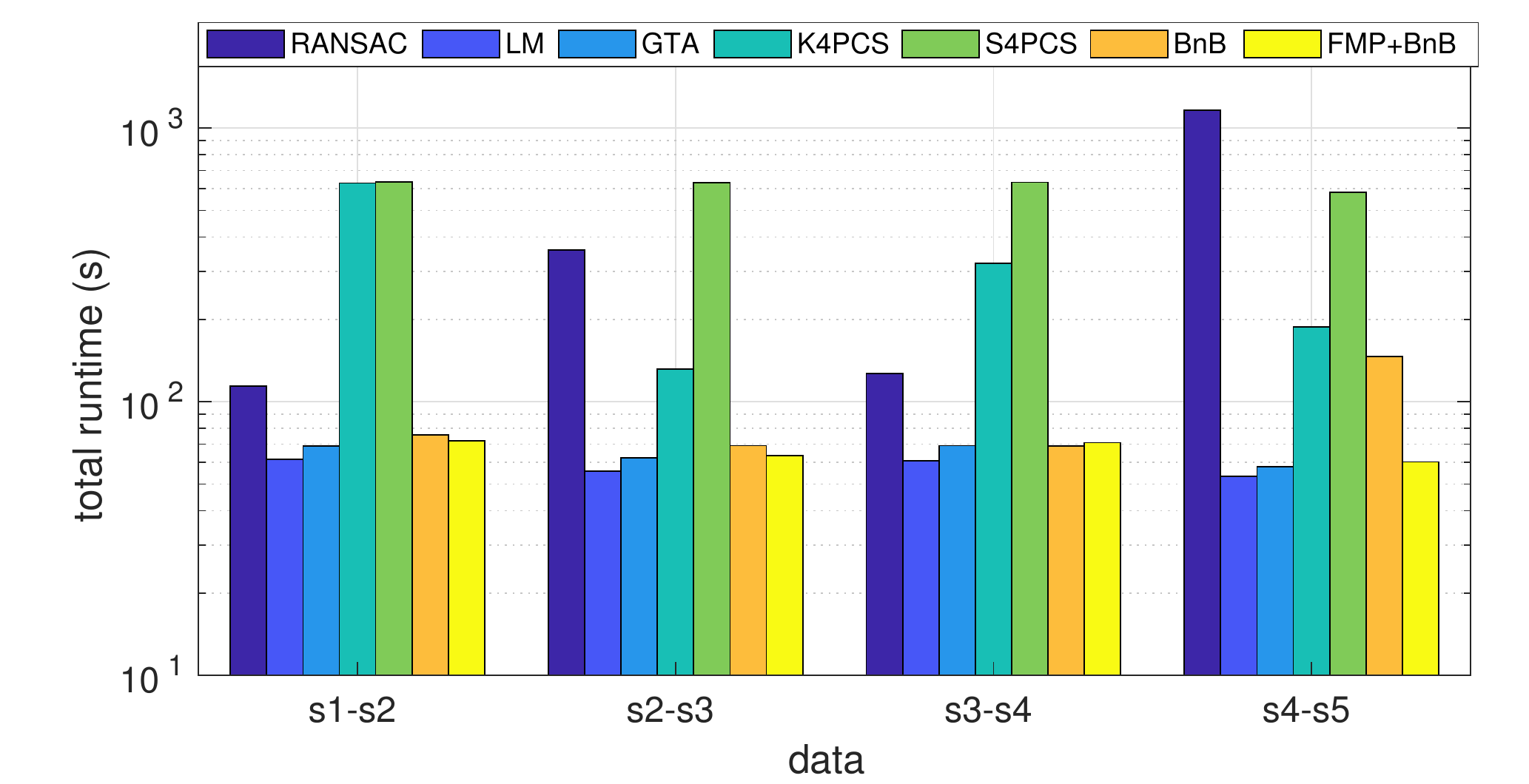}\label{subfig:Reg_Arch_regTimeWithPrep}}
\subfigure[Total runtime---\textit{Facility}.]{\includegraphics[width=0.8\columnwidth]{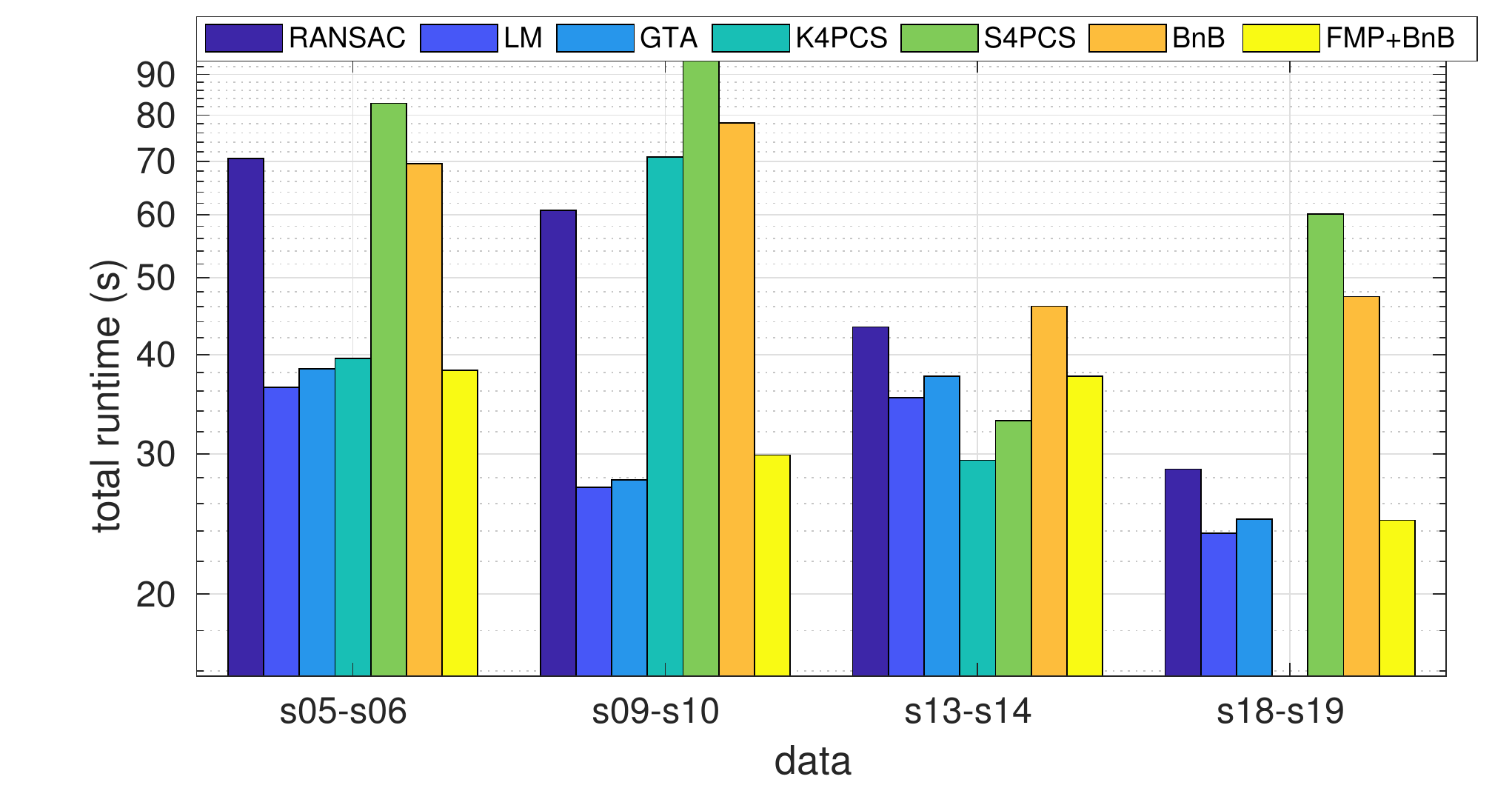}\label{subfig:Reg_Facility_regTimeWithPrep}}
\caption{The log scaled runtime of all methods with (up) and without (bottom) input genreration (only generating keypoints for K4PCS and S4PCS) on real-world data.}\label{fig:Reg_Time_Real}
\end{figure*}

\subsection{Challenging real-world data}\label{sec:expReal}

\begin{table}[t]
\centering
\scriptsize{
\begin{tabular}{ |c|c|c|c|c|c|  }
 \hline
 \multicolumn{6}{|c|}{\textit{Arch}} \\
 \hline
 Data & $|\bP|$ & $|\bQ|$ & $|\bP_{key}|$ & $|\bQ_{key}|$ & $|\cC|$  \\
 \hline
s01-s02 & 23561732 & 30898999 & 7906 & 4783 & 19879 \\
s02-s03 & 30898999 & 25249235 & 4783 & 7147 & 19344 \\
s03-s04 & 25249235 & 29448437 & 7147 & 5337 & 22213 \\
s04-s05 & 29448437 & 27955953 & 5337 & 4676 & 15529 \\
 \hline
\end{tabular}

\begin{tabular}{ |c|c|c|c|c|c|  }
 \hline
 \multicolumn{6}{|c|}{\textit{Facility}} \\
 \hline
 Data & $|\bP|$ & $|\bQ|$ & $|\bP_{key}|$ & $|\bQ_{key}|$ & $|\cC|$  \\
 \hline
s05-s06 & 10763194 & 10675580 & 2418 & 2727 & 10679 \\
s09-s10 & 10745525 & 10627814 & 2960 & 1327 & 7037 \\
s12-s13 & 10711291 & 10441772 & 1227 & 2247 & 6001 \\
s18-s19 & 10541274 & 10602884 & 1535 & 2208 & 7260 \\
 \hline
\end{tabular}
}
\caption{Size of the real-world data. $|\bP_{key}|$ and $|\bQ_{key}|$:
  number of keypoints.}\label{tab:dataReal}
\end{table}

\begin{figure*}[!htb]\centering
\subfigure[s3-s4---\textit{Arch}.]{\includegraphics[width=0.8\columnwidth]{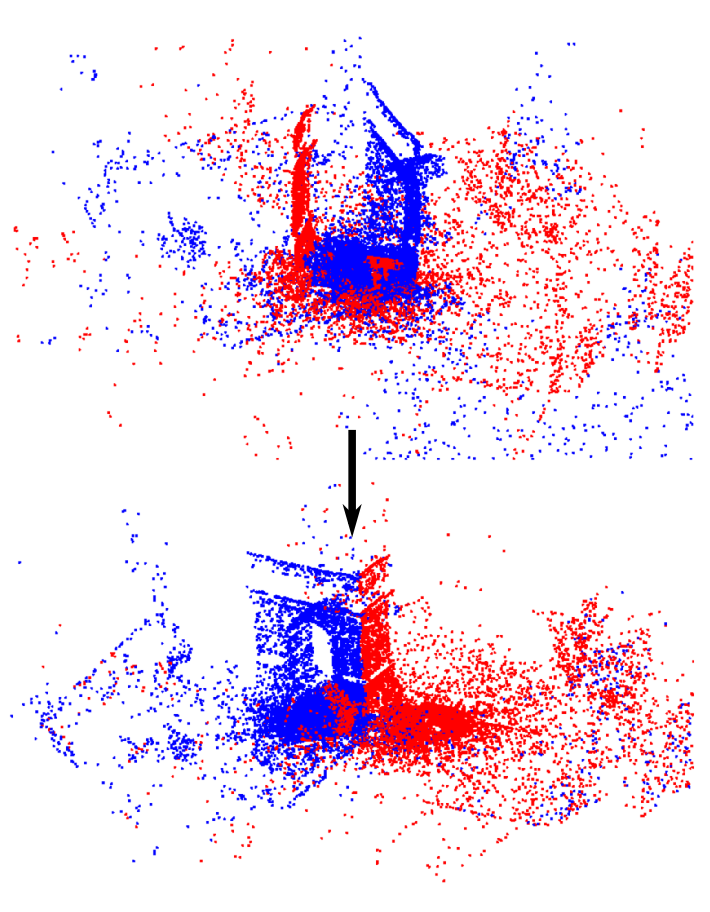}\label{subfig:Reg_Arch_Vis_3}}
\subfigure[s4-s5---\textit{Arch}.]{\includegraphics[width=0.8\columnwidth]{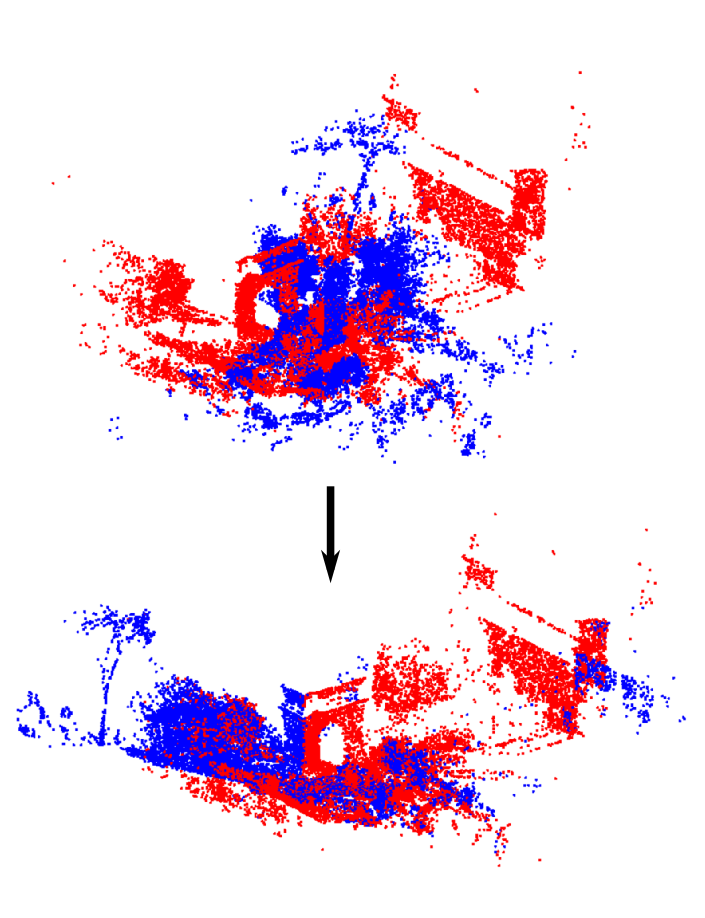}\label{subfig:Reg_Arch_Vis_4}}
\subfigure[s9-s10---\textit{Facility}.]{\includegraphics[width=0.8\columnwidth]{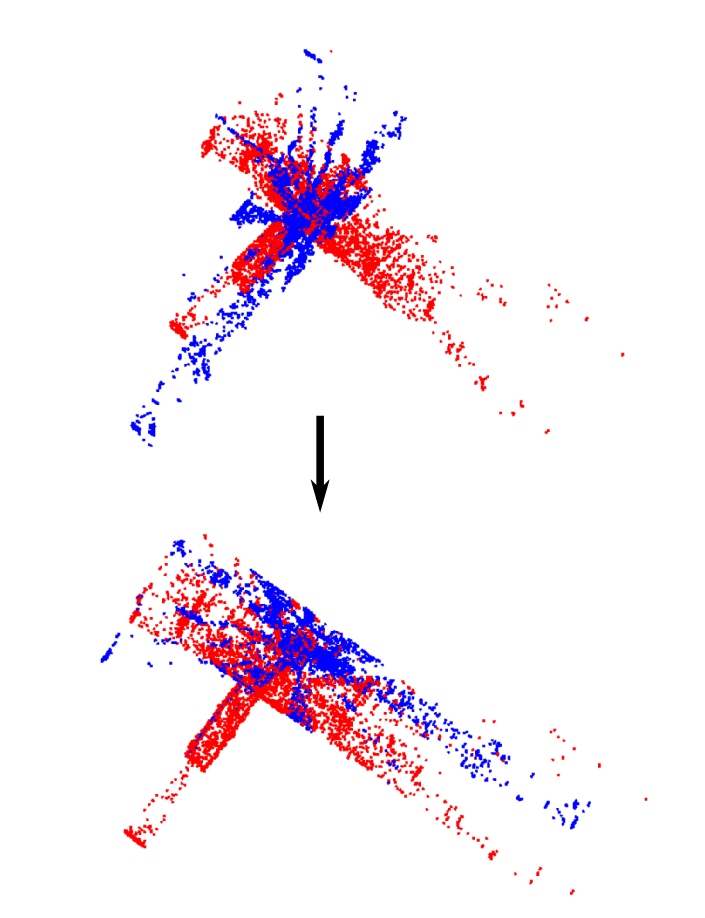}\label{subfig:Reg_Facility_Vis_1}}
\subfigure[s18-s19---\textit{Facility}.]{\includegraphics[width=0.8\columnwidth]{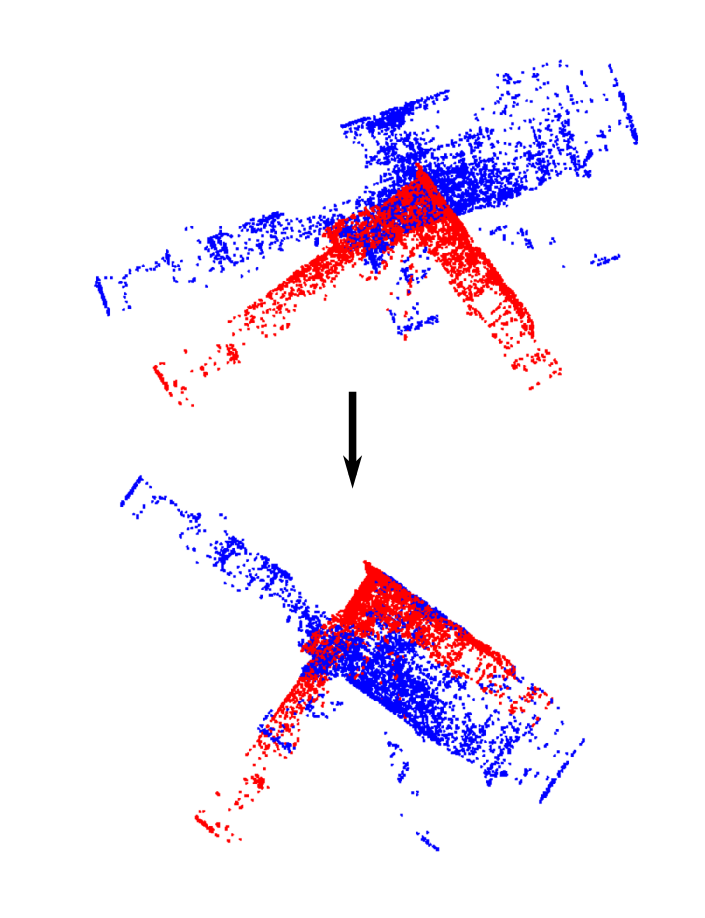}\label{subfig:Reg_Facility_Vis_3}}
\caption{Registered pairwise real-world data. 10k points are shown for
  each point cloud.  
    }\label{fig:Reg_Real_Vis}
\end{figure*}

To demonstrate the practicality of our method, comparisons on large
sale LiDAR
datasets\footnote{\url{http://www.prs.igp.ethz.ch/research/completed\_projects/automatic\_registration\_of\_point\_clouds.html}}
were also performed. Figure~\ref{fig:Reg_Acc_Real} reports the
accuracy of all methods on an outdoor dataset,~\textit{Arch}, and an
indoor dataset,~\textit{Facility}. Among the datasets used in
~\citep{theiler2015globally}, these are the most challenging ones,
which most clearly demonstrate the advantages of our proposed
method. We point out that both are not staged, but taken from real
surveying projects and representative for the registration problems
that arise in difficult field campaigns. For completeness,
Appendix~\ref{app:MoreResults} contains results on easier data, where
most of the compared methods work well). The accuracy was again
measured by the difference between the estimated and ground truth
(provided in the selected datasets, see~\citep{theiler2015globally} for
details) relative poses. Similar as before, at a reasonable error
threshold%
\footnote{Note, the threshold refers to errors after initial alignment
  from scratch. Of course, the result can be refined with ICP-type
  methods that are based on all points, not just a sparse
  match set.} %
of $15$ cm for translation, respectively $1^\circ$ for
rotation, our method had 100\% success rate; whereas failure cases
occurred with all other methods. And as shown in
Figure~\ref{fig:Reg_Time_Real}, FMP + BnB showed comparable speed than
its competitors on all data. It successfully registered millions of
points and tens of thousands of point matches (see
Table~\ref{tab:dataReal}) within 100 s, including the match
generation time.
We see the excellent balance between high accuracy and reliability on
the one hand, and low computation time on the other hand as a
particular strength of our method. 

Figures~\ref{fig:Reg_Real_Vis}--\ref{fig:Reg_Trees_Vis} visually show
various large scale scenes (pair-wise and complete) registered by our
method; more detailed demonstrations can be found in the video demo in our project homepage. Note that the runtime for registering a complete dataset (in
Figure~\ref{fig:regArchOutAll} and~\ref{fig:resultFacade}
to~\ref{fig:resultTrees}) was slightly less than the sum of
pair-wise runtimes, since a scan forms part of multiple pairs, but
keypoints and FPFH features need to be extracted only once and can
then be reused.

\begin{figure}[!htb]\centering
\subfigure{\includegraphics[width=1\columnwidth]{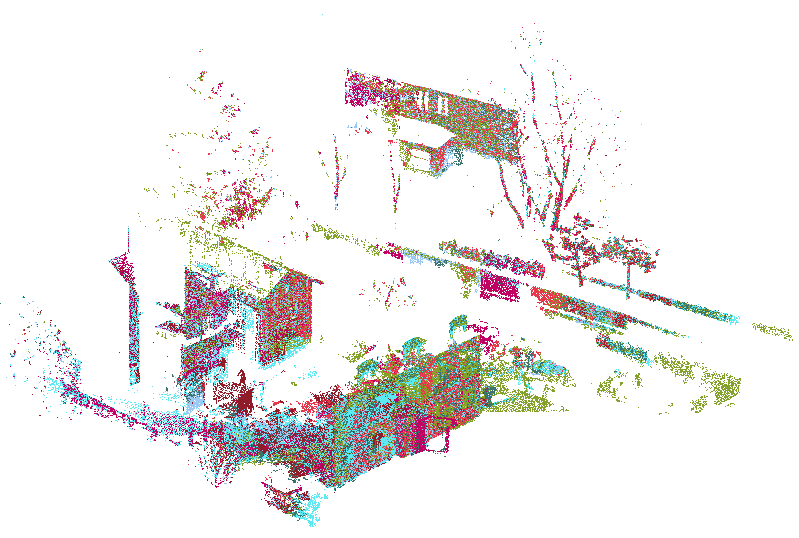}}
\caption{Complete registration result---\textit{Facade}. All 7 scans were registered in 134.95s.}\label{fig:Reg_Facade_Vis}
\end{figure}
\begin{figure}[!htb]\centering
\subfigure{\includegraphics[width=1\columnwidth]{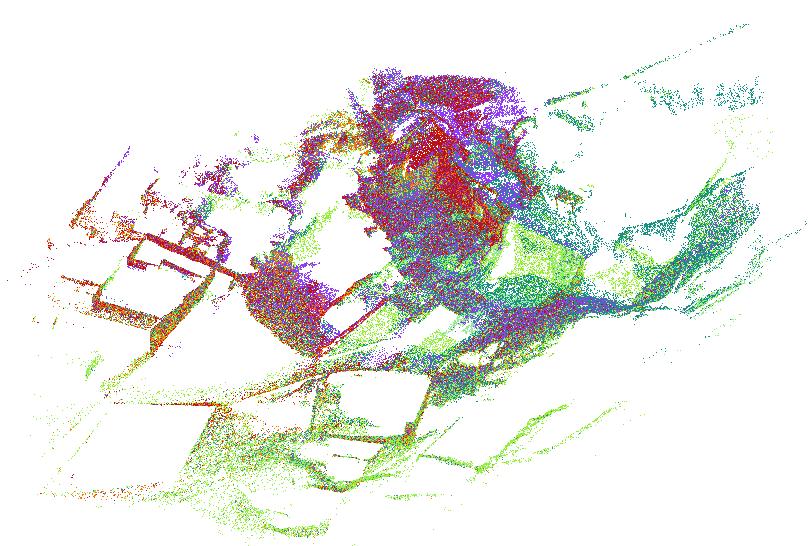}}
\caption{Complete registration result---\textit{Courtyard}. All 8 scans were registered in 84.02s.}\label{fig:Reg_Courtyard_Vis}
\end{figure}
\begin{figure}[!htb]\centering
\subfigure{\includegraphics[width=1\columnwidth]{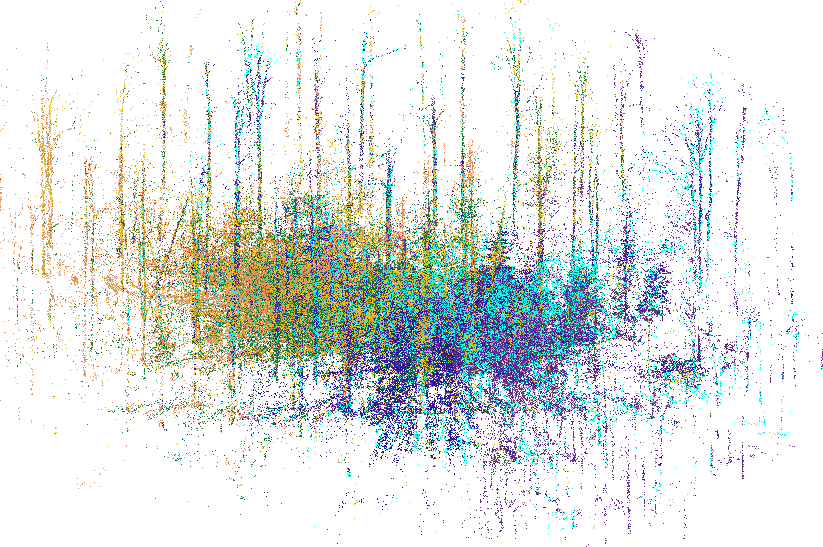}}
\caption{Complete registration result---\textit{Trees}. All 6 scans
  were registered in 142.72s.  
    }\label{fig:Reg_Trees_Vis}
\end{figure}

\section{Conclusion}
We have described a practical and optimal solution for
terrestrial LiDAR scan registration.
The main characteristic of our approach is that it combines a reliable,
optimal solver with high computational efficiency, by
exploiting two properties of terrestrial LiDAR: (1) it restricts the
registration problem to 4DOF (3D translation + azimuth), since modern
laser scanners have built-in level compensators. And (2) it
aggressively prunes the number of matches used to compute the
transformation, realising that the sensor noise level is low, therefore
a small set of corresponding points is sufficient for accurate
registration, so long as they are correct.
Given some set of candidate correspondences that may be contaminated with
a large number of incorrect matches (which is often the case for real
3D point clouds); our algorithm first applies a fast pruning method
that leaves the optimum of the alignment objective unchanged,
then uses the reduced point set to find that optimum with the
branch-and-bound method.
The pruning greatly accelerates the solver while keeping intact the
optimality w.r.t.\ the original problem.
The BnB solver explicitly searches the 3D translation space, which can
be bounded efficiently; while solving for 1D rotation implicitly with
a novel, polynomial-time algorithm. Experiments show that our
algorithm is significantly more reliable than previous methods, with competitive speed. 

\section{Acknowledgement}
This work was supported by the ARC grant DP160103490.

\section{Appendix}
\subsection{Compute $\mathtt{int}_i$ and solve the max-stabbing problem}~\label{app:compInt}
\begin{figure}[!htb]\centering
\subfigure{\includegraphics[width=1\columnwidth]{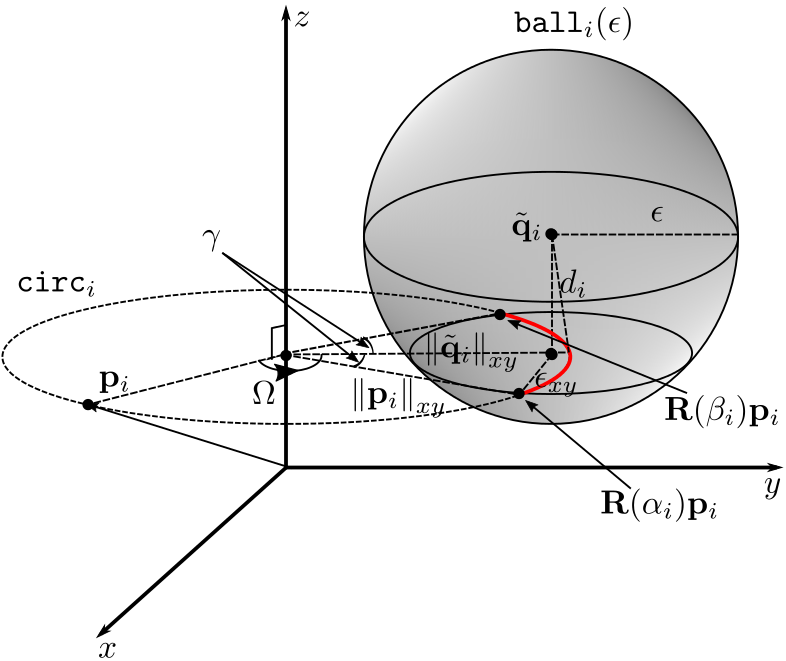}}
\caption{Illustration of Algorithm~\ref{alg:ComputInt}, $\mathtt{int}_i$ is rendered in red.}\label{fig:compInt}
\end{figure}
\begin{algorithm}[t]\centering
\caption{Compute the angle interval that aligns $\bp_i$ with $\tilde{\bq}_i$}
\label{alg:ComputInt}                         
\begin{algorithmic}[1]      
\REQUIRE $\bp_i$, $\tilde{\bq}_i$ and $\epsilon$, such that $d_i \le \epsilon$.
\STATE $\epsilon_{xy} \leftarrow \sqrt{\epsilon^2-(\bp_i(3)-\tilde{\bq}_i(3))^2}$
\IF{$|\norm{\bp_i}_{xy}+\norm{\tilde{\bq}_i}_{xy}|\leq\epsilon_{xy}$}\label{alg:ComputeInt_if}
\RETURN $[0,2\pi]$.
\ELSE
\STATE $\Omega = azi(\tilde{\bq}_i)-azi(\bp_i), \gamma = \arccos{\frac{\norm{\bp_i}^2_{xy}+\norm{\tilde{\bq}_i}^2_{xy}-\epsilon_{xy}}{2\norm{\bp_i}_{xy}\norm{\tilde{\bq}_i}_{xy}}}$.
\RETURN $[\Omega-\gamma,\Omega+\gamma]$.
\ENDIF
\end{algorithmic}
\end{algorithm}
Algorithm~\ref{alg:ComputInt} shows how to compute $\mathtt{int}_i = [\alpha_i,\beta_i]$ for each $(\bp_i,\tilde{\bq}_i)$ during rotation search. As shown in Figure~\ref{fig:compInt}, $\mathtt{circ}_i$ and $\mathtt{ball}_i(\epsilon)$ intersect if and only if the closest distance $d_i$ from $\tilde{\bq}_i$ to $\mathtt{circ}_i$ is within $\epsilon$, where 
\begin{align}\label{eq:circBallInt}
&d_i = \sqrt{(\norm{\bp_i}_{xy}-\norm{\tilde{\bq}_i}_{xy})^2+(\bp_i(3)-\tilde{\bq}_i(3))^2}.
\end{align}
In the above equation, $\bp_i(3)$ is the 3rd channel of $\bp_i$ and $\norm{\bp_i}_{xy} = \sqrt{\bp_i(1)^2+\bp_i(2)^2}$ is the \emph{horizontal} length of $\bp_i$. And the two intersecting points of $\mathtt{circ}_i$ and $\mathtt{ball}_i(\epsilon)$, namely $\bR(\alpha_i)\bp_i$ and $\bR(\beta_i)\bp_i$, have the same \emph{azimuthal} angular distance $\gamma$ to $\tilde{\bq}_i$. And by computing the azimuth $\Omega$ from $\bp_i$ to $\tilde{\bq}_i$, i.e., $\Omega = azi(\tilde{\bq}_i)-azi(\bp_i)$, where $azi(\cdot)$ is the azimuth of a point, $\mathtt{int}_i$ is simply $[\Omega-\gamma,\Omega+\gamma]$. Note that when $\mathtt{circ}_i$ is inside $\mathtt{ball}_i(\epsilon)$ (Line~\ref{alg:ComputeInt_if}), $\mathtt{int}_i$ is $[0,2\pi]$.

$\gamma$ is computed by the law of cosines\footnote{\url{https://en.wikipedia.org/wiki/Law\_of\_cosines}}, because it is an interior angle of the triangle whose three edge lengths are $\norm{\bp_i}_{xy}$, $\norm{\tilde{\bq}_i}_{xy}$ and $\epsilon_{xy}$, where $\epsilon_{xy}$ is the \emph{horizontal} distance between $\tilde{\bq}_i$ and $\bR(\alpha_i)\bp_i$ (or $\bR(\beta_i)\bp_i$).

\begin{algorithm}[t]\centering
\caption{Max-Stabbing algorithm for 1D rotation estimation}
\label{alg:1DRot}                         
\begin{algorithmic}[1]      
\REQUIRE $U = \{\mathtt{int}_i\}_{i=1}^M$, where $\mathtt{int}_i = [\alpha_{i},\beta_{i}]$.
\STATE $\bV \leftarrow \bigcup\limits_{i=1}^{M} \{[\alpha_{i}, 0],[\beta_{i}, 1]\}$, sort $\bV$ ascendingly according to~\ref{itm:a1} and~\ref{itm:a2}.
\STATE $\tilde{\delta} \leftarrow 0$, $\delta \leftarrow 0$. 
\FOR{each $\bv \in \bV$}
\IF{$\bv(2) = 0$}
\STATE $\delta \leftarrow \delta+1$. And if $\delta>\tilde{\delta}$, then $\tilde{\delta} \leftarrow \delta, \tilde{\theta} \leftarrow \bv(1)$.
\ELSE
\STATE $\delta \leftarrow \delta - 1$.
\ENDIF
\ENDFOR 
\RETURN $\tilde{\delta}, \tilde{\theta}$.
\end{algorithmic}
\end{algorithm}

After computing all $\mathtt{int}_i$'s, the max-stabbing
problem~\eqref{eq:objStab} can be efficiently solved by
Algorithm~\ref{alg:1DRot}. Observe that one of the endpoints among all
$\mathtt{int}_i$'s must achieve the max-stabbing, the idea of
Algorithm~\ref{alg:1DRot} is to compute the stabbing value $\delta$
for all endpoints and find the maximum one.

We first pack all endpoints into an array $\bV =
\bigcup\limits_{i=1}^{M} \{[\alpha_{i}, 0],[\beta_{i}, 1]\}$. The
$0/1$ label attached to each endpoint represents whether it is the end
of an $\mathtt{int}_i$, i.e., $\beta_i$. Then, to efficiently compute each
$\delta$, we sort $\bV$ so that the endpoints with
\begin{enumerate}
\item[\namedlabel{itm:a1}{\textbf{a1}}: ] smaller angles; 
\item[\namedlabel{itm:a2}{\textbf{a2}}: ] and the same angle but are the start of an $\mathtt{int}_i$ (the assigned label is 0); 
\end{enumerate}
are moved to the front.  After sorting, we initialize $\delta$ to 0
and traverse $\bV$ from the first element to the last. When we ``hit"
the beginning of an $\mathtt{int}_i$, we increment $\delta$ by 1, and when we
``hit" the end of an $\mathtt{int}_i$, $\delta$ is reduced by 1. The
max-stabbing value $\tilde{\delta}$ and its corresponding angle
$\tilde{\theta}$ are returned in the end.

Since the time complexity for sorting and traversing $\bV$ is
respectively $\mathcal{O}(M\log M)$ and $\mathcal{O}(M)$, the one for
Algorithm~\ref{alg:1DRot} is $\mathcal{O}(M\log M)$. The space complexity of
$\mathcal{O}(M)$ can be achieved with advanced sorting algorithms like merge
sort\footnote{\url{https://en.wikipedia.org/wiki/Merge\_sort}}.

\subsection{Proof of Lemma~\ref{lem:upBnd}}\label{app:lemProof}
First, to prove~\eqref{eq:lem_upBnd}, we re-express $\bt^*(\mathbb{S}) = \underset{\bt\in \mathbb{S}}{\argmax}~U(\bt \mid \cC,\epsilon)$ as $\bt_{\mathbb{S}}+\bt^*{'}(\mathbb{S})$, where $\norm{\bt^*{'}(\mathbb{S})}\le d_{\mathbb{S}}$. Using this re-expression, we have
\begin{subequations}\label{eq:proofUpBnd}
\begin{align}
&\norm{\bR(\theta)\bp_i+\bt_{\mathbb{S}}+\bt^*{'}(\mathbb{S})-\bq_i}\le \epsilon\label{eq:proofUpBnd1}\\
\Rightarrow &\norm{\bR(\theta)\bp_i+\bt_{\mathbb{S}}-\bq_i}-\norm{\bt^*{'}(\mathbb{S})} \le \epsilon\label{eq:proofUpBnd2}\\
\Rightarrow &\norm{\bR(\theta)\bp_i+\bt_{\mathbb{S}}-\bq_i}-d_{\mathbb{S}}\le \epsilon \\ \Leftrightarrow &\norm{\bR(\theta)\bp_i+\bt_{\mathbb{S}}-\bq_i}\le \epsilon+d_{\mathbb{S}}.\label{eq:proofUpBnd3}
\end{align}
\end{subequations}
\eqref{eq:proofUpBnd1} to~\eqref{eq:proofUpBnd2} is due to the triangle inequality. According to~\eqref{eq:proofUpBnd1} and~\eqref{eq:proofUpBnd3}, as long as $(\bp_i,\bq_i)$ contributes 1 in $\underset{\bt\in \mathbb{S}}{\max}~U(\bt \mid \cC,\epsilon) = U(\bt_{\mathbb{S}}+\bt^*{'}(\mathbb{S}) \mid \cC, \epsilon)$, it must also contribute 1 in $U(\bt_\mathbb{S} \mid \cC, \epsilon+d_\mathbb{S})$, i.e., $U(\bt_\mathbb{S} \mid \cC, \epsilon+d_\mathbb{S}) \ge \underset{\bt\in \mathbb{S}}{\max}~U(\bt \mid \cC,\epsilon)$.

And when $\mathbb{S}$ tends to a point $\bt$, $d_\mathbb{S}$ tends to 0. Therefore,
$ \bar{U}(\mathbb{S} \mid \cC, \epsilon) = U(\bt_{\mathbb{S}} \mid \cC, \epsilon+d_\mathbb{S})\rightarrow U(\bt \mid \cC, \epsilon)$.

\subsection{Further results}\label{app:MoreResults}

Figure~\ref{fig:resultFacade} to~\ref{fig:resultTrees} show the accuracy and runtime of our method and all other competitors in  Section~\ref{sec:expReal} for \textit{Facade}, \textit{Courtyard} and \textit{Trees}. The inlier threshold $\epsilon$ were set to 0.05 m for \textit{Facade}, 0.2 m for \textit{Courtyard} and 0.2 m for \textit{Trees}.

\begin{figure}[!htb]\centering
\subfigure[Rotation error.]{\includegraphics[width=0.49\columnwidth]{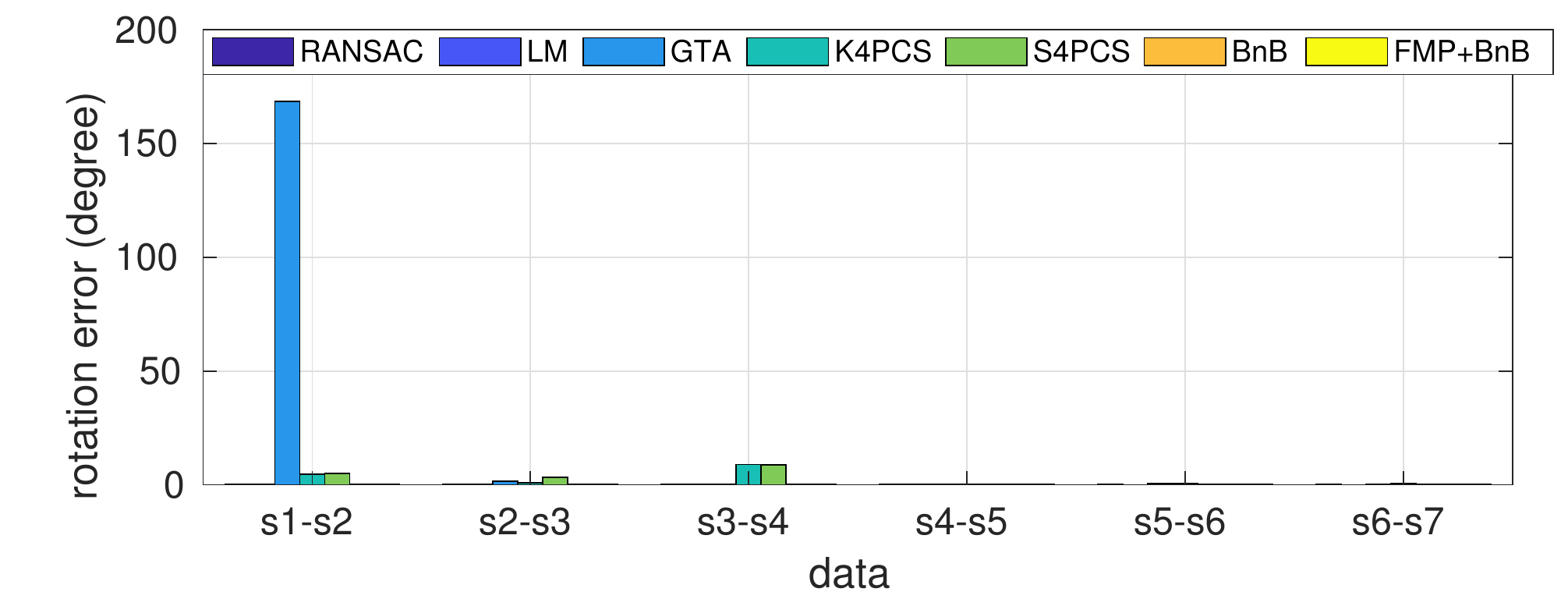}\label{subfig:Reg_Facade_eAng}}
\subfigure[Translation error.]{\includegraphics[width=0.49\columnwidth]{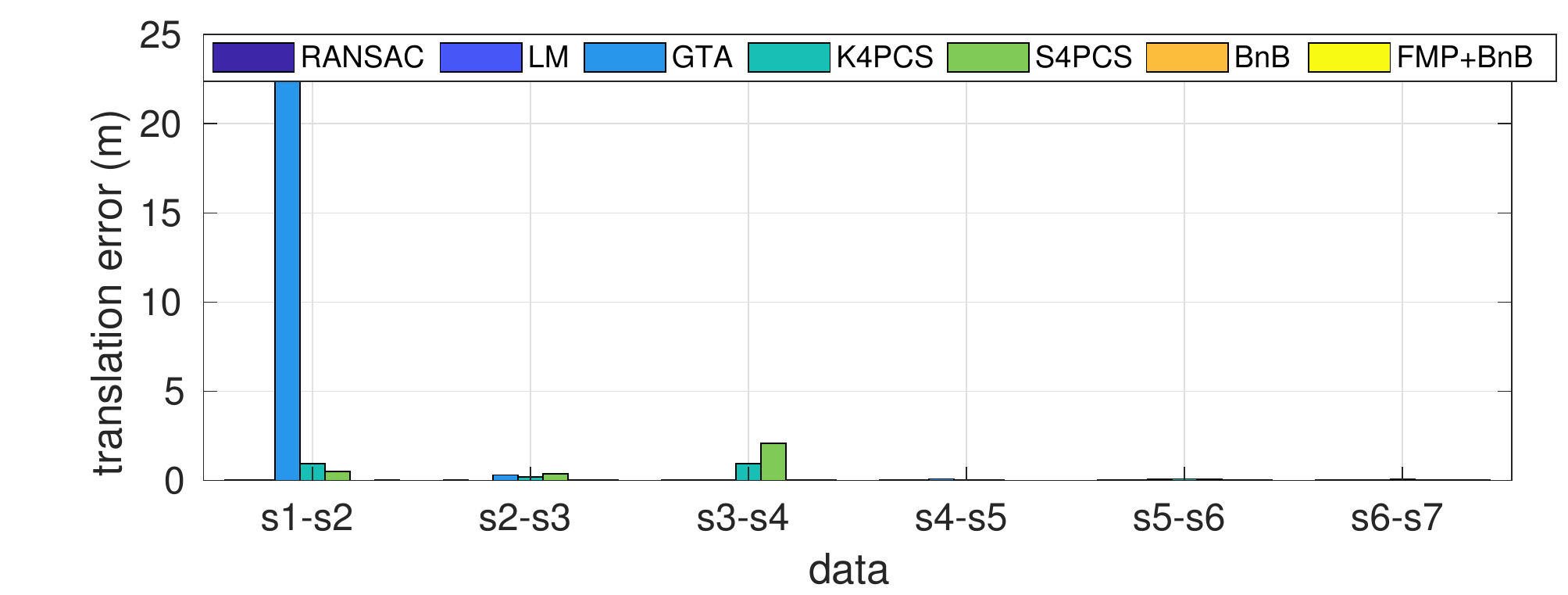}\label{subfig:Reg_Facade_eTran}}
\subfigure[Runtime for optimization.]{\includegraphics[width=0.49\columnwidth]{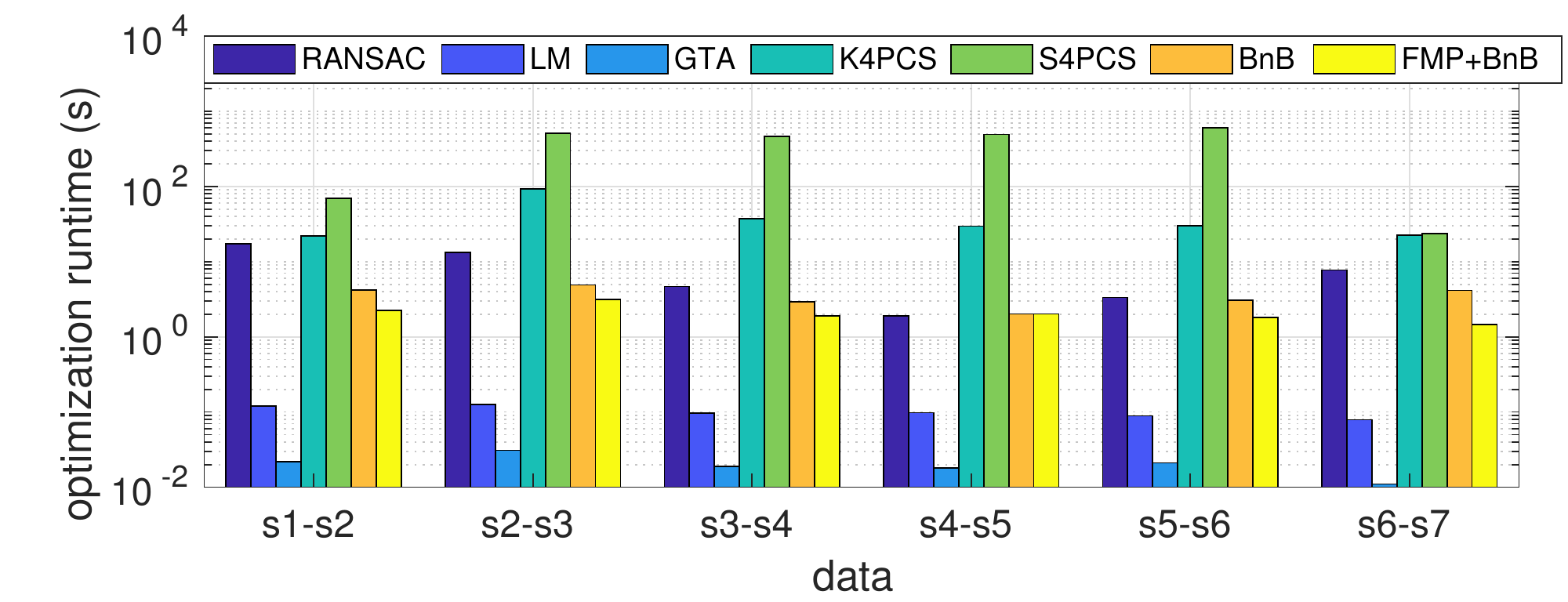}\label{subfig:Reg_Facade_regTime}}
\subfigure[Total runtime.]{\includegraphics[width=0.49\columnwidth]{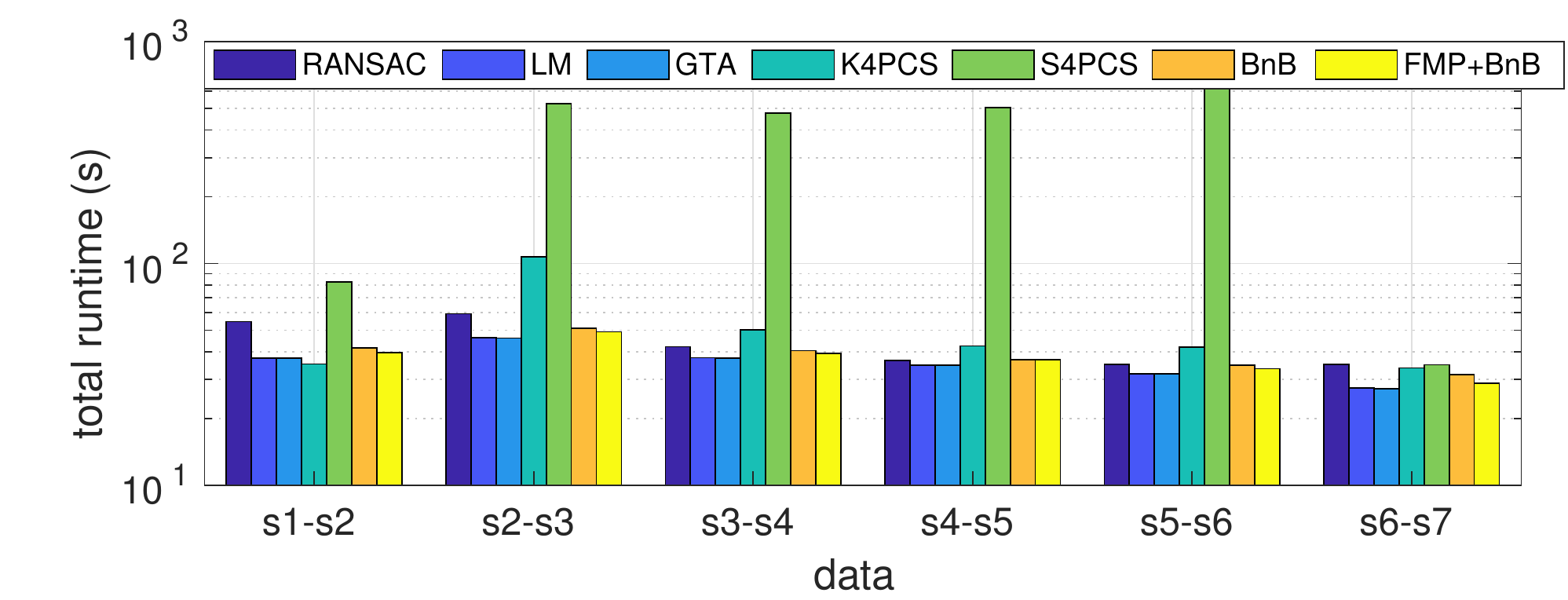}\label{subfig:Reg_Facade_regTimeWithPrep}}
\caption{Accuracy and log scaled runtime for \textit{Facade}.}\label{fig:resultFacade}
\end{figure}

\begin{figure}[!htb]\centering
\subfigure[Rotation error.]{\includegraphics[width=0.49\columnwidth]{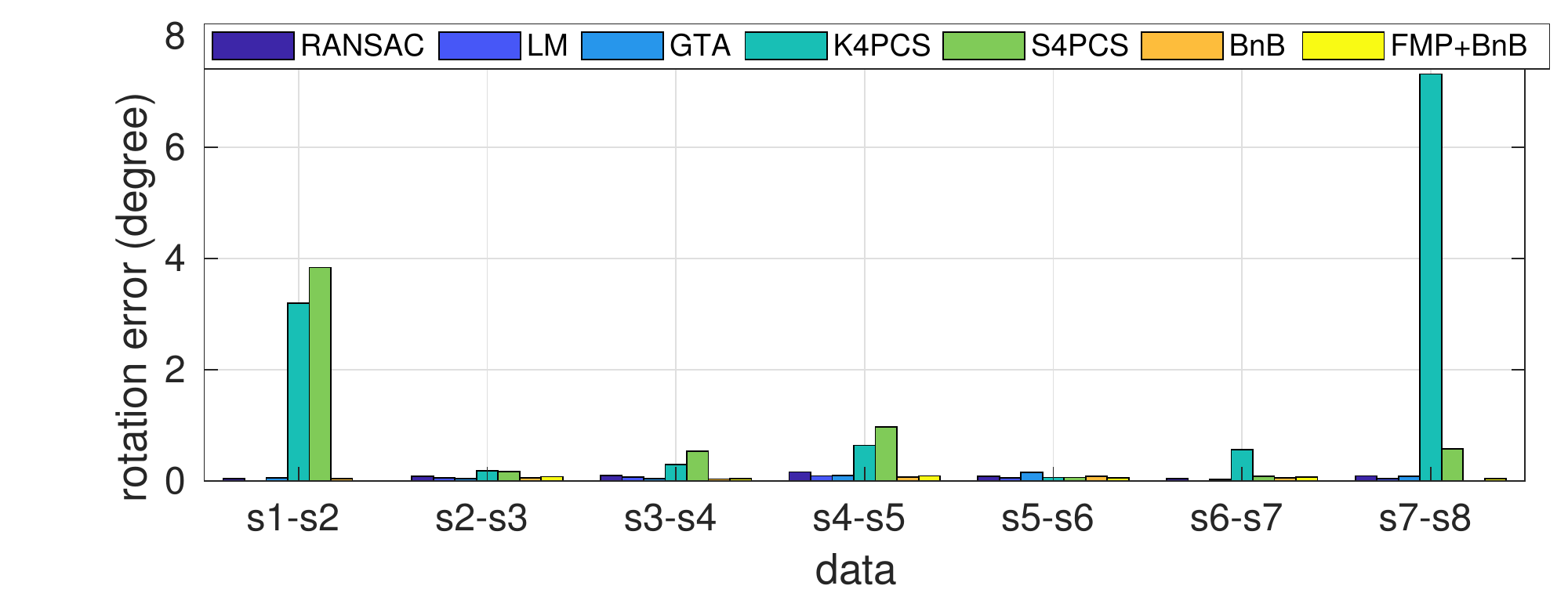}\label{subfig:Reg_Courtyard_eAng}}
\subfigure[Translation error.]{\includegraphics[width=0.49\columnwidth]{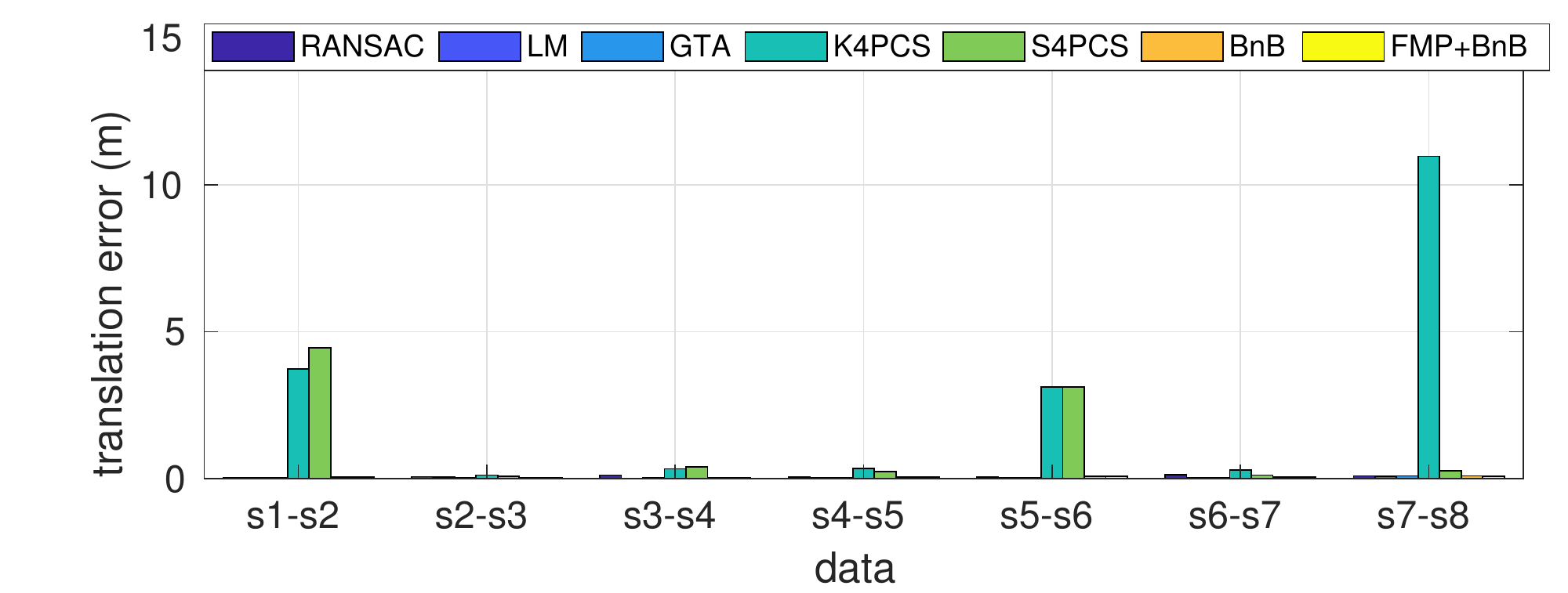}\label{subfig:Reg_Courtyard_eTran}}
\subfigure[Runtime for optimization.]{\includegraphics[width=0.49\columnwidth]{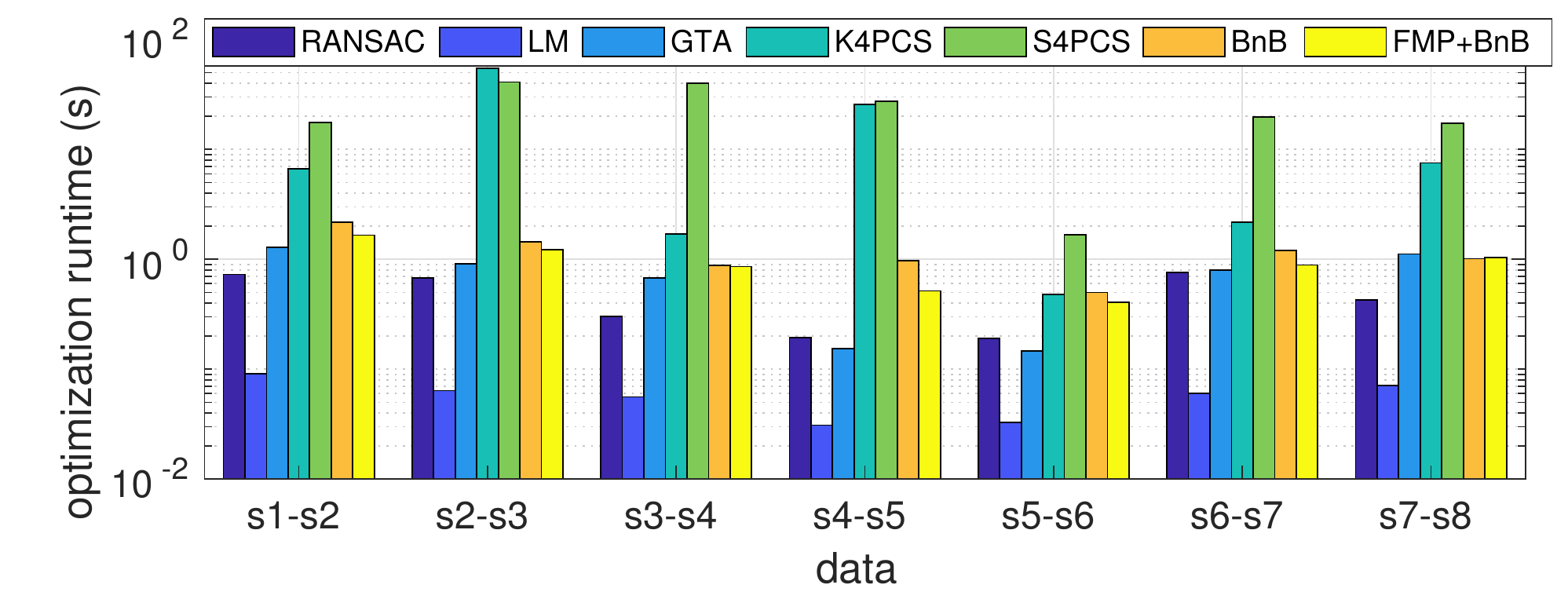}\label{subfig:Reg_Courtyard_regTime}}
\subfigure[Total runtime.]{\includegraphics[width=0.49\columnwidth]{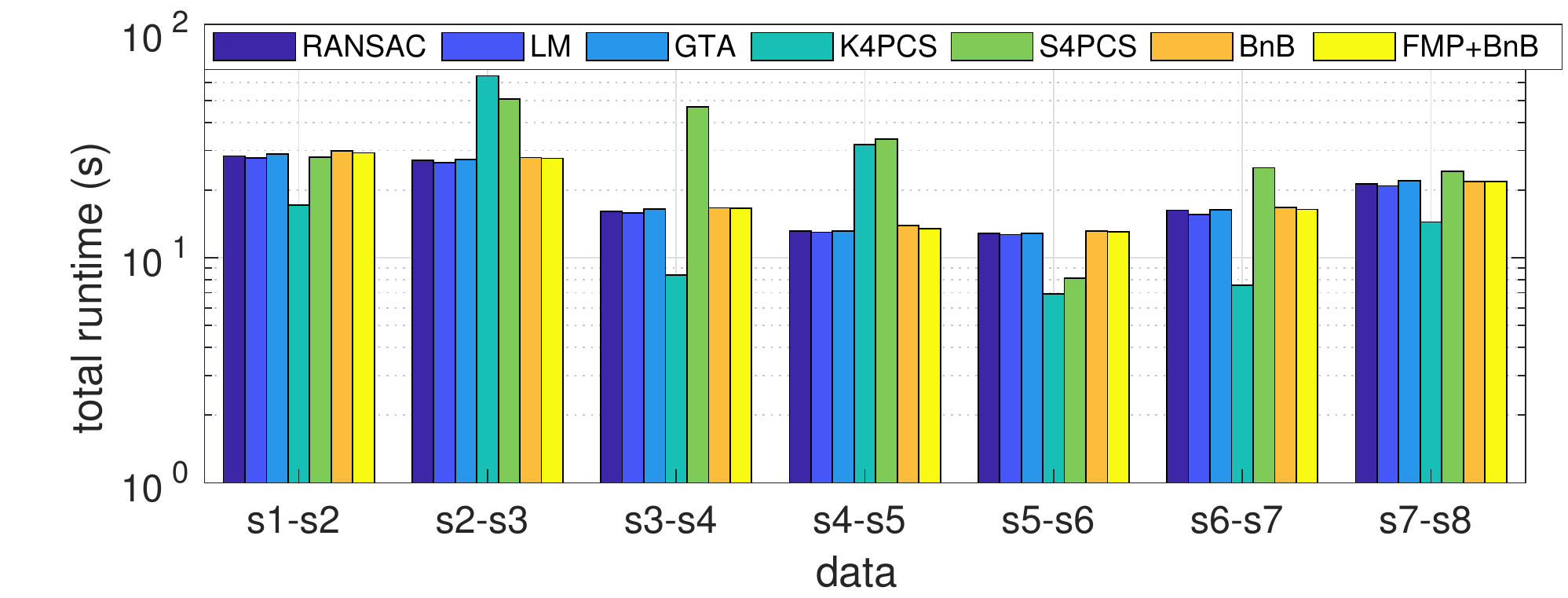}\label{subfig:Reg_Courtyard_regTimeWithPrep}}
\caption{Accuracy and log scaled runtime for \textit{Courtyard}.}\label{fig:resultCourtyard}
\end{figure}

\begin{figure}[!htb]\centering
\subfigure[Rotation error.]{\includegraphics[width=0.49\columnwidth]{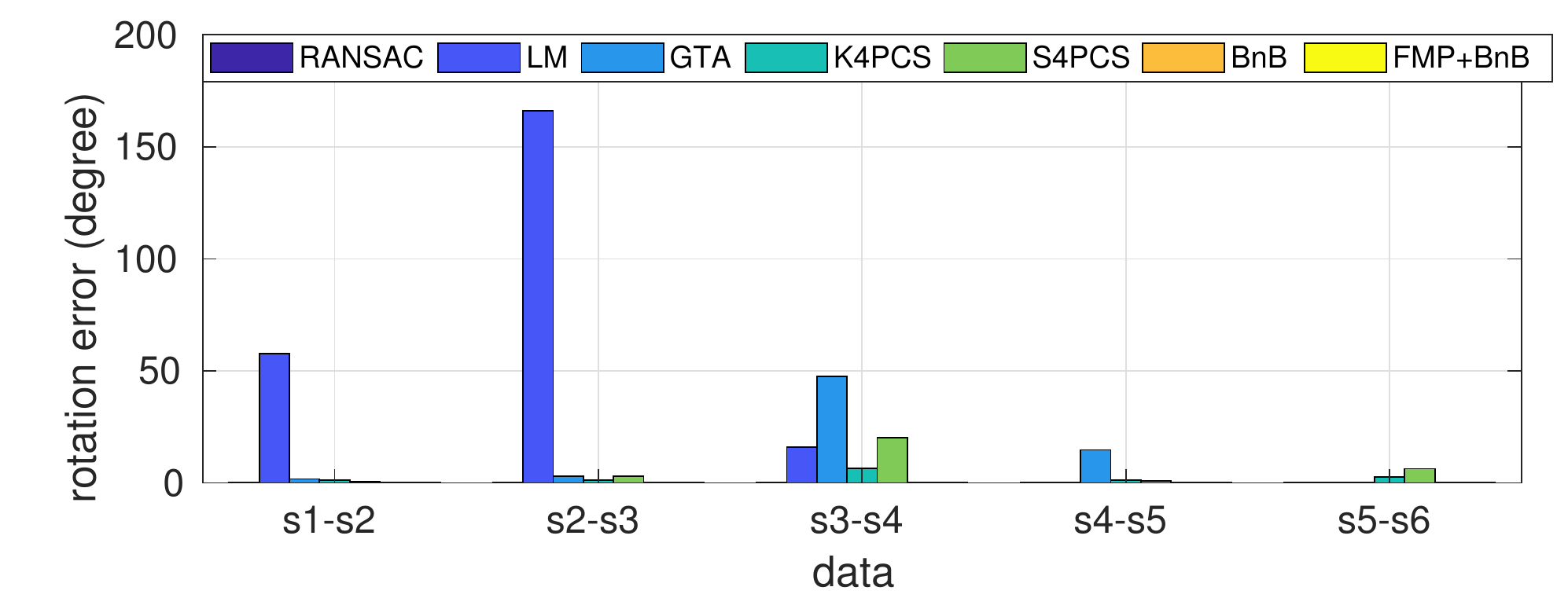}\label{subfig:Reg_Trees_eAng}}
\subfigure[Translation error.]{\includegraphics[width=0.49\columnwidth]{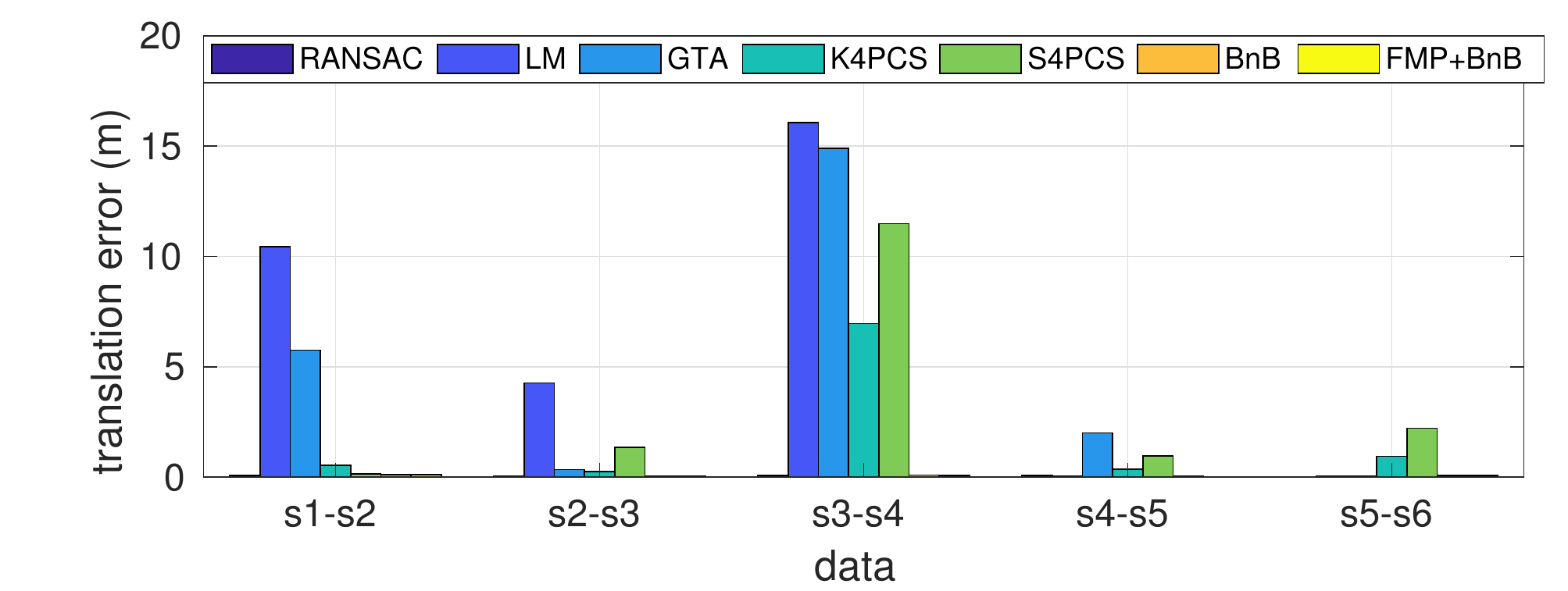}\label{subfig:Reg_Trees_eTran}}
\subfigure[Runtime for optimization.]{\includegraphics[width=0.49\columnwidth]{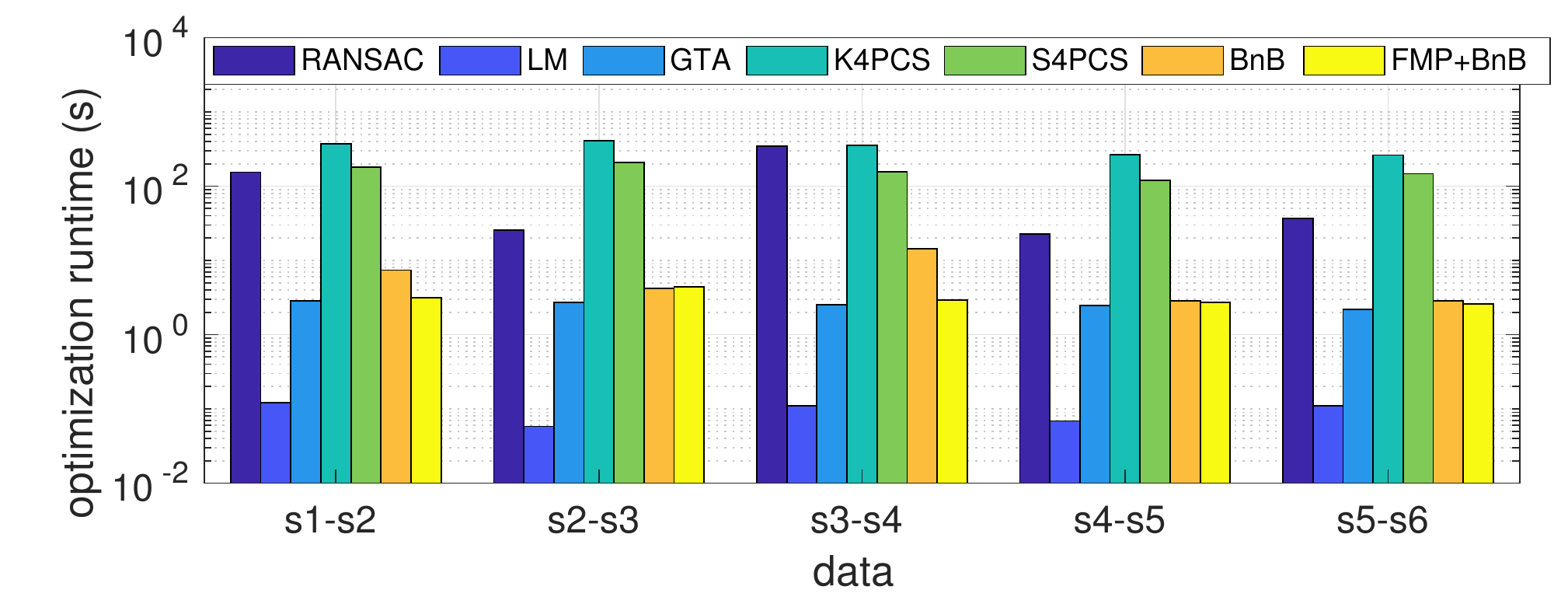}\label{subfig:Reg_Trees_regTime}}
\subfigure[Total runtime.]{\includegraphics[width=0.49\columnwidth]{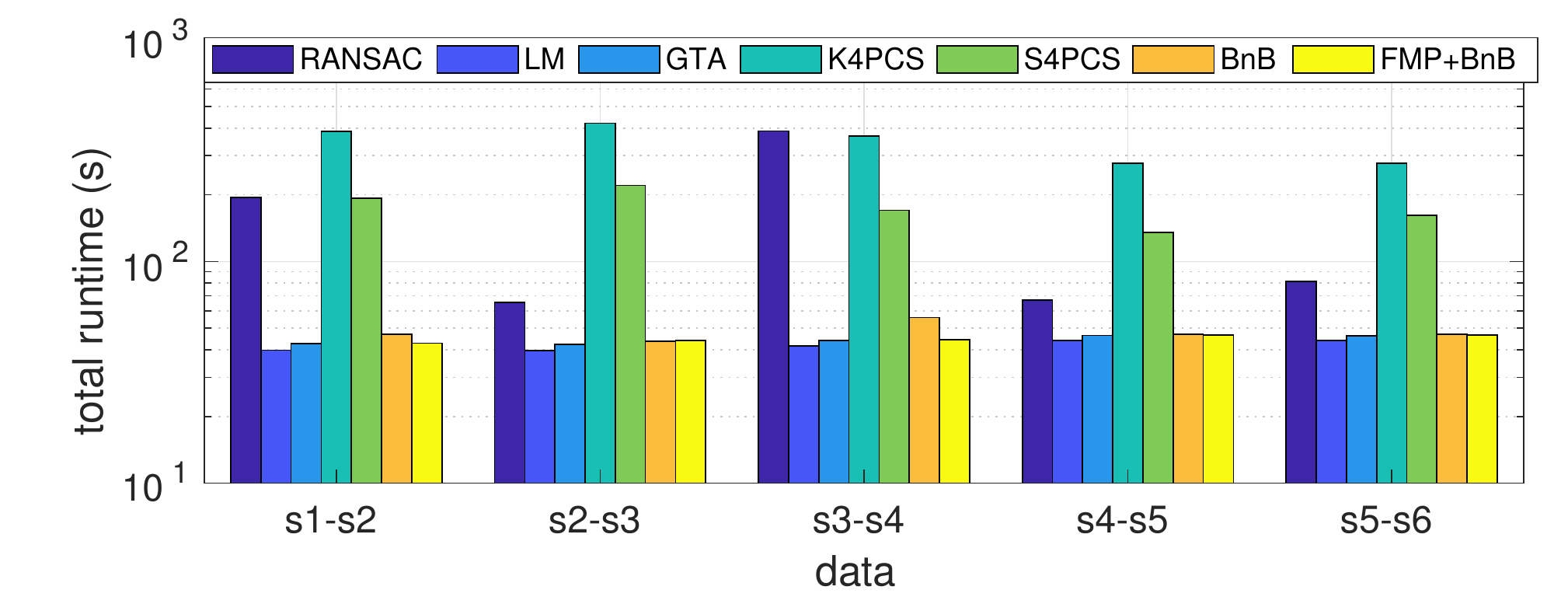}\label{subfig:Reg_Trees_regTimeWithPrep}}
\caption{Accuracy and log scaled runtime for \textit{Trees}.}\label{fig:resultTrees}
\end{figure}

\clearpage

\bibliography{4DoFReg_revised}

\end{document}